\documentclass[twoside,11pt]{article}
\usepackage{jmlr2e}
  
\usepackage{xcolor}
\usepackage{amsmath}
\usepackage{paralist}   
\usepackage[mathscr]{euscript}
\usepackage{bm}
\usepackage{url}
\usepackage{multirow} 
\usepackage{afterpage}
\usepackage{subfigure}
\usepackage{enumitem} 
\usepackage{lscape}  
\usepackage[breaklinks=true]{hyperref}
\usepackage{tikz} 
\usepackage{breakcites}
\usetikzlibrary{arrows,calc}
\definecolor{darkblue}{rgb}{0,0.08,0.45}
 
\usepackage{soul}

\hypersetup{  
  colorlinks = true, 
  linkcolor  = darkblue, 
  citecolor  = darkblue,
  filecolor  = darkblue, 
  urlcolor   = darkblue
} 
    
\allowdisplaybreaks

\newcommand{\ie}{\emph{i.e.}}
\newcommand{\eg}{\emph{e.g.}}

\newcommand{\inspace}{\ensuremath{\mathcal{X}}}   
\newcommand{\pp}[1]{\ensuremath{\mathbb{#1}}}     
\newcommand{\pspace}{\ensuremath{\mathscr{P}}}   

\newcommand{\hbspace}{\ensuremath{\mathscr{H}}}   
\newcommand{\nbspace}{\ensuremath{\mathscr{N}}}

\newcommand{\empmm}{\ensuremath{\hat{\mu}}}

\newcommand{\rr}{\mathbb{R}} 		          
\newcommand{\ep}{\mathbb{E}}                      
\newcommand{\kmat}{\mathbf{K}}                  
\newcommand{\bvec}{\bm{\beta}}                  
\newcommand{\id}{\mathbf{I}}

\newcommand{\dd}{\, \mathrm{d}}

\jmlrheading{1}{2015}{xx-xx}{5/14; Revised 3/15}{--/--}{Krikamol Muandet, Bharath Sriperumbudur, Kenji Fukumizu, Arthur Gretton, and Bernhard Sch\"olkopf}

\ShortHeadings{Kernel Mean Shrinkage Estimators}{Muandet, Sriperumbudur, Fukumizu, Gretton, and Sch\"olkopf}
\firstpageno{1}
 
\begin{document} 

\title{Kernel Mean Shrinkage Estimators}
 
\author{\name Krikamol Muandet\thanks{Contributed equally} \email krikamol@tuebingen.mpg.de \\
       \addr Empirical Inference Department, Max Planck Institute for Intelligent Systems \\
       Spemannstra\ss e 38, T\"ubingen 72076, Germany 
       \AND
       \name Bharath Sriperumbudur$^*$ \email bks18@psu.edu \\ 
       \addr Department of Statistics, Pennsylvania State University \\
       University Park, PA 16802, USA
       \AND
       \name Kenji Fukumizu \email fukumizu@ism.ac.jp \\ 
       \addr The Institute of Statistical Mathematics \\ 
       10-3 Midoricho, Tachikawa, Tokyo 190-8562 Japan              
       \AND 
       \name Arthur Gretton \email arthur.gretton@gmail.com \\
       \addr Gatsby Computational Neuroscience Unit, CSML, University College London \\
       Alexandra House, 17 Queen Square, London - WC1N 3AR, United Kingdom
       \AND
       \name Bernhard Sch\"olkopf \email bs@tuebingen.mpg.de \\
       \addr Empirical Inference Department, Max Planck Institute for Intelligent Systems \\
       Spemannstra\ss e 38, T\"ubingen 72076, Germany}

\editor{Ingo Steinwart}

\maketitle

\begin{abstract}
  A mean function in a reproducing kernel Hilbert space (RKHS), or a kernel mean, is central to kernel methods in that it is used by many classical algorithms 
  such as kernel principal component analysis, and it also forms the core inference step of modern kernel methods that rely on embedding probability distributions in RKHSs. 
  Given a finite sample, an empirical average has been used commonly as a standard estimator of the true kernel mean. Despite a widespread use of this estimator, we show that it can be improved 
  thanks to the well-known Stein phenomenon. We propose a new family of estimators called kernel mean shrinkage estimators (KMSEs), which benefit from both theoretical justifications and good empirical performance. The results demonstrate that the proposed estimators outperform the standard one, especially in a ``large $d$, small $n$'' paradigm. 
\end{abstract}


\section{Introduction}
This paper aims to improve the estimation of the mean function in a reproducing kernel Hilbert space (RKHS), or a kernel mean, from a finite sample. A kernel mean is defined with respect to a probability distribution $\pp{P}$ over a measurable space $\inspace$ by
\begin{equation}
  \label{eq:kernel-mean}
  \mu_{\pp{P}} \triangleq \int_{\inspace}k(x,\cdot)\;\dd\pp{P}(x) \in
  \hbspace ,
\end{equation}  
\noindent where $\mu_{\pp{P}}$ is a Bochner integral (see, \eg, \citet[Chapter 2]{Diestel-77} and \citet[Chapter 1]{Dinculeanu:2000} for a definition of Bochner integral) and $\hbspace$ is a separable RKHS 
endowed with a measurable reproducing kernel $k:\inspace\times\inspace\rightarrow\rr$ such that 
$\int_\mathcal{X} \sqrt{k(x,x)}\,\dd\pp{P}(x)<\infty$.\footnote{The separability of $\hbspace$ and measurability of $k$ ensures that $k(\cdot,x)$ is a $\hbspace$-valued measurable function for all $x\in\mathcal{X}$
\citep[Lemma A.5.18]{Steinwart-08}. The separability of $\hbspace$ is guaranteed by choosing $\mathcal{X}$ to be a separable topological space and $k$ to be continuous \citep[Lemma 4.33]{Steinwart-08}.\label{fnote:1}} 
Given an i.i.d sample $x_1,x_2,\ldots,x_n$ from $\pp{P}$, the most natural estimate of the true kernel mean is empirical average
\begin{equation}
  \label{eq:empirical-mean} 
  \hat{\mu}_{\pp{P}} \triangleq \frac{1}{n}\sum_{i=1}^n k(x_i,\cdot)\,.
\end{equation}
We refer to this estimator as a \emph{kernel mean estimator} (KME). Though it is the most commonly used estimator of the true kernel mean, the key contribution of this work is to show that there exist estimators that can improve upon this standard estimator.

The kernel mean has recently gained attention in the machine learning community, thanks to the introduction of Hilbert space embedding for distributions \citep{Berlinet04:RKHS,Smola07Hilbert}. Representing the distribution as a mean function in the RKHS has several advantages. First, if the kernel $k$ is \emph{characteristic}, the map $\pp{P}\mapsto\mu_{\pp{P}}$ is injective.\footnote{The notion of characteristic kernel is closely related to the notion of universal kernel. In brief, if the kernel is universal, it is also characteristic, but the reverse direction is not necessarily the case. See, \eg, \citet{Sriperumbudur10b:Universal}, for more detailed accounts on this topic.} That is, it preserves all information about the distribution \citep{Fukumizu04:DRS,Sriperumbudur08injectivehilbert}. Second, basic operations on the distribution can be carried out by means of inner products in RKHS, \eg, $\ep_{\pp{P}}[f(x)]=\langle f,\mu_{\pp{P}}\rangle_{\hbspace}$ for all $f\in\hbspace$, which is an essential step in probabilistic inference 
\citep[see, e.g.,][]{Song11:LTGM}. Lastly, no intermediate density estimation is 
required, for example, when testing for homogeneity from finite samples. Thus, the algorithms become less susceptible to the curse of dimensionality; see, \eg, \citet[Section 6.5]{Wasserman06:NS} and \citet{Sriperumbudur12:Empirical}.

The aforementioned properties make Hilbert space embedding of distributions appealing to many algorithms in modern kernel methods, namely, two-sample testing via maximum mean discrepancy (MMD) \citep{Gretton07:MMD,Gretton12:KTT}, kernel independence tests \citep{Gretton-08}, Hilbert space embedding of HMMs \citep{Song10:HMM}, and kernel Bayes rule \citep{Fukumizu11:KBR}. The performance of these algorithms relies directly on the quality of the empirical estimate $\hat{\mu}_{\pp{P}}$.

In addition, the kernel mean has played much more fundamental role as a basic building block of many kernel-based learning algorithms \citep{Vapnik98:SLT,Scholkopf98:NCA}. For instance, nonlinear component analyses, such as kernel principal component analysis (KPCA), kernel Fisher discriminant analysis (KFDA), and kernel canonical correlation analysis (KCCA), rely heavily on mean functions and covariance operators in RKHS \citep{Scholkopf98:NCA,Fukumizu07:KCCA}. The kernel $K$-means algorithm performs clustering in feature space using mean 
functions as representatives of the clusters \citep{Dhillon04:KKS}. Moreover, the kernel mean also served as a basis in early development of algorithms for classification, density estimation, and anomaly detection \citep[Chapter 5]{Shawe04:KMPA}. All of these employ the empirical average in \eqref{eq:empirical-mean} as an estimate of the true kernel mean. 

We show in this work that the empirical estimator in \eqref{eq:empirical-mean} is, in a certain sense, not optimal, i.e., there exist ``better" estimators (more below), and 
then propose simple estimators that outperform the empirical estimator. While it is reasonable to argue that 
$\hat{\mu}_{\pp{P}}$ is the ``best'' possible estimator of $\mu_{\pp{P}}$ if nothing is known about $\pp{P}$ (in fact $\hat{\mu}_{\pp{P}}$ is minimax in the sense of 
\citet[Theorem 25.21, Example 25.24]{Vaart-98}), in this paper we show that ``better'' estimators of $\mu_\pp{P}$ can be constructed if mild assumptions are made on $\pp{P}$. 
This work is to some extent inspired by Stein's seminal work in 1955, which showed that the maximum likelihood estimator (MLE) of 
the mean, $\theta$ of a multivariate Gaussian distribution $\mathcal{N}(\theta,\sigma^2\id)$ is ``inadmissible'' \citep{Stein55:Inadmissible}---i.e., there exists 
a better estimator---though it is minimax optimal. 
In particular, Stein showed that there exists an estimator that always achieves smaller total mean squared error regardless of the true $\theta\in\mathbb{R}^d$, when $d\ge 3$. 
Perhaps the best known estimator of such kind is James-Stein’s estimator \citep{Stein61:JSE}. Formally, if $X\sim\mathcal{N}(\theta,\sigma^2\id)$ with $d \geq 3$, 
the estimator $\delta(X)=X$ for $\theta$ is inadmissible in mean squared sense and is dominated 
by the following estimator
\begin{equation}
  \label{eq:JS-estimator}
  \delta_{\text{JS}}(X) = \left(1 - \frac{(d-2)\sigma^2}{\|X\|^2}\right)X,
\end{equation}
\noindent i.e., $\pp{E}\Vert \delta_{\text{JS}}(X)-\theta\Vert^2 \le \pp{E}\Vert \delta(X)-\theta\Vert^2$ for all $\theta$ and there exists at least one $\theta$ for which 
$\pp{E}\Vert \delta_{\text{JS}}(X)-\theta\Vert^2 < \pp{E}\Vert \delta(X)-\theta\Vert^2$.

Interestingly, the James-Stein estimator is itself inadmissible, and there exists a wide class of estimators that outperform the MLE, see, \eg, \citet{Berger76:quadratic}. 
Ultimately, Stein's result suggests that one can construct estimators better than the usual empirical estimator 
if the relevant parameters are estimated jointly and if the definition of risk ultimately looks at all of these parameters (or coordinates) together. 
This finding is quite remarkable as it is counter-intuitive as to why joint estimation should yield better estimators when all parameters are mutually independent \citep{Efron77:paradox}.  
Although the Stein phenomenon has been extensively studied in the statistics community, it has not received much attention in the machine learning community.

The James-Stein estimator is a special case of a larger class of estimators known as \emph{shrinkage estimators} \citep{Gruber98:Shrinkage}. In its most general form, 
the shrinkage estimator is a combination of a model with low bias and high variance, and a model with high bias but low variance. 
For example, one might consider the following estimator:
\begin{equation*} 
  \hat{\theta}_{\text{shrink}} \triangleq \lambda\tilde{\theta} + (1-\lambda)\hat{\theta}_{\text{ML}} ,
\end{equation*}
\noindent where $\lambda \in [0,1]$, $\hat{\theta}_{\text{ML}}$ denotes the usual maximum likelihood estimate of $\theta$, and $\tilde{\theta}$ is an arbitrary point in the input space. In the case of James-Stein estimator, we have $\tilde{\theta}=0$. 
Our proposal of shrinkage estimator to estimate $\mu_\pp{P}$ will rely on the same principle, but will differ fundamentally from the Stein's seminal 
works and those along this line in two aspects. First, our setting is ``non-parametric'' in the sense that we do not assume any parametric form for the distribution, 
whereas most of traditional works focus on some specific distributions, \eg, the Gaussian distribution. The non-parametric setting is very important in most applications of 
kernel means because it allows us to perform statistical inference without making any assumption on the parametric form of the true distribution $\pp{P}$. Second, our 
setting involves a ``non-linear feature map'' into a high-dimensional space. For example, if we use the Gaussian RBF kernel (see \eqref{eq:rbf-kernel}), the mean function 
$\mu_{\pp{P}}$ lives in an infinite-dimensional space. As a result, higher moments of the distribution come into play and therefore one cannot adopt Stein's setting 
straightforwardly as it involves only the first moment. A direct generalization of James-Stein estimator to infinite-dimensional Hilbert space has been considered, for example, 
in \citet{Berger83:GP-Stein,Mandelbaum87:admissibility,Privault08:GP-malliavin}. In those works, the parameter to be estimated is assumed to be the mean of a 
Gaussian measure on the Hilbert space from which samples are drawn. In contrast, our setting involves samples that are drawn from $\pp{P}$ defined on an 
arbitrary measurable space, and not from a Gaussian measure defined on a Hilbert space. 
 
\subsection{Contributions} 
In the following, we present the main contributions of this work.
  \begin{enumerate}
    \item In Section~\ref{subsec:admissibility}, we propose kernel mean shrinkage estimators and show that these estimators can theoretically improve upon the 
    standard 
  empirical estimator, $\hat{\mu}_\pp{P}$ in terms of the mean squared error (see Theorem \ref{thm:main-result} and Proposition~\ref{prop:positive-part}), however, requiring the knowledge of the true kernel mean.
  We relax this condition in Section~\ref{subsec:consequence} (see Theorem \ref{thm:main}) where without requiring the knowledge of the true kernel mean, we construct shrinkage estimators
  that are \emph{uniformly} better (in mean squared error) than the empirical estimator over a class of distributions $\pspace$. For bounded continuous translation invariant kernels, we show
  that $\pspace$ reduces to a class of distributions whose characteristic functions have an $L^2$-norm bounded by a given constant. Through concrete choices for $\pspace$ in Examples~\ref{exm:gaussian} and \ref{exm:stein}, we discuss the implications of the proposed estimator.
  \item While the proposed estimators in Section~\ref{subsec:admissibility} and \ref{subsec:consequence} are theoretically interesting, they are not useful in practice as they require the knowledge of the 
  true data generating distribution. In Section~\ref{subsec:data-dep-alpha} (see Theorem~\ref{thm:conc}), we present a completely data-dependent estimator (say $\check{\mu}_\pp{P}$)---referred to as B-KMSE---that is $\sqrt{n}$-consistent and satisfies
  \begin{equation}\pp{E}\Vert\check{\mu}_\pp{P}-\mu_\pp{P}\Vert^2_\hbspace<\pp{E}\Vert\hat{\mu}_\pp{P}-\mu_\pp{P}\Vert^2_\hbspace+O(n^{-3/2})\,\,\text{as}\,\,\,n\rightarrow\infty.\label{Eq:exp-bound}\end{equation}
  \item In Section~\ref{sec:optimization}, we present a regularization interpretation for the proposed shrinkage estimator, wherein the shrinkage parameter is shown to be directly related to the regularization
  parameter. Based on this relation, we present an alternative approach to choosing the shrinkage parameter (different from the one proposed in Section~\ref{subsec:data-dep-alpha}) through
  leave-one-out cross-validation, and show that the corresponding shrinkage estimator (we refer to it as R-KMSE) is also $\sqrt{n}$-consistent and satisfies (\ref{Eq:exp-bound}).
\item The regularization perspective also sheds light on constructing new shrinkage estimators that incorporate specific information about the RKHS, 
  based on which we present a new $\sqrt{n}$-consistent shrinkage estimator---referred to as S-KMSE---in Section~\ref{sec:flexible} (see Theorem~\ref{thm:consistency-fkmse} and Remark~\ref{rem:compare})
  that takes into account spectral information of the covariance operator in RKHS. We establish the relation of S-KMSE to the problem of
  learning smooth operators \citep{Grunewalder13:SO} on $\hbspace$, and propose a leave-one-out cross-validation method to obtain a data-dependent shrinkage parameter.
  However, unlike B-KMSE and R-KMSE, it remains an open question as to whether S-KMSE with a data-dependent shrinkage parameter is consistent and satisfies an inequality similar 
  to (\ref{Eq:exp-bound}). The difficulty in answering these questions lies with the complex form of the estimator, $\tilde{\mu}_{\pp{P}}$ which is constructed so as to
  capture the spectral information of the covariance operator.
  \item In Section~\ref{sec:experiments}, we empirically evaluate the proposed shrinkage estimators of kernel mean on both synthetic data and several real-world scenarios including Parzen window classification, 
  density estimation and discriminative learning on distributions. The experimental results demonstrate the benefits of our shrinkage estimators over the standard one.
  \end{enumerate}
While a shorter version of this work already appeared in 
\citet{Muandet14:KMSE, Muandet14:Spectral}---particularly, the ideas in Sections~\ref{subsec:admissibility}, \ref{sec:optimization} and \ref{sec:flexible}---, this extended version provides a rigorous 
theoretical treatment (through Theorems~\ref{thm:main}, \ref{thm:conc}, \ref{thm:skmse-loocv-consistency}, \ref{thm:consistency-fkmse} and Proposition~\ref{thm:akmse-loocv} which are new) 
for the proposed estimators and also contains additional experimental results.

\section{Kernel Mean Shrinkage Estimators}
\label{sec:inadmissibility}
In this section, we first provide some definitions and notation that are used throughout the paper, following which we present a shrinkage estimator of $\mu_\pp{P}$. The rest of the
section presents various properties including the inadmissibility of the empirical estimator.

\subsection{Definitions \& Notation}
For
$a\triangleq(a_1,\ldots,a_d)\in\pp{R}^d$, $\Vert
a\Vert_2\triangleq\sqrt{\sum^d_{i=1}a^2_i}$. For a topological
space $\mathcal{X}$, $C(\mathcal{X})$ (\emph{resp.} $C_b(\mathcal{X})$) denotes the space of
all continuous (\emph{resp.} bounded continuous) functions on $\mathcal{X}$. For a
locally compact Hausdorff space $\mathcal{X}$, $f\in C(\mathcal{X})$ is said to
\emph{vanish at infinity} if for every $\epsilon > 0$ the set $\{x :
|f(x)|\ge\epsilon\}$ is compact. The class of all continuous $f$ on $\mathcal{X}$
which vanish at infinity is denoted as $C_0(\mathcal{X})$. $M_b(\mathcal{X})$ (\emph{resp.} $M^1_+(\inspace)$) denotes the set of all finite
Borel (\emph{resp.} probability) measures defined on $\inspace$. For $\mathcal{X}\subset\pp{R}^d$, 
$L^r(\mathcal{X})$ denotes the Banach space of $r$-power ($r\ge
1$) Lebesgue integrable functions. For $f\in L^r(\mathcal{X})$, $\Vert
f\Vert_{L^r}\triangleq\left(\int_\mathcal{X}|f(x)|^r\dd x\right)^{1/r}$ denotes
the $L^r$-norm of $f$ for $1\le r<\infty$. The Fourier transform of $f\in L^1(\pp{R}^d)$ is defined as $f^\wedge(\omega)\triangleq(2\pi)^{-d/2}\int_{\pp{R}^d}f(x) e^{-\sqrt{-1}\omega^{\top}x}\dd x,\,\omega\in\pp{R}^d$. The 
characteristic function of $\pp{P}\in M^1_+(\pp{R}^d)$ is defined as $\phi_\pp{P}(\omega)\triangleq \int e^{\sqrt{-1}\omega^{\top}x}\dd\pp{P}(x),\,\omega\in\pp{R}^d$.

An RKHS over a set $\inspace$ is a Hilbert space $\hbspace$ consisting of functions on $\inspace$ such that for each $x\in\inspace$ there is a function $k_x\in\hbspace$ with the property
\begin{equation}
  \label{eq:reproducing}
  \langle f,k_x\rangle_{\hbspace} = f(x), \quad \forall f\in\hbspace .
\end{equation}
The function $k_x(\cdot)\triangleq k(x,\cdot)$ is called the \emph{reproducing kernel} of $\hbspace$ and the equality \eqref{eq:reproducing} is called the \emph{reproducing property} of $\hbspace$. The space $\hbspace$ is endowed with inner product $\langle\cdot,\cdot\rangle_{\hbspace}$ and norm $\|\cdot\|_{\hbspace}$. Any symmetric and positive semi-definite kernel function $k:\inspace\times\inspace\rightarrow\rr$ uniquely determines an RKHS \citep{aronszajn50reproducing}. One of the most popular kernel functions is the Gaussian radial basis function (RBF) kernel on $\inspace=\rr^d$,
\begin{equation}
  \label{eq:rbf-kernel}
  k(x,y) = \exp\left(-\frac{\|x-y\|^2_2}{2\sigma^2}\right), \quad x,y\in\inspace,
\end{equation}
\noindent where $\|\cdot\|_2$ denotes the Euclidean norm and $\sigma>0$ is the bandwidth. 
For $x\in
\hbspace_1$ and $y\in \hbspace_2$, $x\otimes y$ 
denotes the tensor product of $x$ and $y$, and can be seen as an operator from $\hbspace_2$ to $\hbspace_1$ as
$(x\otimes y)z=x\langle y,z\rangle_{\hbspace_2}$ for any $z\in \hbspace_2$, where $\hbspace_1$ and
$\hbspace_2$ are Hilbert spaces.

We assume throughout the paper that we observe a sample $x_1,x_2,\ldots,x_n \in \inspace$ of size $n$ drawn independently and identically (i.i.d.) from some unknown distribution $\pp{P}$ defined 
over a separable topological space $\inspace$. Denote by $\mu$ and $\empmm$ the true kernel mean \eqref{eq:kernel-mean} and its empirical estimate \eqref{eq:empirical-mean} respectively. 
We remove the subscript for ease of notation, but we will use $\mu_{\pp{P}}$ (resp. $\hat{\mu}_{\pp{P}}$) and $\mu$ (resp. $\hat{\mu}$) interchangeably. For the well-definedness of
$\mu$ as a Bochner integral, throughout the paper we assume that $k$ is continuous and $\int_\mathcal{X}k(x,x)\,\dd\pp{P}(x)<\infty$ (see Footnote~\ref{fnote:1}). We measure the quality of an estimator $\tilde{\mu}\in \hbspace$ of $\mu$ by the 
risk function, $R:\hbspace\times\hbspace\rightarrow\rr$, $R(\mu,\tilde{\mu}) =  \ep\|\mu - \tilde{\mu}\|^2_{\hbspace}$, 
where $\ep$ denotes the expectation over the choice of random sample of size $n$ drawn i.i.d. from the distribution $\pp{P}$. 
When $\tilde{\mu}=\hat{\mu}$, for the ease of notation, we will use $\Delta$ to denote $R(\mu,\hat{\mu})$, which can be rewritten as
\begin{eqnarray}\label{eq:RiskKF}
 \Delta&{} ={}& \ep\|\empmm - \mu\|^2_{\hbspace} = \ep\|\empmm\|^2_\hbspace-\|\mu\|^2_\hbspace=\frac{1}{n^2}\sum^n_{i,j=1}\ep_{x_i,x_j} k(x_i,x_j)-\Vert\mu\Vert^2_\hbspace\nonumber\\
    &{}={}& \frac{1}{n^2}\sum^n_{i=1}\ep_{x_i} k(x_i,x_i)+\frac{1}{n^2}\sum^n_{i\ne j}\ep_{x_i,x_j} k(x_i,x_j)-\Vert\mu\Vert^2_\hbspace\nonumber\\
    &{}={}& \frac{1}{n}\left(\ep_xk(x,x) -
      \ep_{x,\tilde{x}}k(x,\tilde{x})\right) ,
  \end{eqnarray}
  \noindent where $\Vert\mu\Vert^2_\hbspace=\ep_{x,\tilde{x}}[k(x,\tilde{x})] \triangleq \ep_{x\sim\pp{P}}[\ep_{\tilde{x}\sim\pp{P}}[k(x,\tilde{x})]]$ with $x$ and $\tilde{x}$ being independent copies. 
  An estimator $\hat{\mu}_1$ is said to be \emph{as good as} $\hat{\mu}_2$ if $R(\mu,\hat{\mu}_1) \leq R(\mu,\hat{\mu}_2)$ for any $\pp{P}$, and is \emph{better than} $\hat{\mu}_2$ 
  if it is as good as $\hat{\mu}_2$ and $R(\mu,\hat{\mu}_1) < R(\mu,\hat{\mu}_2)$ for at least one $\pp{P}$. An estimator is said to be \emph{inadmissible} if there exists a better estimator.

\subsection{Shrinkage Estimation of $\mu_\pp{P}$} \label{subsec:admissibility}

We propose the following kernel mean estimator
\begin{equation}
  \label{eq:shrinkage-estimator}
  \hat{\mu}_{\alpha} \triangleq \alpha f^* + (1-\alpha)\hat{\mu}
\end{equation}
\noindent where $\alpha\ge 0$ and $f^*$ is a fixed, but arbitrary function in $\hbspace$. Basically, it is a shrinkage estimator that shrinks the empirical estimator toward a function $f^*$ by an amount specified by $\alpha$. The choice of $f^*$ can be arbitrary, but we will assume 
that $f^*$ is chosen independent of the sample. If $\alpha=0$, the estimator $\hat{\mu}_{\alpha}$ reduces to the empirical estimator $\hat{\mu}$. We denote by $\Delta_\alpha$ the 
risk of the shrinkage estimator in \eqref{eq:shrinkage-estimator}, \ie, $\Delta_\alpha \triangleq R(\mu,\hat{\mu}_\alpha)$.

Our first theorem asserts that the shrinkage estimator $\hat{\mu}_{\alpha}$ achieves smaller risk than that of the empirical estimator $\hat{\mu}$ given an appropriate choice of $\alpha$, regardless of the function $f^*$.
\begin{theorem}
  \label{thm:main-result}
Let $\mathcal{X}$ be a separable topological space. Then for all distributions $\pp{P}$ and continuous kernel $k$ satisfying $\int k(x,x)\,\dd\pp{P}(x)<\infty$, $\Delta_\alpha < \Delta$ if and only if 
  \begin{equation}
    \label{eq:opt-alpha}
    \alpha \in \left(0,\frac{2\Delta}{\Delta + \|f^* - \mu\|^2_{\hbspace}}\right).
  \end{equation}
  In particular, $\arg\min_{\alpha\in\pp{R}}(\Delta_\alpha-\Delta)$ is unique and is given by $\alpha_*\triangleq \frac{\Delta}{\Delta+\Vert f^*-\mu\Vert^2_\hbspace}$.
\end{theorem}

\begin{proof}
Note that
$$
\Delta_\alpha=\pp{E}\|\hat{\mu}_\alpha - \mu\|^2_{\hbspace}= \left\| \mathbb{E}[\hat{\mu}_\alpha] - \mu\right\|^2_{\hbspace} + \mathbb{E}\left\|\hat{\mu}_\alpha - \mathbb{E}\hat{\mu}_\alpha\right\|^2_{\hbspace}
  = \Vert\mathrm{Bias}(\hat{\mu}_\alpha)\Vert^2_\hbspace + \mathrm{Var}(\hat{\mu}_\alpha),$$
where $$\mathrm{Bias}(\hat{\mu}_\alpha)= \mathbb{E}[\hat{\mu}_\alpha] - \mu=\mathbb{E}[\alpha f^* + (1-\alpha)\hat{\mu}] - \mu=\alpha(f^*-\mu)$$ and 
$$\mathrm{Var}(\hat{\mu}_\alpha)=(1-\alpha)^2\mathbb{E}\left\|\hat{\mu} - \mu\right\|^2_{\hbspace}=(1-\alpha)^2\Delta.$$
Therefore,
\begin{equation}
  \Delta_\alpha = \alpha^2\left\| f^* - \mu \right\|^2_{\hbspace} + (1-\alpha)^2\Delta,
\end{equation}
\ie, $\Delta_\alpha - \Delta = \alpha^2\left[\Delta +
      \|f^*-\mu\|^2_{\hbspace}\right] - 2\alpha\Delta$. 
This is clearly negative if and only if (\ref{eq:opt-alpha}) holds
and is uniquely minimized at $\alpha_* \triangleq \frac{\Delta}{\Delta + \|f^* - \mu\|^2_{\hbspace}}$.\vspace{-5mm}
\end{proof}
\begin{remark}
  \label{rem:remark1} 
  \begin{enumerate}[label=(\roman{*})]
  \item The shrinkage estimator always improves upon the standard one regardless of the direction of shrinkage, as specified by the choice of $f^*$. In other words, there exists a wide class of kernel mean estimators that achieve smaller risk than the standard one.
    
  \item The range of $\alpha$ depends on the choice of $f^*$. The further $f^*$ is from $\mu$, the smaller the range of $\alpha$ becomes. Thus, the shrinkage gets smaller if $f^*$ is chosen such that it is far from the true kernel mean. This effect is akin to James-Stein estimator.
    
  \item From \eqref{eq:opt-alpha}, since $0 < \alpha < 2$, \ie, $0 < (1-\alpha)^2 < 1$, it follows that $\mathrm{Var}(\hat{\mu}_\alpha) < \mathrm{Var}(\hat{\mu})=\Delta$, \ie, 
the shrinkage estimator always improves upon the empirical estimator in terms of the variance. Further improvement can be gained by reducing the bias by incorporating the prior knowledge about the location of $\mu$ via $f^*$. 
This implies that we can potentially gain ``twice'' by adopting the shrinkage estimator: by reducing variance of the estimator and by incorporating prior knowledge in choosing $f^*$ such that it is close to the true kernel mean. 
  \end{enumerate}
\end{remark}
While Theorem~\ref{thm:main-result} shows $\hat{\mu}$ to be inadmissible by providing a family of estimators that are better than $\hat{\mu}$, the result is not useful
as all these estimators require the knowledge of $\mu$ (which is the parameter of interest) through the range of $\alpha$ given in (\ref{eq:opt-alpha}). In Section~\ref{subsec:consequence}, we 
investigate Theorem~\ref{thm:main-result} and show that $\hat{\mu}_\alpha$ can be constructed under some weak assumptions on $\pp{P}$, without requiring the 
knowledge of $\mu$. 
From \eqref{eq:opt-alpha}, the existence of positive $\alpha$ is guaranteed if and only if the risk of the empirical estimator is non-zero. Under some 
assumptions on $k$, the following result shows that $\Delta = 0$ if and only if the distribution $\pp{P}$ is a Dirac distribution, \ie, the distribution $\pp{P}$ is a point mass. 
This result ensures, in many non-trivial cases, a non-empty range of $\alpha$ for which $\Delta_\alpha - \Delta < 0$.

\begin{proposition}
  \label{prop:zero-risk}
Let $k(x,y)=\psi(x-y),\,\,x,y\in\mathbb{R}^d$ be a characteristic kernel where $\psi\in C_b(\pp{R}^d)$ is positive definite. Then $\Delta = 0$ if and only if $\mathbb{P} = \delta_x$ for some $x\in\mathbb{R}^d$.
\end{proposition}
\begin{proof} 
See Section~\ref{proof:prop:zero-risk}.
\end{proof}

\subsubsection{Positive-part Shrinkage Estimator}
\label{sec:positive-part}

Similar to James-Stein estimator, we can show that the positive-part version of $\hat{\mu}_\alpha$ also outperforms $\hat{\mu}$, where the positive-part estimator is defined by
\begin{equation}
  \label{eq:positive-part}
  \hat{\mu}_\alpha^+ \triangleq \alpha f^* + (1-\alpha)_+\hat{\mu}
\end{equation}
\noindent with $(a)_{+}\triangleq a$ if $a>0$ and zero otherwise. Equation \eqref{eq:positive-part} can be rewritten as
\begin{equation}
  \label{eq:posdec}
  \hat{\mu}_\alpha^+ = 
  \begin{cases} 
    \alpha f^* + (1-\alpha)\hat{\mu}, &  \quad 0\leq \alpha \leq 1 \\
    \alpha f^* & \quad 1 < \alpha < 2 .
  \end{cases} 
\end{equation}
Let $\Delta_\alpha^+ \triangleq \mathbb{E}\|\hat{\mu}_\alpha^+ - \mu\|^2_{\hbspace}$ be the risk of the positive-part estimator. Then, the following result shows that 
$\Delta_\alpha^+ \leq \Delta_\alpha$, given that $\alpha$ satisfies \eqref{eq:opt-alpha}. 
\begin{proposition}
  \label{prop:positive-part}
  For any $\alpha$ satisfying \eqref{eq:opt-alpha}, we have that $\Delta_\alpha^+ \leq \Delta_\alpha < \Delta$.
\end{proposition}
\begin{proof}
  According to \eqref{eq:posdec}, we decompose the proof into two parts. First, if $0\leq\alpha \leq 1$, $\hat{\mu}_\alpha$ and $\hat{\mu}^+_\alpha$ behave exactly the same. Thus, $\Delta_\alpha^+ = \Delta_\alpha$. On the other hand, when $1 < \alpha < 2$, the bias-variance decomposition of these estimators yields
  \begin{eqnarray*}
    \Delta_\alpha = \alpha^2\| f^* - \mu \|^2_{\hbspace} + (1-\alpha)^2\mathbb{E}\|\hat{\mu} - \mu\|^2_{\hbspace}\quad \text{and} \quad
    \Delta_\alpha^+ = \alpha^2\| f^* - \mu \|^2_{\hbspace}.
\end{eqnarray*}
It is easy to see that $\Delta_\alpha^+ < \Delta_\alpha$ when $1 < \alpha < 2$. This concludes the proof. 
\end{proof}
Proposition \ref{prop:positive-part} implies that, when estimating $\alpha$, it is better to restrict the value of $\alpha$ to be smaller than 1, although it can be greater than 1, as suggested by Theorem \ref{thm:main-result}. The reason is that if $0\leq \alpha \leq 1$, the bias is an increasing function of $\alpha$, whereas the variance is a decreasing function of $\alpha$. On the other hand, if $\alpha > 1$, both bias and variance become increasing functions of $\alpha$. 
We will see later in Section \ref{sec:optimization} that $\hat{\mu}_\alpha$ and $\hat{\mu}^+_\alpha$ 
can be obtained naturally as a solution to a regularized regression problem.

\subsection{Consequences of Theorem \ref{thm:main-result}}\label{subsec:consequence}

As mentioned before, while Theorem \ref{thm:main-result} is interesting from the perspective of showing that the shrinkage estimator, $\hat{\mu}_\alpha$
performs better---in the mean squared sense---than the empirical estimator, it unfortunately relies on the fact that $\mu_\pp{P}$ (\ie, the object of interest) is known, which makes
$\hat{\mu}_\alpha$ uninteresting. Instead of knowing $\mu_\pp{P}$, which requires the knowledge of $\pp{P}$, in this section, we show that a
shrinkage estimator can be constructed that performs better than the empirical estimator, uniformly over a class of probability distributions. To this end, we introduce the notion of an
oracle upper bound. 

Let $\pspace$ be a class of probability distributions $\pp{P}$ defined on a measurable space $\mathcal{X}$. We define an oracle upper bound as
  \begin{equation*}
    U_{k,\pspace} \triangleq \inf_{\pp{P}\in\pspace} \frac{2\Delta}{\Delta + \|f^* - \mu\|_{\hbspace}^2} .
  \end{equation*}
It follows immediately from Theorem \ref{thm:main-result} and the definition of $U_{k,\pspace}$ that if $U_{k,\pspace}\ne 0$, then for any $\alpha\in (0,U_{k,\pspace})$, $\Delta_\alpha-\Delta < 0$ holds ``uniformly'' for all $\pp{P}\in\pspace$.
%
Note that by virtue of Proposition \ref{prop:zero-risk}, the class $\pspace$ cannot contain the Dirac measure $\delta_x$ (for any $x\in\pp{R}^d$) if the kernel $k$ is translation invariant and characteristic on $\pp{R}^d$. Below we give concrete examples of $\pspace$ for which 
$U_{k,\pspace}\ne 0$ so that the above uniformity statement holds.
In particular, we show in Theorem~\ref{thm:main} below that for $\mathcal{X}=\pp{R}^d$, 
if a non-trivial bound on the $L^2$-norm of the characteristic function of $\pp{P}$ is known, it is possible to construct shrinkage estimators that are better (in mean squared error) 
than the empirical average. In such a case, unlike in Theorem~\ref{thm:main-result}, $\alpha$ does not depend on the individual distribution $\pp{P}$, but only on an upper bound associated with a class $\pspace$.


\begin{theorem}\label{thm:main}
Let $k(x,y)=\psi(x-y),\,x,y\in\rr^d$ with $\psi\in C_b(\rr^d)\cap L^1(\rr^d)$ and $\psi$ is a positive definite function with $\psi(0)>0$. For a given constant $A\in (0,1)$, let $A_\psi:=\frac{A(2\pi)^{d/2}\psi(0)}{\Vert \psi\Vert_{L_1}}$ and
\begin{equation*}
  \pspace_{k,A} \triangleq \left\{\pp{P}\in M^1_+(\rr^d):\Vert \phi_{\pp{P}}\Vert_{L^2}\le \sqrt{A_\psi}\right\} ,
\end{equation*}
where $\phi_{\pp{P}}$ denotes the characteristic function of $\pp{P}$. Then for all $\pp{P}\in\pspace_{k,A}$, $\Delta_\alpha<\Delta$ if 
\begin{equation*}
  \alpha\in \left(0,\frac{2(1-A)}{1+(n-1)A+\frac{n\Vert f^*\Vert^2_{\hbspace}}{\psi(0)}+\frac{2n\sqrt{A}\Vert f^*\Vert_{\hbspace}}{\sqrt{\psi(0)}}}\right].
\end{equation*}
\end{theorem}
\begin{proof}
By Theorem~\ref{thm:main-result}, we have that
\begin{equation}
  \Delta_\alpha<\Delta,\,\,\forall\,\,\alpha\in\left(0,\frac{2\Delta }{\Delta+\Vert f^*-\mu\Vert^2_{\hbspace}} \right).\label{Eq:optima}\end{equation}
Consider
\begin{eqnarray}
\frac{\Delta}{\Delta+\Vert f^*-\mu\Vert^2_{\hbspace}}&{}={}&\frac{\pp{E}_xk(x,x)-\pp{E}_{x,\tilde{x}}k(x,\tilde{x})}{\pp{E}_xk(x,x)-\pp{E}_{x,\tilde{x}} k(x,\tilde{x})+n\Vert f^*-\mu\Vert^2_{\hbspace}}\nonumber\\
&{}\stackrel{(\dagger)}{=}{}&\frac{1-\frac{\pp{E}_{x,\tilde{x}}k(x,\tilde{x})}{\pp{E}_xk(x,x)}}{1+(n-1)\frac{\pp{E}_{x,\tilde{x}}k(x,\tilde{x})}{\pp{E}_xk(x,x)}+\frac{n\Vert f^*\Vert^2_{\hbspace}}{\pp{E}_xk(x,x)}-\frac{2n\langle f^*,\mu\rangle_{\hbspace}}{\pp{E}_xk(x,x)}}\nonumber\\
&{}\ge{}&\frac{1-\frac{\pp{E}_{x,\tilde{x}}k(x,\tilde{x})}{\pp{E}_xk(x,x)}}{1+(n-1)\frac{\pp{E}_{x,\tilde{x}}k(x,\tilde{x})}{\pp{E}_xk(x,x)}+\frac{n\Vert f^*\Vert^2_{\hbspace}}{\pp{E}_xk(x,x)}+\frac{2n\Vert f^*\Vert_{\hbspace}\sqrt{\pp{E}_{x,\tilde{x}}k(x,\tilde{x})}}{\pp{E}_xk(x,x)}},\label{Eq:bound}
\end{eqnarray}
where the division by $\pp{E}_xk(x,x)$ in ($\dagger$) is valid since $\pp{E}_xk(x,x)=\psi(0)>0$.
Note that the numerator in the r.h.s.~of (\ref{Eq:bound}) is non-negative since $$\pp{E}_{x,\tilde{x}}k(x,\tilde{x})\le \pp{E}_x\sqrt{k(x,x)}\pp{E}_{\tilde{x}}\sqrt{k(\tilde{x},\tilde{x})}\le \pp{E}_xk(x,x)$$ 
with equality holding if and only if $\pp{P}=\delta_y$ for some $y\in\pp{R}^d$ (see Proposition~\ref{prop:zero-risk}). However, for any $A\in(0,1)$ and $y\in\pp{R}^d$, it is easy to
verify that $\delta_y\notin\pspace_{k,A}$, which implies the numerator in fact positive. The
denominator is clearly positive since $\mathbb{E}_{x,\tilde{x}}k(x,\tilde{x})\ge 0$ and therefore the r.h.s.~of (\ref{Eq:bound}) is positive. Also note that 
\begin{eqnarray}
\pp{E}_{x,\tilde{x}}k(x,\tilde{x})&{}={}&\int\int \psi(x-y)\dd\pp{P}(x)\dd\pp{P}(y)\stackrel{(\ast)}{=}{}\int
|\phi_{\pp{P}}(\omega)|^2\psi^\wedge(\omega)\,\dd\omega\nonumber\\
&{}\le{}&
\sup_{\omega\in\rr^d}\psi^\wedge(\omega)\Vert
\phi_{\pp{P}}\Vert^2_{L_2}\leq(2\pi)^{-d/2}\Vert \psi\Vert_{L_1}\Vert
\phi_{\pp{P}}\Vert^2_{L_2},
\end{eqnarray}
where $\psi^\wedge$ is the Fourier transform of $\psi$ and $(\ast)$ follows---see (16) in the proof of Proposition 5 in \citet{Sriperumbudur10b:Universal}---by invoking Bochner's theorem \cite[Theorem 6.6]{Wendland-05}, which states that $\psi$ is 
Fourier transform of a non-negative finite Borel measure with density $(2\pi)^{-d/2}\psi^\wedge$, \ie, $\psi(x)=(2\pi)^{-d/2}\int e^{-ix^{\top}\omega}\psi^\wedge(\omega)\,\dd\omega$, $x\in\mathbb{R}^d$.
As $\pp{E}_xk(x,x)=\psi(0)$, we have that 
$$
\frac{\pp{E}_{x,\tilde{x}}k(x,\tilde{x})}{\pp{E}_xk(x,x)}\le\frac{A\Vert \phi_\pp{P}\Vert^2_{L^2}}{A_\psi}
$$ 
and therefore for any
$\pp{P}\in\pspace_{k,A}$, $\frac{\pp{E}_{x,\tilde{x}}k(x,\tilde{x})}{\pp{E}_xk(x,x)}\le A$.
Using this in \eqref{Eq:bound} and combining it with \eqref{Eq:optima} yields the result.\vspace{-5mm}
\end{proof}
\begin{remark} 
  \begin{enumerate}[label=(\roman{*})]
  \item Theorem~\ref{thm:main} shows that for any $\pp{P}\in\pspace_{k,A}$, it is possible to construct a shrinkage estimator that dominates the empirical estimator, \ie, the shrinkage estimator has a strictly smaller risk than that of the empirical estimator.
  \item Suppose that $\pp{P}$ has a density, denoted by $p$, with respect to the Lebesgue measure and $\phi_{\pp{P}}\in L^2(\pp{R}^d)$. By Plancherel's theorem, 
    $p\in L^2(\pp{R}^d)$ as $\Vert p\Vert_{L_2}=\Vert \phi_{\pp{P}}\Vert_{L_2}$, which means that $\pspace_{k,A}$ includes distributions with square integrable densities 
    (note that in general not every $p$ is square integrable). Since 
    $$\Vert \phi_{\pp{P}}\Vert^2_{L_2}=\int |\phi_\pp{P}(\omega)|^2\,\dd\omega\le \sup_{\omega\in\pp{R}^d}|\phi_\pp{P}(\omega)|\int |\phi_\pp{P}(\omega)|\,\dd\omega= \Vert \phi_\pp{P}\Vert_{L_1},$$ 
where we used the fact that $\sup_{\omega\in\mathbb{R}^d}|\phi_\pp{P}(\omega)|=1$, it is easy to check that 
    $$\left\{\pp{P}\in M^1_+(\rr^d):\Vert \phi_{\pp{P}}\Vert_{L^1}\le\frac{A(2\pi)^{d/2}\psi(0)}{\Vert \psi\Vert_{L_1}}\right\}\subset \pspace_{k,A}.$$ This means bounded densities belong to $\pspace_{k,A}$ as $\phi_\pp{P}\in L^1(\rr^d)$ implies that 
    $\pp{P}$ has a density, $p\in C_0(\rr^d)$. Moreover, it is easy to check that larger the value of $A$, larger is the class $\pspace_{k,A}$ and smaller is the range of $\alpha$ for which $\Delta_\alpha<\Delta$ and vice-versa.
\end{enumerate}
\end{remark}

In the following, we present some concrete examples to elucidate Theorem~\ref{thm:main}.
\begin{example}[Gaussian kernel and Gaussian distribution]
  Define 
  \begin{equation*}
    \nbspace \triangleq \left\{\pp{P}\in M^1_+(\rr^d)\,\Big{|}\,\dd\pp{P}(x)=\frac{1}{(2\pi\sigma^2)^{d/2}}e^{ -\frac { \Vert x-\theta\Vert^2_2}{2\sigma^2}}\,\dd x,\,\,\theta\in\pp{R}^d,\,\sigma>0\right\},
\end{equation*}
\noindent where $\psi(x)=e^{-\Vert x\Vert^2_2/2\tau^2},\,x\in\pp{R}^d$ and $\tau>0$. For $\pp{P}\in\nbspace$, it is easy to verify that
\begin{equation*}
  \phi_{\pp{P}}(\omega)=e^{\sqrt{-1}\theta^{\top}\omega-\frac{1}{2}\sigma^2\Vert\omega\Vert^2_2},\,\omega\in\pp{R}^d\,\,\,\text{and}\,\,\,\Vert \phi_{\pp{P}}\Vert^2_{L_2}=\int e^{-\sigma^2\Vert \omega\Vert^2_2}\,\dd\omega=(\pi/\sigma^2)^{d/2}.
\end{equation*}
Also, $\Vert\psi\Vert_{L_1}=(2\pi\tau^2)^{d/2}$. Therefore, for $\pspace_{k,A} \triangleq \{\pp{P}\in\nbspace : \sigma^2\ge\pi\tau^2/A^{2/d}\}$, assuming $f^*=0$, we obtain the result in Theorem~\ref{thm:main}, \ie,
the result in Theorem~\ref{thm:main} holds for all Gaussian distributions that are smoother (having larger variance) than that of the kernel.
\label{exm:gaussian}
\end{example}

\begin{example}[Linear kernel] 
\label{exm:stein}
  Suppose $f^*=0$ and $k(x,y)=x^{\top}y$. While the setting of Theorem~\ref{thm:main} does not fit this choice of $k$, an inspection of its proof shows that it is possible to
  construct a shrinkage estimator that improves upon $\mu_\pp{P}$ for an appropriate class of distributions. To this end, let $\vartheta$ and $\Sigma$ represent the mean vector and covariance matrix of a distribution $\pp{P}$ defined on $\pp{R}^d$. Then it
  is easy to check that $\frac{\pp{E}_{x,\tilde{x}}k(x,\tilde{x})}{\pp{E}_xk(x,x)}=\frac{\Vert \vartheta\Vert^2_2}{\emph{trace}(\Sigma)+\Vert \vartheta\Vert^2_2}$ and therefore 
  for a given $A\in (0,1)$, define 
  \begin{equation*}
    \pspace_{k,A} \triangleq \left\{\pp{P}\in M^1_+(\rr^d) \,{\Big|}\, \frac{\Vert\vartheta\Vert^2_2}{\emph{trace}(\Sigma)}\le\frac{A}{1-A}\right\}.
  \end{equation*}
  From \eqref{Eq:optima} and \eqref{Eq:bound}, it is clear that for any $\pp{P}\in \pspace_{k,A}$, $\Delta_\alpha<\Delta$ if $\alpha\in \left(0,\frac{2(1-A)}{1+(n-1)A}\right]$. 
  Note that this choice of kernel yields the setting similar to classical James-Stein estimation. In James-Stein estimation, $\pp{P}\in \nbspace$ (see Example~\ref{exm:gaussian} for the 
  definition of $\nbspace$) and $\vartheta$ is estimated as $(1-\tilde{\alpha})\hat{\vartheta}$---which improves upon $\hat{\vartheta}$---where $\tilde{\alpha}$ depends on the sample $(x_i)^n_{i=1}$ 
  and $\hat{\vartheta}$ is the sample mean. In our case, for all $\pp{P}\in \pspace_{k,A}=\left\{\pp{P}\in\nbspace\,:\,\Vert\vartheta\Vert_2\le\sigma\sqrt{\frac{dA}{1-A}}\right\}$, $\Delta_\alpha<\Delta$ if $\alpha\in \left(0,\frac{2(1-A)}{1+(n-1)A}\right]$. 
  In addition, in contrast to the James-stein estimator which improves upon the empirical estimator (\ie, sample mean) for only $d\ge 3$, we note here that the proposed estimator  
  improves for any $d$ as long as $\pp{P}\in \pspace_{k,A}$. On the other hand, the proposed estimator requires some knowledge about  
  the distribution (particularly a bound on $\Vert\vartheta\Vert_2$), which the James-Stein estimator does not (see Section \ref{sec:js-connection} for more details).

\end{example}

\subsection{Data-Dependent Shrinkage Parameter}\label{subsec:data-dep-alpha}

The discussion so far showed that the shrinkage estimator in (\ref{eq:shrinkage-estimator}) performs better than the empirical estimator if the data generating distribution
satisfies a certain mild condition (see Theorem~\ref{thm:main}; Examples~\ref{exm:gaussian} and \ref{exm:stein}). However, since this condition is usually not checkable in practice, the shrinkage estimator lacks applicability. 
In this section, we present a completely data driven shrinkage estimator by estimating the shrinkage parameter $\alpha$ from data so that the estimator does not require any knowledge of the data generating distribution.

Since the maximal difference between $\Delta_\alpha$ and $\Delta$ occurs at $\alpha_\ast$ (see Theorem~\ref{thm:main-result}), given an i.i.d.~sample 
$X=\{x_1,x_2,\ldots,x_n\}$ from $\pp{P}$, we propose to estimate $\mu$ using $\hat{\mu}_{\tilde{\alpha}}=(1-\tilde{\alpha})\hat{\mu}$ (\ie, assuming $f^\ast=0$)
where $\tilde{\alpha}$ is an estimator of $\alpha_* = \Delta/(\Delta + \|\mu\|^2_{\hbspace})$ given by 
\begin{equation}
  \label{eq:empirical-alpha}
  \tilde{\alpha} = \frac{\hat{\Delta}}{\hat{\Delta} + \|\hat{\mu}\|^2_{\hbspace}} ,
\end{equation}
with $\hat{\Delta}$ and $\hat{\mu}$ being the empirical versions of $\Delta$ and $\mu$, respectively (see Theorem~\ref{thm:conc} for precise definitions). The following result shows that $\tilde{\alpha}$ is a $n\sqrt{n}$-consistent 
estimator of $\alpha_\ast$ and $\Vert \hat{\mu}_{\tilde{\alpha}}-\mu\Vert_\hbspace$ concentrates around $\Vert\hat{\mu}_{\alpha_\ast}-\mu\Vert_\hbspace$. In addition, we 
show that $$\Delta_{\alpha_\ast}\le \Delta_{\tilde{\alpha}}\le \Delta_{\alpha_\ast}+O(n^{-3/2})\,\,\,\text{as}\,\,\,n\rightarrow\infty,$$ 
which means the performance of $\hat{\mu}_{\tilde{\alpha}}$ is similar to that 
of the best estimator (in mean squared sense) of the form $\hat{\mu}_\alpha$. In what follows, we will call the estimator $\hat{\mu}_{\tilde{\alpha}}$ an \emph{empirical-bound kernel mean shrinkage estimator (B-KMSE)}.

\begin{theorem}
  \label{thm:conc}
  Suppose $n\ge 2$ and $f^\ast=0$. Let $k$ be a continuous kernel on a separable topological space $\mathcal{X}$ satisfying $\int_\mathcal{X}k(x,x)\dd\pp{P}(x)<\infty$. 
  Define 
\begin{equation*} \hat{\Delta} \triangleq \frac{\hat{\pp{E}}k(x,x)-\hat{\pp{E}}k(x,\tilde{x})}{n} \quad \text{and} \quad \|\hat{\mu}\|^2_{\hbspace} \triangleq \frac{1}{n^2}\sum_{i,j=1}^n k(x_i,x_j) 
\end{equation*}
\noindent where $\hat{\pp{E}}k(x,x)\triangleq\frac{1}{n}\sum^n_{i=1}k(x_i,x_i)$ and
$\hat{\pp{E}}k(x,\tilde{x})\triangleq\frac{1}{n(n-1)}\sum^n_{i\ne j}k(x_i,x_j)$ are the
empirical estimators of $\pp{E}_xk(x,x)$ and $\pp{E}_{x,\tilde{x}}k(x,\tilde{x})$ respectively. Assume there exist finite constants $\kappa_1>0$, $\kappa_2>0$, $\sigma_1>0$ and $\sigma_2>0$ such that
\begin{equation}
\pp{E}\Vert k(\cdot,x)-\mu\Vert^m_{\hbspace}\le\frac{m!}{2}\sigma^2_1\kappa^{m-2}_1,\,\,\,\,\forall\,m\ge 2.\label{Eq:bern-condition-1}
\end{equation}
and 
\begin{equation}
\pp{E}|k(x,x)-\pp{E}_xk(x,x)|^m\le\frac{m!}{2}\sigma^2_2\kappa^{m-2}_2,\,\,\,\,\forall\,m\ge 2.\label{Eq:bern-condition-2}
\end{equation}
Then 
\begin{equation}|\tilde{\alpha}-\alpha_\ast|=O_{\pp{P}}(n^{-3/2})\,\,\,\text{and}\,\,\,\Big|\Vert \hat{\mu}_{\tilde{\alpha}}-\mu\Vert_{\hbspace}-\Vert \hat{\mu}_{\alpha_\ast}-\mu\Vert_\hbspace\Big|= O_{\pp{P}}(n^{-3/2})\nonumber
\end{equation}
as $n\rightarrow\infty$. In particular, \begin{equation}\min_\alpha \pp{E}\Vert \hat{\mu}_\alpha-\mu\Vert^2_\hbspace\le \pp{E}\Vert \hat{\mu}_{\tilde{\alpha}}-\mu\Vert^2_\hbspace\le \min_\alpha \pp{E}\Vert \hat{\mu}_\alpha-\mu\Vert^2_\hbspace+O(n^{-3/2})\label{Eq:oracle}\end{equation}
as $n\rightarrow\infty$.
\end{theorem}
\begin{proof}
See Section~\ref{proof:thm:conc}.\vspace{-7mm}
\end{proof}
\begin{remark}\label{rem:consistent}
  \begin{enumerate}[label=(\roman{*})]
  \item $\hat{\mu}_{\tilde{\alpha}}$ is a $\sqrt{n}$-consistent estimator of $\mu$. This follows from 
\begin{eqnarray}\Vert \hat{\mu}_{\tilde{\alpha}}-\mu\Vert_\hbspace&{}\le {}&
\Vert \hat{\mu}_{\alpha_\ast}-\mu\Vert_\hbspace+O_{\pp{P}}(n^{-3/2})\nonumber\\
&{}\le{}& (1-\alpha_\ast)\Vert \hat{\mu}-\mu\Vert_\hbspace+\alpha_\ast\Vert\mu\Vert_\hbspace+O_{\pp{P}}(n^{-3/2})\nonumber
\end{eqnarray}
with $$\alpha_\ast=\frac{\Delta}{\Delta+\Vert \mu\Vert^2_\hbspace}=\frac{\pp{E}_xk(x,x)-\pp{E}_{x,\tilde{x}}k(x,\tilde{x})}{\pp{E}_xk(x,x)+(n-1)\pp{E}_{x,\tilde{x}}k(x,\tilde{x})}=O(n^{-1})$$ as $n\rightarrow \infty$. 
Using \eqref{Eq:bound-2}, we obtain $\Vert \hat{\mu}_{\tilde{\alpha}}-\mu\Vert_\hbspace=O_{\pp{P}}(n^{-1/2})$ as $n\rightarrow\infty$, which implies that $\hat{\mu}_{\tilde{\alpha}}$ is a $\sqrt{n}$-consistent 
estimator of $\mu$.

\item Equation \eqref{Eq:oracle} shows that $\Delta_{\tilde{\alpha}}\le \Delta_{\alpha_\ast}+O(n^{-3/2})$ where $\Delta_{\alpha_\ast}<\Delta$ (see Theorem~\ref{thm:main-result}) and therefore 
for any $\pp{P}$ satisfying \eqref{Eq:bern-condition-1} and \eqref{Eq:bern-condition-2}, $\Delta_{\tilde{\alpha}}<\Delta+O(n^{-3/2})$ as $n\rightarrow \infty$.

\item Suppose the kernel is bounded, \ie, $\sup_{x,y\in\mathcal{X}}|k(x,y)|\le\kappa<\infty$. Then it is easy
to verify that \eqref{Eq:bern-condition-1} and \eqref{Eq:bern-condition-2} hold with $\sigma_1=\sqrt{\kappa}$, $\kappa_1=2\sqrt{\kappa}$, $\sigma_2=\kappa$ and $\kappa_2=2\kappa$ and therefore
the claims in Theorem~\ref{thm:conc} hold for bounded kernels.

\item For $k(x,y)=x^{\top}y$, we have
$$\pp{E}\Vert k(\cdot,x)-\mu\Vert^m_\hbspace=\pp{E}\left(\Vert k(\cdot,x)-\mu\Vert^2_\hbspace\right)^{m/2}=\pp{E}\left(\Vert x-\pp{E}_xx\Vert^2_2\right)^{m/2}=\pp{E}\Vert x-\pp{E}_xx\Vert^m_2$$
and $$\pp{E}|k(x,x)-\pp{E}_xk(x,x)|^m=\pp{E}|\Vert x\Vert^2_2-\pp{E}_x\Vert x\Vert^2_2|^m.$$ The conditions in \eqref{Eq:bern-condition-1} and \eqref{Eq:bern-condition-2}
hold for $\pp{P}\in\nbspace$ where $\nbspace$ is defined in Example~\ref{exm:gaussian}. With $\pp{P}\in\nbspace$ and $k(x,y)=x^{\top}y$,
the problem of estimating $\mu$ reduces to estimating $\theta$, for which we have presented a James-Stein-like estimator, $\hat{\mu}_{\tilde{\alpha}}$ that satisfies the oracle inequality in \eqref{Eq:oracle}.

\item While the moment conditions in \eqref{Eq:bern-condition-1} and \eqref{Eq:bern-condition-2} are obviously satisfied by bounded kernels, for unbounded kernels, these conditions are quite stringent as they require all the higher moments to exist. These conditions can be weakened and the proof of Theorem~\ref{thm:conc} can be carried out using Chebyshev inequality instead of Bernstein's inequality but at the cost of a slow rate in \eqref{Eq:oracle}.
\end{enumerate}
\end{remark}
\subsection{Connection to James-Stein Estimator} 
\label{sec:js-connection} 
In this section, we explore the connection of our proposed estimator in (\ref{eq:shrinkage-estimator}) to the James-Stein estimator. Recall that Stein's setting deals with estimating the mean of the 
Gaussian distribution $\mathcal{N}(\theta,\sigma^2 \id_d)$, which can be viewed as a special case of kernel mean estimation when we restrict to the class of distributions 
$\pspace \triangleq \{\mathcal{N}(\theta, \sigma^2\id_d) \,|\, \theta\in\rr^d \}$ and a linear kernel $k(x,y)=x^{\top} y,\,x,y\in\pp{R}^d$ (see Example~\ref{exm:stein}).
In this case, it is easy to verify that $\Delta = d\sigma^2/n$ and $\Delta_\alpha<\Delta$ for
\begin{equation*}
  \alpha \in \left(0,\frac{2d\sigma^2}{d\sigma^2 + n\|\theta\|^2}\right).
\end{equation*}
Let us assume that $n=1$, in which case, we obtain $\Delta_\alpha<\Delta$ for $\alpha \in \left(0,\frac{2d\sigma^2}{\mathbb{E}_x\|x\|^2}\right)$ as 
$\mathbb{E}_x\|x\|^2 = \|\theta\|^2 + d\sigma^2$. Note that the choice of $\alpha$ is dependent on $\pp{P}$ through $\mathbb{E}_x\|x\|^2$ which is not known in practice. 
To this end, we replace it with the empirical version $\|x\|^2$ that depends only on the sample $x$. For an arbitrary constant $c\in(0,2d)$, the shrinkage estimator (assuming $f^*=0$) can thus be written as 
\begin{equation*}
  \hat{\mu}_\alpha = (1-\alpha)\hat{\mu} = \left(1 - \frac{c\sigma^2}{\|x\|^2}\right)x = x - \frac{c\sigma^2x}{\|x\|^2},
\end{equation*}
which is exactly the James-Stein estimator in \eqref{eq:JS-estimator}. This particular way of estimating the shrinkage parameter $\alpha$ has an intriguing consequence, 
as shown in Stein's seminal works \citep{Stein55:Inadmissible,Stein61:JSE}, that the shrinkage estimator $\hat{\mu}_\alpha$ can be shown to dominate the maximum likelihood estimator $\hat{\mu}$ \emph{uniformly} over all $\theta$. 
 
While it is compelling to see that there is seemingly a fundamental principle underlying both these settings, this connection also reveals crucial difference 
between our approach and classical setting of Stein---notably, original James-Stein estimator improves upon the sample mean even when $\alpha$ is 
data-dependent (see $\hat{\mu}_\alpha$ above), however, with the crucial assumption that $x$ is normally distributed.
\section{Kernel Mean Estimation as Regression Problem}
\label{sec:optimization}

In Section \ref{sec:inadmissibility}, we have shown that James-Stein-like shrinkage estimator, \ie, Equation \eqref{eq:shrinkage-estimator}, improves upon the empirical estimator
in estimating the kernel mean. 
In this section, we provide a regression perspective to shrinkage estimation. The starting point of the connection between regression and shrinkage estimation is the observation that
the kernel mean $\mu_{\pp{P}}$ and its empirical estimate $\hat{\mu}_{\pp{P}}$ can be obtained as minimizers of the following risk functionals,
\begin{equation*}  
  \mathcal{E}(g) \triangleq \int_{\inspace}\left\|k(\cdot,x) -
    g\right\|^2_{\hbspace} \dd\pp{P}(x)\,\,\,\text{and}\,\,\,
  \widehat{\mathcal{E}}(g) \triangleq \frac{1}{n}\sum_{i=1}^n\left\|k(\cdot,x_i)
    - g\right\|^2_{\hbspace},
\end{equation*}
\noindent respectively \citep{Kim12:RKDE}. 
Given these formulations, it is natural to ask if minimizing the regularized version of $\widehat{\mathcal{E}}(g)$ will give a ``better'' estimator. While this question is
interesting, it has to be noted that in principle, 
there is really no need to consider a regularized formulation as the problem of minimizing $\widehat{\mathcal{E}}$ is not ill-posed, unlike in function estimation 
or regression problems. To investigate this question, we consider the minimization of the following regularized empirical risk functional,

\begin{equation}
  \label{eq:empirical-loss}
  \widehat{\mathcal{E}}_{\lambda}(g) \triangleq \widehat{\mathcal{E}}(g) + \lambda\Omega(\|g\|_{\hbspace})
  = \frac{1}{n}\sum_{i=1}^n\left\|k(\cdot,x_i) - g\right\|^2_{\hbspace} + \lambda\Omega(\|g\|_{\hbspace}),
\end{equation} 
\noindent where $\Omega:\rr_{+}\rightarrow\rr_{+}$ denotes a monotonically increasing function and $\lambda>0$ is the regularization parameter. 
By representer theorem \citep{Scholkopf01:GRT}, any function $g\in\hbspace$ that is a minimizer of \eqref{eq:empirical-loss} lies in a subspace spanned by $\{k(\cdot,x_1),\ldots,k(\cdot,x_n)\}$, \ie, 
$g=\sum_{j=1}^n\beta_jk(\cdot,x_j)$ for some $\bvec\triangleq[\beta_1,\ldots,\beta_n]^{\top}\in\rr^n$. Hence, by setting $\Omega(\|g\|_\hbspace)=\|g\|^2_\hbspace$, we can rewrite \eqref{eq:empirical-loss} in terms of $\bvec$ as
\begin{eqnarray}
  \label{eq:shrinkage-loss}
  \widehat{\mathcal{E}}(g) + \lambda\Omega(\|g\|_{\hbspace})=\bvec^\top \kmat\bvec -
  2\bvec^\top \kmat\mathbf{1}_n + \lambda\bvec^\top\kmat\bvec + c,
\end{eqnarray} 
\noindent where $\kmat$ is an $n\times n$ Gram matrix such that $\kmat_{ij} = k(x_i,x_j)$, $c$ is a constant that does not depend on $\bm{\beta}$, and
$\mathbf{1}_n = [1/n,1/n,\ldots,1/n]^\top$. Differentiating \eqref{eq:shrinkage-loss} with respect to $\bvec$ and setting it to zero yields an optimal weight vector 
$\bvec = \left(\frac{1}{1+\lambda}\right)\mathbf{1}_n$ 
and so the minimizer of \eqref{eq:empirical-loss} is given by
\begin{equation}\hat{\mu}_\lambda=\frac{1}{1+\lambda}\hat{\mu}=\left(1-\frac{\lambda}{1+\lambda}\right)\hat{\mu}\triangleq (1-\alpha)\hat{\mu},\label{eq:simple-easy}\end{equation}
which is nothing but the shrinkage estimator in \eqref{eq:shrinkage-estimator} with $\alpha = \frac{\lambda}{1+\lambda}$ and $f^\ast=0$. This provides a nice relation between shrinkage estimation
and regularized risk minimization, wherein the regularization helps in shrinking the estimator $\hat{\mu}$ towards zero although it is not required from the point of view of ill-posedness. 
In particular, 
since $0 < 1-\alpha < 1$, $\hat{\mu}_\lambda$ corresponds to a \emph{positive-part} estimator proposed in Section \ref{sec:positive-part} when $f^*=0$. 

Note that $\hat{\mu}_\lambda$ is a consistent estimator of $\mu$ as $\lambda\rightarrow 0$ and $n\rightarrow \infty$, which follows from $$\Vert \hat{\mu}_\lambda-\mu\Vert_\hbspace\le
\frac{1}{1+\lambda}\Vert \hat{\mu}-\mu\Vert_\hbspace+\frac{\lambda}{1+\lambda}\Vert \mu\Vert_\hbspace\le O_{\pp{P}}(n^{-1/2})+O(\lambda).$$ In particular $\lambda=\tau n^{-1/2}$ (for some constant $\tau>0$) yields the slowest 
possible rate for $\lambda\rightarrow 0$ such that the best possible rate of $n^{-1/2}$ is obtained for $\Vert \hat{\mu}_\lambda-\mu\Vert_\hbspace\rightarrow 0$ as $n\rightarrow \infty$. In addition, following 
the idea in Theorem~\ref{thm:main}, it is easy to show that $\pp{E}\Vert \hat{\mu}_\lambda-\mu\Vert^2_\hbspace<\Delta$ if $\tau\in \left(0,\frac{2\sqrt{n}\Delta}{\Vert \mu\Vert^2_\hbspace-\Delta}\right)$.
Note that $\hat{\mu}_\lambda$ is not useful in practice as $\lambda$ is not known \emph{a priori}. However, by choosing $$\lambda=\frac{\hat{\Delta}}{\Vert\hat{\mu}\Vert^2_\hbspace},$$ it is easy to verify (see Theorem~\ref{thm:conc} and Remark~\ref{rem:consistent})
that \begin{equation}\pp{E}\Vert \hat{\mu}_\lambda-\mu\Vert^2_\hbspace < \pp{E}\Vert \hat{\mu}-\mu\Vert^2_\hbspace+O(n^{-3/2}) \label{Eq:r-kmse-bound}\end{equation}
as $n\rightarrow\infty$. Owing to the connection of $\hat{\mu}_\lambda$ to a regression problem, in the following, we present an alternate data-dependent choice of $\lambda$ obtained from
leave-one-out cross validation (LOOCV) that also satisfies (\ref{Eq:r-kmse-bound}), and we refer to the corresponding estimator as \emph{regularized kernel mean shrinkage estimator (R-KMSE)}.

To this end, for a given shrinkage parameter $\lambda$, denote by 
$\hat{\mu}^{(-i)}_{\lambda}$ 
as the kernel mean estimated from $\{x_j\}^n_{j=1}\backslash \{x_i\}$. 
We will 
measure the quality of $\hat{\mu}^{(-i)}_{\lambda}$ by how well it approximates $k(\cdot,x_i)$ with the overall quality being quantified by the cross-validation score,
\begin{equation}
  \label{eq:loocv-score}
  LOOCV(\lambda) = \frac{1}{n}\sum_{i=1}^n\left\| k(\cdot,x_i) - \hat{\mu}_{\lambda}^{(-i)}\right\|^2_{\hbspace}.
\end{equation}
The LOOCV formulation in \eqref{eq:loocv-score} differs from the one used in regression, wherein instead of measuring the deviation of the prediction 
made by the function on the omitted observation, we measure the deviation between the feature map of the omitted observation and the function itself. 
The following result shows that the shrinkage parameter in $\hat{\mu}_\lambda$ (see (\ref{eq:simple-easy})) can be obtained analytically by minimizing (\ref{eq:loocv-score}) and requires $O(n^2)$ operations to compute.


\begin{proposition}
  \label{thm:skmse-loocv}
  Let $n\ge 2$, $\rho:=\frac{1}{n^2}\sum_{i,j=1}^n k(x_i,x_j)$ and $\varrho:=\frac{1}{n}\sum_{i=1}^n k(x_i,x_i)$. Assuming $n\rho>\varrho$, the unique minimizer of $LOOCV(\lambda)$ 
  is given by 
  \begin{equation}
    \label{eq:skmse-optimal}
    \lambda_r = \frac{n(\varrho-\rho)}{(n-1)(n\rho -\varrho)}.
  \end{equation}
\end{proposition}
\begin{proof}
See Section~\ref{proof:thm:skmse-loocv}.
\end{proof}
\indent It is instructive to compare \begin{equation}\alpha_r=\frac{\lambda_r}{\lambda_r+1}=\frac{\varrho - \rho}{(n-2)\rho + \varrho/n}\label{eq:alpha-star}\end{equation}
to the one in \eqref{eq:empirical-alpha}, where the latter can be shown to be $\frac{\varrho-\rho}{\varrho+(n-2)\rho}$, by noting that 
$\hat{\ep}k(x,x) = \varrho$ and $\hat{\pp{E}}k(x,\tilde{x})=\frac{n\rho-\varrho}{n-1}$
(in Theorem \ref{thm:conc}, we employ the $U$-statistic estimator of $\ep_{x,\tilde{x}}k(x,\tilde{x})$, whereas $\rho$ in Proposition~\ref{thm:skmse-loocv} can be seen as a $V$-statistic counterpart). 
This means $\alpha_r$ obtained from LOOCV will be relatively larger than the one obtained from \eqref{eq:empirical-alpha}. Like in Theorem \ref{thm:conc}, the requirement that $n\geq 2$ in Theorem \ref{thm:skmse-loocv} stems from the fact that at least two data points are needed to evaluate the LOOCV score.
Note that $n\rho>\varrho$ if and only if $\hat{\pp{E}}k(x,\tilde{x})>0$, which is guaranteed if the kernel is positive valued. We refer to $\hat{\mu}_{\lambda_r}$ as R-KMSE, whose $\sqrt{n}$-consistency is established by the following 
result, which also shows that $\hat{\mu}_{\lambda_r}$ satisfies (\ref{Eq:r-kmse-bound}).
%
\begin{theorem}\label{thm:skmse-loocv-consistency}
Let $n\ge 2$, $n\rho>\varrho$ where $\rho$ and $\varrho$ are defined in Proposition~\ref{thm:skmse-loocv} and $k$ satisfies the assumptions in Theorem~\ref{thm:conc}. Then $\Vert\hat{\mu}_{\lambda_r}-\mu\Vert_\hbspace=O_{\pp{P}}(n^{-1/2})$, 
\begin{equation}\min_{\alpha}\pp{E}\Vert \hat{\mu}_\alpha-\mu\Vert^2_\hbspace\le \pp{E}\Vert \hat{\mu}_{\lambda_r}-\mu\Vert^2_\hbspace\le \min_{\alpha}\pp{E}\Vert \hat{\mu}_\alpha-\mu\Vert^2_\hbspace+O(n^{-3/2})\label{eq:risk-lambda}\end{equation}
where $\hat{\mu}_\alpha=(1-\alpha)\hat{\mu}$ and therefore 
\begin{equation}\pp{E}\Vert \hat{\mu}_{\lambda_r}-\mu\Vert^2_\hbspace< \pp{E}\Vert \hat{\mu}-\mu\Vert^2_\hbspace+O(n^{-3/2})\label{eq:risk-lambda-1}\end{equation}
as $n\rightarrow\infty$.
\end{theorem}
\begin{proof}
See Section~\ref{proof:thm:skmse-loocv-consistency}.\vspace{-4mm}
\end{proof}

\section{Spectral Shrinkage Estimators}
\label{sec:flexible}
Consider the following regularized risk minimization problem
\begin{equation}
  \label{eq:smooth-operator-pop}
  {\arg\inf}_{\mathbf{F}\in\hbspace\otimes\hbspace} \quad \mathbb{E}_{x\sim\pp{P}}\left\| k(x,\cdot) - \mathbf{F}[k(x,\cdot)]\right\|^2_{\hbspace}+ \lambda\|\mathbf{F}\|^2_{\text{HS}},
\end{equation}
\noindent where the minimization is carried over the space of Hilbert-Schmidt operators, $\mathbf{F}$ on $\hbspace$ with $\Vert\mathbf{F}\Vert_{\text{HS}}$ being the Hilbert-Schmidt
norm of $\mathbf{F}$. As an interpretation, we are finding a smooth operator $\mathbf{F}$ that maps $k(x,\cdot)$ to itself (see \citet{Grunewalder13:SO} for 
more details on this smooth operator framework). 
It is not difficult to 
show that the solution to \eqref{eq:smooth-operator-pop} is given by
$\mathbf{F}=\Sigma_{\mathit{XX}}(\Sigma_{\mathit{XX}} + \lambda I)^{-1}$ where 
$\Sigma_{\mathit{XX}}=\int k(\cdot.x)\otimes k(\cdot,x)\, \dd\pp{P}(x)$ is a covariance operator defined on $\hbspace$ \citep[see, e.g.,][]{Grunewalder12:LGBPP}. Note that
$\Sigma_{\mathit{XX}}$ is a Bochner integral, which is well-defined as a Hilbert-Schmidt operator if $\mathcal{X}$ is a separable topological space and $k$ is a continuous kernel satisfying
$\int k(x,x)\,\dd\pp{P}(x)<\infty$. Consequently, let us define 
$$\mu_{\lambda} = \mathbf{F}\mu = \Sigma_{\mathit{XX}}(\Sigma_{\mathit{XX}} + \lambda I)^{-1}\mu,$$ 
which is an approximation to $\mu$ as it can be shown that 
$\Vert \mu_\lambda-\mu\Vert_\hbspace\rightarrow 0$ as $\lambda\rightarrow 0$ (see the proof of Theorem~\ref{thm:consistency-fkmse}). Given an i.i.d. sample $x_1,\ldots,x_n$ from $\pp{P}$, the empirical counterpart of \eqref{eq:smooth-operator-pop} is given by
\begin{equation}
  {\arg\min}_{\mathbf{F}\in\hbspace\otimes\hbspace} \quad \frac{1}{n}\sum_{i=1}^n\left\| k(x_i,\cdot) - \mathbf{F}[k(x_i,\cdot)]\right\|^2_{\hbspace} + \lambda\|\mathbf{F}\|^2_{\text{HS}}
  \label{eq:emp-smooth}
\end{equation}
\noindent resulting in 
\begin{equation}\label{eq:flexible}
\check{\mu}_{\lambda} \triangleq \mathbf{F}\hat{\mu} = \hat{\Sigma}_{\mathit{XX}}(\hat{\Sigma}_{\mathit{XX}}+\lambda I)^{-1}\hat{\mu}
\end{equation}
where $\hat{\Sigma}_{\mathit{XX}}$ is the empirical covariance operator on $\hbspace$ given by
$$\hat{\Sigma}_{\mathit{XX}}=\frac{1}{n}\sum^n_{i=1}k(\cdot.x_i)\otimes k(\cdot,x_i).$$ Unlike $\hat{\mu}_\lambda$ in (\ref{eq:simple-easy}), $\check{\mu}_\lambda$ shrinks $\hat{\mu}$ differently in each coordinate by taking the eigenspectrum of $\hat{\Sigma}_{XX}$ into account (see Proposition \ref{thm:akmse-shrinkage}) and so we refer to it as the \emph{spectral kernel mean shrinkage estimator (S-KMSE)}.
\begin{proposition} 
  \label{thm:akmse-shrinkage}
  Let $\{(\gamma_i,\phi_i)\}^n_{i=1}$ be eigenvalue and eigenfunction pairs of $\hat{\Sigma}_{\mathit{XX}}$. Then
  $$
    \check{\mu}_{\lambda} = \sum_{i=1}^n \frac{\gamma_i}{\gamma_i+\lambda}\langle \hat{\mu},\phi_i\rangle_\hbspace\phi_i.
  $$
\end{proposition}
\begin{proof} 
Since $\hat{\Sigma}_{\mathit{XX}}$ is a finite rank operator, it is compact. Since it is also a self-adjoint operator on $\hbspace$, by Hilbert-Schmidt theorem
\citep[Theorems VI.16, VI.17]{Reed-72}, we have $\hat{\Sigma}_{\mathit{XX}}=\sum^n_{i=1}\gamma_i\langle \phi_i,\cdot\rangle_\hbspace\phi_i$. The result follows by using this
in (\ref{eq:flexible}).
\end{proof} 
As shown in Proposition~\ref{thm:akmse-shrinkage}, the effect of S-KMSE is to reduce the contribution of high frequency components of $\hat{\mu}$ (\ie, contribution of $\hat{\mu}$ along the directions 
corresponding to smaller $\gamma_i$) when $\hat{\mu}$ is expanded in terms of the eigenfunctions of 
the empirical covariance operator, which are nothing but the kernel PCA basis \citep[Section 4.3]{Rasmussen06:GPML}. This means, similar to R-KMSE, S-KMSE also shrinks $\hat{\mu}$
towards zero, however, the difference being that while R-KMSE shrinks equally in all coordinates, S-KMSE controls the amount of shrinkage by the information contained in 
each coordinate. In particular, S-KMSE takes into account more information about the kernel by allowing for different amount of shrinkage in each coordinate direction according 
to the value 
of $\gamma_i$, wherein the shrinkage is small in the coordinates whose $\gamma_i$ are large. 
Moreover, Proposition \ref{thm:akmse-shrinkage} reveals that the effect of shrinkage is akin to \emph{spectral filtering} \citep{Bauer07:Regularization}---which in our case corresponds to
Tikhonov regularization---wherein S-KMSE filters out the high-frequency components of the spectral representation of the kernel mean. \citet{Muandet14:Spectral} leverages this 
observation and generalizes S-KMSE to a family of shrinkage estimators via spectral filtering algorithms.

The following result presents an alternate representation for $\check{\mu}_\lambda$, using which we relate the smooth operator formulation in (\ref{eq:emp-smooth})
to the regularization formulation in (\ref{eq:empirical-loss}).
\begin{proposition}\label{pro:equivalence}
Let $\Phi:\pp{R}^n\rightarrow\hbspace$, $\mathbf{a}\mapsto \sum^n_{i=1}a_i k(\cdot,x_i)$ where $\mathbf{a}\triangleq(a_1,\ldots,a_n)$. 
Then
$$\check{\mu}_\lambda=\hat{\Sigma}_{\mathit{XX}}(\hat{\Sigma}_{\mathit{XX}}+\lambda I)^{-1}\hat{\mu}=\Phi(\kmat + n\lambda\id)^{-1}\kmat\mathbf{1}_n,$$ where 
$\kmat$ is the Gram matrix, $I$ is an identity operator on $\hbspace$, $\mathbf{I}$ is an $n\times n$ identity matrix and $\mathbf{1}_n\triangleq[1/n,\ldots,1/n]^\top$.
\end{proposition}
\begin{proof}
See Section~\ref{proof:pro:equivalence}.
\end{proof}
From Proposition~\ref{pro:equivalence}, it is clear that \begin{equation}\check{\mu}_{\lambda} = \frac{1}{\sqrt{n}}\sum_{j=1}^n(\bvec_s)_jk(\cdot,x_j)\label{eq:temppp}\end{equation}
where $\bvec_s\triangleq\sqrt{n} (\kmat+n\lambda \id)^{-1}\kmat\mathbf{1}_n$. Given the form of $\check{\mu}_\lambda$ in (\ref{eq:temppp}), it is easy to verify that $\bvec_s$ is the minimizer of (\ref{eq:empirical-loss}) 
when $\widehat{\mathcal{E}}_\lambda$ is minimized over $\{g=\frac{1}{\sqrt{n}}\sum_{j=1}^n(\bvec)_jk(\cdot,x_j):\bvec\in\pp{R}^n\}$ with $\Omega(\Vert g\Vert_\hbspace)\triangleq\Vert \bm{\beta}\Vert^2_2$.

The following result, discussed in Remark~\ref{rem:compare}, establishes the consistency and convergence rate of S-KMSE, $\check{\mu}_\lambda$. 
\begin{theorem}\label{thm:consistency-fkmse}
Suppose $\mathcal{X}$ is a separable topological space and $k$ is a continuous, bounded kernel on $\mathcal{X}$. Then the following hold.
\begin{itemize}
 \item[(i)] If $\mu\in \overline{\mathcal{R}(\Sigma_{\mathit{XX}})}$, then $\Vert \check{\mu}_\lambda-\mu\Vert_\hbspace\rightarrow 0$ as $\lambda\sqrt{n}\rightarrow\infty$, $\lambda\rightarrow 0$ and $n\rightarrow\infty$.
 \item[(ii)] If $\mu\in\mathcal{R}(\Sigma_{\mathit{XX}})$, then $\Vert \check{\mu}_\lambda-\mu\Vert_\hbspace=O_{\pp{P}}(n^{-1/2})$ for $\lambda=cn^{-1/2}$ with $c>0$ being a constant independent of $n$.
\end{itemize}
%
\end{theorem}
\begin{proof}
See Section~\ref{proof:thm:consistency-fkmse}.\vspace{-6mm}
\end{proof}
\begin{remark}\label{rem:compare} 
While Theorem~\ref{thm:consistency-fkmse}(i) shows that S-KMSE, $\check{\mu}_\lambda$ is not universally consistent, \ie, S-KMSE is not consistent for all $\pp{P}$
but only for those $\pp{P}$ that satisfies $\mu\in \overline{\mathcal{R}(\Sigma_{\mathit{XX}})}$, under some additional conditions on the kernel, the universal consistency of S-KMSE
can be guaranteed. This is achieved by assuming that constant functions are included in $\hbspace$, \ie, $1\in\hbspace$.
Note that if $1\in\hbspace$, then it is easy to check that there exists $g\in\hbspace$ (choose $g=1$) such that $\mu=\Sigma_{\mathit{XX}}g=\int k(\cdot,x)g(x)\,\dd\pp{P}(x)$, \ie, $\mu\in\mathcal{R}(\Sigma_{\mathit{XX}})$,
and, therefore, by Theorem~\ref{thm:consistency-fkmse}, $\check{\mu}_\lambda$ is not only universally consistent but also achieves a rate of $n^{-1/2}$. 
Choosing $k(x,y)=\tilde{k}(x,y)+b,\,x,y\in\mathcal{X},\,b>0$ where $\tilde{k}$ is any bounded, continuous positive definite kernel ensures that $1\in\hbspace$.
\end{remark}
Note that the estimator $\check{\mu}_\lambda$ requires the knowledge of the shrinkage or regularization parameter, $\lambda$. Similar to R-KMSE, below, we present a data dependent 
approach to select $\lambda$ using leave-one-out cross validation.
While the shrinkage parameter for R-KMSE can be obtained in a simple closed form (see Proposition~\ref{thm:skmse-loocv}), we will see below that finding the corresponding parameter for S-KMSE is
more involved. Evaluating the score function (\ie, \eqref{eq:loocv-score}) na\"{\i}vely requires one to solve for $\hat{\mu}_{\lambda}^{(-i)}$ explicitly for every $i$, which is 
computationally expensive. The following result provides an alternate expression for the score, which can be evaluated efficiently. We would like to point out that a variation of
Proposition~\ref{thm:akmse-loocv} already appeared in \citet[Theorem 4]{Muandet14:KMSE}. However, Theorem 4 in \citet{Muandet14:KMSE} uses an inappropriate choice of $\hat{\mu}^{(-i)}_\lambda$,
which we fixed in the following result.
\begin{proposition}\label{thm:akmse-loocv}
 The LOOCV score of S-KMSE is given by
 \begin{eqnarray*}
 LOOCV(\lambda)&{}={}&\frac{1}{n}\emph{tr}\left((\kmat+\lambda_n \mathbf{I})^{-1}\kmat(\kmat+\lambda_n \mathbf{I})^{-1}\mathbf{A}_\lambda\right)-\frac{2}{n}\emph{tr}\left((\kmat+\lambda_n \mathbf{I})^{-1}\mathbf{B}_\lambda\right)\nonumber\\
 &{}{}&\qquad	+\frac{1}{n}\sum^n_{i=1}k(x_i,x_i), 
 \end{eqnarray*}
 where $\lambda_n\triangleq(n-1)\lambda$, $\mathbf{A}_\lambda\triangleq\frac{1}{(n-1)^2}\sum^n_{i=1}\mathbf{c}_{i,\lambda}\mathbf{c}^\top_{i,\lambda}$, $\mathbf{B}_\lambda\triangleq\frac{1}{n-1}\sum^n_{i=1}\mathbf{c}_{i,\lambda}\mathbf{k}^\top_i$, $d_{i,\lambda}\triangleq\mathbf{k}^\top_i(\kmat+\lambda_n\mathbf{I})^{-1}\mathbf{e}_i$,
 \begin{eqnarray}
 \mathbf{c}_{i,\lambda}&{}\triangleq{}&\mathbf{K}\mathbf{1}-\mathbf{k}_i-\mathbf{e}_i\mathbf{k}^\top_i\mathbf{1}+\mathbf{e}_ik(x_i,x_i)+\frac{\mathbf{e}_i\mathbf{k}^\top_i(\kmat+\lambda_n \mathbf{I})^{-1}\kmat\mathbf{1}}{1-d_{i,\lambda}}-\frac{\mathbf{e}_i\mathbf{k}^\top_i(\kmat+\lambda_n \mathbf{I})^{-1}\mathbf{k}_i}{1-d_{i,\lambda}}\nonumber\\
 &{}{}&\qquad-\frac{\mathbf{e}_i\mathbf{k}^\top_i(\kmat+\lambda_n \mathbf{I})^{-1}\mathbf{e}_i\mathbf{k}^\top_i\mathbf{1}}{1-d_{i,\lambda}}+
 \frac{\mathbf{e}_i\mathbf{k}^\top_i(\kmat+\lambda_n \mathbf{I})^{-1}\mathbf{e}_ik(x_i,x_i)}{1-d_{i,\lambda}},\nonumber
 \end{eqnarray}
 $\mathbf{k}_i$ is the $i^{th}$ column of $\kmat$, $\mathbf{1}\triangleq(1,\ldots,1)^\top$ and $\mathbf{e}_i\triangleq(0,0,\ldots,1,\ldots,0)^\top$ with $1$
 being in the $i^{th}$ position. Here $\emph{tr}(\mathbf{A})$ denotes the trace of a square matrix $\mathbf{A}$.
\end{proposition}
\begin{proof} See Section~\ref{proof:thm:akmse-loocv}.
\end{proof}
Unlike R-KMSE, a closed form expression for the minimizer of $LOOCV(\lambda)$ in Proposition~\ref{thm:akmse-loocv} is not possible and so proving the consistency of S-KMSE
along with results similar to those in Theorem~\ref{thm:skmse-loocv-consistency} are highly non-trivial. Hence, we are not able to provide any theoretical comparison of $\check{\mu}_\lambda$ (with $\lambda$
being chosen as a minimizer of $LOOCV(\lambda)$ in Proposition~\ref{thm:akmse-loocv}) with $\hat{\mu}$. However, in the next section, we provide an empirical comparison through
simulations where we show that the S-KMSE outperforms the empirical estimator.

\section{Proofs}
In this section, we present the missing proofs of the results of Sections~\ref{sec:inadmissibility}--\ref{sec:flexible}.

\subsection{Proof of Proposition~\ref{prop:zero-risk}}\label{proof:prop:zero-risk}
( $\Rightarrow$ ) If $\mathbb{P} = \delta_x$ for some $x\in\inspace$, then $\empmm=\mu=k(\cdot,x)$ and thus $\Delta=0$. \\
  ( $\Leftarrow$ ) Suppose $\Delta = 0$. It follows from (\ref{eq:RiskKF}) that $\iint (k(x,x) - k(x,y)) \dd\pp{P}(x)\dd\pp{P}(y) = 0$. 
Since $k$ is translation invariant, this reduces to 
$$
\iint (\psi(0) - \psi(x-y))\dd\pp{P}(x)\dd\pp{P}(y) = 0.
$$ 
By invoking Bochner's theorem \cite[Theorem 6.6]{Wendland-05}, which states that $\psi$ is the Fourier transform of a non-negative finite Borel measure $\Lambda$, \ie, $\psi(x)=\int e^{-ix^{\top}\omega}\dd\Lambda(\omega),\,\,x\in\mathbb{R}^d$, we obtain 
(see (16) in the proof of Proposition 5 in \citet{Sriperumbudur10b:Universal})
$$\iint \psi(x-y)\dd\pp{P}(x)\dd\pp{P}(y)=\int \vert \phi_{\pp{P}}(\omega)\vert^2\dd\Lambda(\omega),$$
thereby yielding \begin{equation}\int (|\phi_{\pp{P}}(\omega)|^2 - 1) \dd\Lambda(\omega) = 0,\label{Eq:zero}\end{equation} 
where $\phi_\pp{P}$ is the characteristic function of $\pp{P}$. 
Note that $\phi_{\pp{P}}$ is uniformly continuous and $|\phi_{\pp{P}}|\leq 1$. Since $k$ is characteristic, Theorem 9 in \citet{Sriperumbudur10:Metrics} implies that $\mathrm{supp}(\Lambda) = \rr^d$, using 
which in (\ref{Eq:zero}) yields $|\phi_{\pp{P}}(\omega)| = 1$ for all $\omega\in\mathbb{R}^d$. Since $\phi_\pp{P}$ is positive definite on $\pp{R}^d$, it follows from \citet[Lemma 1.5.1]{Sasvari-13} 
that $\phi_\pp{P}(\omega)=e^{\sqrt{-1}\omega^{\top}x}$ for some $x\in\pp{R}^d$ and thus $\pp{P} = \delta_x$.

\subsection{Proof of Theorem~\ref{thm:conc}}\label{proof:thm:conc}
Before we prove Theorem~\ref{thm:conc}, we present Bernstein's inequality in separable Hilbert spaces, quoted from \citet[Theorem 3.3.4]{Yurinksy-95}, which will be used to prove Theorem~\ref{thm:conc}.
\begin{theorem}[Bernstein's inequality]\label{thm:bernstein}
Let $(\Omega,\mathcal{A},P)$ be a probability space, $H$ be a separable Hilbert space, $B>0$ and $\theta>0$. Furthermore, let $\xi_1,\ldots,\xi_n:\Omega\rightarrow H$ be zero mean independent
random variables satisfying 
\begin{equation}\sum^n_{i=1}\pp{E}\Vert \xi_i\Vert^m_{H}\le \frac{m!}{2}\theta^2B^{m-2}.\label{Eq:bern-main}\end{equation}
Then for any $\tau>0$,
$$P^n\left\{(\xi_1,\ldots,\xi_n):\left\Vert \sum^n_{i=1}\xi_i\right\Vert_H\ge 2B\tau+\sqrt{2\theta^2\tau}\right\}\le 2 e^{-\tau}.$$
\end{theorem}
\begin{proof}\textbf{(of Theorem~\ref{thm:conc})}
Consider
\begin{eqnarray}
\tilde{\alpha}-\alpha_\ast&{}={}&\frac{\hat{\Delta}}{\hat{\Delta}+\Vert
\hat{\mu}\Vert^2_\hbspace}-\frac{\Delta}{\Delta+\Vert
\mu\Vert^2_\hbspace}=\frac{\hat{\Delta}\Vert\mu\Vert^2_\hbspace-\Delta\Vert\hat{\mu}\Vert^2_\hbspace}{(\hat{\Delta}+\Vert
\hat{\mu}\Vert^2_\hbspace)(\Delta+\Vert
\mu\Vert^2_\hbspace)}\nonumber\\
&{}={}&\frac{\hat{\Delta}(\Vert\mu\Vert^2_\hbspace-\Vert
\hat{\mu}\Vert^2_\hbspace)}{(\Delta+\Vert
\mu\Vert^2_\hbspace)(\hat{\Delta}+\Vert
\hat{\mu}\Vert^2_\hbspace)} + \frac{(\hat{\Delta}-\Delta)\Vert\hat{\mu}\Vert^2_\hbspace}{(\Delta+\Vert
\mu\Vert^2_\hbspace)(\hat{\Delta}+\Vert
\hat{\mu}\Vert^2_\hbspace)}\nonumber\\
&{}={}&\frac{\tilde{\alpha}(\Vert\mu\Vert^2_\hbspace-\Vert
\hat{\mu}\Vert^2_\hbspace)}{(\Delta+\Vert
\mu\Vert^2_\hbspace)} + \frac{(\hat{\Delta}-\Delta)(1-\tilde{\alpha})}{(\Delta+\Vert
\mu\Vert^2_\hbspace)}.\nonumber
\end{eqnarray}
Rearranging $\tilde{\alpha}$, we obtain
$$\tilde{\alpha}-\alpha_\ast=\frac{\alpha_\ast(\Vert\mu\Vert^2_\hbspace-\Vert
\hat{\mu}\Vert^2_\hbspace)+(1-\alpha_\ast)(\hat{\Delta}-\Delta)}{\hat{\Delta}+\Vert
\hat{\mu}\Vert^2_\hbspace}.$$
Therefore,
\begin{equation}|\tilde{\alpha}-\alpha_\ast|\le\frac{\alpha_\ast|\Vert\mu\Vert^2_\hbspace-\Vert
\hat{\mu}\Vert^2_\hbspace|+(1+\alpha_\ast)|\hat{\Delta}-\Delta|}{(\Delta+\Vert
\mu\Vert^2_\hbspace)-(\Vert\mu\Vert^2_\hbspace-\Vert
\hat{\mu}\Vert^2_\hbspace)+(\hat{\Delta}-\Delta)},\label{Eq:alpha}
\end{equation}
where it is easy to verify that
\begin{equation}\label{Eq:chain-1}
|\hat{\Delta}-\Delta|\le  \frac{|\pp{E}_{x,\tilde{x}}
k(x,\tilde{x})-\hat{\pp{E}}k(x,\tilde{x})|} { n }+\frac{|\hat{\pp{E}}k(x,x)-\pp{E}_xk(x,x)|}{n}.
\end{equation}
In the following we obtain bounds on $|\hat{\pp{E}}k(x,x)-\pp{E}_xk(x,x)|$, $|\pp{E}_{x,\tilde{x}}
k(x,\tilde{x})-\hat{\pp{E}}k(x,\tilde{x})|$ and $|\Vert\mu\Vert^2_\hbspace-\Vert
\hat{\mu}\Vert^2_\hbspace|$ when the kernel satisfies (\ref{Eq:bern-condition-1}) and (\ref{Eq:bern-condition-2}).\vspace{2mm}\\
\noindent\textbf{Bound on $|\hat{\pp{E}}k(x,x)-\pp{E}_xk(x,x)|$:} \vspace{1mm}\\
Since $k$ is a continuous kernel on a separable topological space $\mathcal{X}$, it follows from Lemma 4.33 of \citet{Steinwart-08} that $\hbspace$ is separable. By defining $\xi_i\triangleq k(x_i,x_i)-\pp{E}_xk(x.x)$, it follows from
(\ref{Eq:bern-condition-2}) that $\theta=\sqrt{n}\sigma_2$ and $B=\kappa_2$ and so by Theorem~\ref{thm:bernstein}, for any $\tau>0$, with probability at least $1-2e^{-\tau}$,
\begin{equation}
 |\hat{\pp{E}}k(x,x)-\pp{E}_xk(x,x)|\le \sqrt{\frac{2\sigma^2_2\tau}{n}}+\frac{2\kappa_2\tau}{n}.\label{Eq:bound-1}
\end{equation}
\noindent\textbf{Bound on $\Vert\hat{\mu}-\mu\Vert_\hbspace$:}\vspace{1mm}\\
By defining 
$\xi_i\triangleq k(\cdot,x_i)-\mu$ and using (\ref{Eq:bern-condition-1}), we have $\theta=\sqrt{n}\sigma_1$ and $B=\kappa_1$. Therefore, by Theorem~\ref{thm:bernstein}, for any $\tau>0$, 
with probability at least $1-2e^{-\tau}$,
\begin{equation}
 \Vert\hat{\mu}-\mu\Vert_\hbspace\le \sqrt{\frac{2\sigma^2_1\tau}{n}}+\frac{2\kappa_1\tau}{n}.\label{Eq:bound-2}
\end{equation}
\noindent\textbf{Bound on $|\Vert\hat{\mu}\Vert^2_\hbspace-\Vert\mu\Vert^2_\hbspace|$:}\vspace{1mm}\\
Since
$$\left|\Vert \hat{\mu}\Vert^2_\hbspace-\Vert\mu\Vert^2_\hbspace\right|\le(\Vert\hat{\mu}\Vert_\hbspace+\Vert\mu\Vert_\hbspace)\Vert\hat{\mu}-\mu\Vert_\hbspace\le (\Vert\hat{\mu}-\mu\Vert_\hbspace+2\Vert\mu\Vert_\hbspace)\Vert\hat{\mu}-\mu\Vert_\hbspace,$$
it follows from (\ref{Eq:bound-2}) that for any $\tau>0$, with probability at least $1-2e^{-\tau}$,
\begin{equation}\label{Eq:bound-3}
 \left|\Vert \hat{\mu}\Vert^2_\hbspace-\Vert\mu\Vert^2_\hbspace\right|\le D_1\sqrt{\frac{\tau}{n}}+D_2\left(\frac{\tau}{n}\right)+D_3\left(\frac{\tau}{n}\right)^{3/2}+D_4\left(\frac{\tau}{n}\right)^2,
\end{equation}
where $(D_i)^4_{i=1}$ are positive constants that depend only on $\sigma^2_1$, $\kappa$ and $\Vert\mu\Vert_\hbspace$, and not on $n$ and $\tau$.\vspace{2mm}\\
\noindent\textbf{Bound on $|\hat{\pp{E}}k(x,\tilde{x})-\pp{E}_{x,\tilde{x}}k(x,\tilde{x})|$:}\vspace{1mm}\\
Since
$$\hat{\pp{E}}k(x,\tilde{x})-\pp{E}_{x,\tilde{x}}k(x,\tilde{x})=\frac{n^2(\Vert\hat{\mu}\Vert^2_\hbspace-\Vert\mu\Vert^2_\hbspace)+n(\pp{E}_xk(x,x)-\hat{\pp{E}}k(x,x))+n(\Vert\mu\Vert^2_\hbspace-\pp{E}_xk(x,x))}{n(n-1)},$$
it follows from (\ref{Eq:bound-1}) and (\ref{Eq:bound-3}) that for any $\tau>0$, with probability at least $1-4e^{-\tau}$,
\begin{eqnarray}
|\hat{\pp{E}}k(x,\tilde{x})-\pp{E}_{x,\tilde{x}}k(x,\tilde{x})|&{}\le{}& F_1\sqrt{\frac{\tau}{n}}+F_2\left(\frac{\tau}{n}\right)+F_3\left(\frac{\tau}{n}\right)^{3/2}+F_4\left(\frac{\tau}{n}\right)^2+\frac{F_5}{n}\nonumber\\
&{}\le{}& F^\prime_1\sqrt{\frac{1+\tau}{n}}+F^\prime_2\left(\frac{1+\tau}{n}\right)+F^\prime_3\left(\frac{1+\tau}{n}\right)^{3/2}+F^\prime_4\left(\frac{1+\tau}{n}\right)^2,\label{Eq:bound-4}
\end{eqnarray}
where $(F_i)^5_{i=1}$ and $(F^\prime_i)^4_{i=1}$ are positive constants that do not depend on $n$ and $\tau$. \vspace{2mm}\\
\noindent\textbf{Bound on $|\tilde{\alpha}-\alpha_\ast|$:}\vspace{1mm}\\
Using (\ref{Eq:bound-1}) and (\ref{Eq:bound-4}) in (\ref{Eq:chain-1}), for any $\tau>0$, with probability at least $1-4e^{-\tau}$,
$$|\hat{\Delta}-\Delta|\le \frac{F^{\prime\prime}_1}{n}\sqrt{\frac{1+\tau}{n}}+\frac{F^{\prime\prime}_2}{n}\left(\frac{1+\tau}{n}\right)+\frac{F^{\prime\prime}_3}{n}\left(\frac{1+\tau}{n}\right)^{3/2}+\frac{F^{\prime\prime}_4}{n}\left(\frac{1+\tau}{n}\right)^2,$$
using which in (\ref{Eq:alpha}) along with (\ref{Eq:bound-3}), we obtain that for any $\tau>0$,
with probability at least $1-4e^{-\tau}$,
\begin{equation}
|\tilde{\alpha}-\alpha_\ast|\le\frac{\sum^4_{i=1}\left(G_{i1}\alpha_\ast+\frac{G_{i2}}{n}(1+\alpha_\ast)\right)\left(\frac{1+\tau}{n}\right)^{i/2}}{\left|\theta_n-\sum^4_{i=1}\left(G_{i1}+\frac{G_{i2}}{n}\right)\left(\frac{1+\tau}{n}\right)^{i/2}\right|},\label{Eq:bound-5} 
\end{equation}
where $\theta_n\triangleq \Delta+\Vert\mu\Vert^2_\hbspace$ and $(G_{i1})^4_{i=1}$, $(G_{i2})^4_{i=1}$ are positive constants that do not depend on $n$ and $\tau$. Since $\alpha_\ast=\frac{\Delta}{\Delta+\Vert\mu\Vert^2_\hbspace}=\frac{\pp{E}_xk(x,x)-\pp{E}_{x,\tilde{x}}k(x,\tilde{x})}{\pp{E}_xk(x,x)+(n-1)\pp{E}_{x,\tilde{x}}k(x,\tilde{x})}=O(n^{-1})$ 
and $\theta_n=\frac{\pp{E}_xk(x,x)+(n-1)\Vert\mu\Vert^2_\hbspace}{n}=O(1)$ as $n\rightarrow\infty$, it follows from (\ref{Eq:bound-5}) that 
%
$|\tilde{\alpha}-\alpha_\ast|=O_{\pp{P}}(n^{-3/2})$ as $n\rightarrow \infty$.\vspace{2mm}\\
\noindent\textbf{Bound on $\left|\Vert \hat{\mu}_{\tilde{\alpha}}-\mu\Vert_\hbspace-\Vert\hat{\mu}_{\alpha_\ast}-\mu\Vert_\hbspace\right|$:}\vspace{1mm}\\
Using (\ref{Eq:bound-2}) and (\ref{Eq:bound-5}) in
$$\left|\Vert \hat{\mu}_{\tilde{\alpha}}-\mu\Vert_\hbspace-\Vert\hat{\mu}_{\alpha_\ast}-\mu\Vert_\hbspace\right|\le \Vert \hat{\mu}_{\tilde{\alpha}}- \hat{\mu}_{\alpha_\ast}\Vert_\hbspace\le |\tilde{\alpha}-\alpha_\ast|\Vert\hat{\mu}-\mu\Vert_\hbspace+|\tilde{\alpha}-\alpha_\ast|\Vert\mu\Vert_\hbspace,$$
for any $\tau>0$, with probability at least $1-4e^{-\tau}$, we have
\begin{eqnarray}
\left|\Vert \hat{\mu}_{\tilde{\alpha}}-\mu\Vert_\hbspace-\Vert\hat{\mu}_{\alpha_\ast}-\mu\Vert_\hbspace\right|&{}\le{}& \frac{\sum^6_{i=1}\left(G^\prime_{i1}\alpha_\ast+\frac{G^\prime_{i2}}{n}(1+\alpha_\ast)\right)\left(\frac{1+\tau}{n}\right)^{i/2}}{\left|\theta_n-\sum^4_{i=1}\left(G_{i1}+\frac{G_{i2}}{n}\right)\left(\frac{1+\tau}{n}\right)^{i/2}\right|},\label{Eq:bound-6} 
\end{eqnarray}
where $(G^\prime_{i1})^6_{i=1}$ and $(G^\prime_{i2})^6_{i=1}$ are positive constants that do not depend on $n$ and $\tau$.
From (\ref{Eq:bound-6}), it is easy to see that $\left|\Vert \hat{\mu}_{\tilde{\alpha}}-\mu\Vert_\hbspace-\Vert\hat{\mu}_{\alpha_\ast}-\mu\Vert_\hbspace\right|=O_{\pp{P}}(n^{-3/2})$ as 
$n\rightarrow \infty$.\vspace{2mm}\\
\noindent\textbf{Bound on $\pp{E}\Vert \hat{\mu}_{\tilde{\alpha}}-\mu\Vert^2_\hbspace-\pp{E}\Vert\hat{\mu}_{\alpha_\ast}-\mu\Vert^2_\hbspace$:}\vspace{1mm}\\
Since
\begin{eqnarray}
\Vert\hat{\mu}_{\tilde{\alpha}}-\mu\Vert^2_{\hbspace}-\Vert\hat{\mu}_{\alpha_\ast}-\mu\Vert^2_{\hbspace}&{}\le{}&(\Vert\hat{\mu}_{\tilde{\alpha}}-\mu\Vert_{\hbspace}+\Vert\hat{\mu}_{\alpha_\ast}-\mu\Vert_{\hbspace})\left|\Vert\hat{\mu}_{\tilde{\alpha}}-\mu\Vert_{\hbspace}-\Vert\hat{\mu}_{\alpha_\ast}-\mu\Vert_{\hbspace}\right|\nonumber\\
&{}\le{}& 2(\Vert\hat{\mu}\Vert_\hbspace+\Vert\mu\Vert_\hbspace)\left|\Vert\hat{\mu}_{\tilde{\alpha}}-\mu\Vert_{\hbspace}-\Vert\hat{\mu}_{\alpha_\ast}-\mu\Vert_{\hbspace}\right|\nonumber\\
&{}\le{}&2(\Vert\hat{\mu}-\mu\Vert_\hbspace+2\Vert\mu\Vert_\hbspace)\left|\Vert\hat{\mu}_{\tilde{\alpha}}-\mu\Vert_{\hbspace}-\Vert\hat{\mu}_{\alpha_\ast}-\mu\Vert_{\hbspace}\right|,\nonumber
\end{eqnarray}
for any $\tau>0$, with probability at least $1-4e^{-\tau}$,
\begin{eqnarray}
\Vert\hat{\mu}_{\tilde{\alpha}}-\mu\Vert^2_{\hbspace}-\Vert\hat{\mu}_{\alpha_\ast}-\mu\Vert^2_{\hbspace}&{}\le{}& \frac{\sum^8_{i=1}\left(G^{\prime\prime}_{i1}\alpha_\ast+\frac{G^{\prime\prime}_{i2}}{n}(1+\alpha_\ast)\right)\left(\frac{1+\tau}{n}\right)^{i/2}}{\left|\theta_n-\sum^4_{i=1}\left(G_{i1}+\frac{G_{i2}}{n}\right)\left(\frac{1+\tau}{n}\right)^{i/2}\right|},\nonumber\\
&{}\le{}&\frac{\sum^8_{i=1}\left(G^{\prime\prime}_{i1}\alpha_\ast+\frac{G^{\prime\prime}_{i2}}{n}(1+\alpha_\ast)\right)\left(\frac{1+\tau}{n}\right)^{i/2}}{\left|\theta_n-\sum^4_{i=1}\left(G_{i1}+\frac{G_{i2}}{n}\right)\left(\frac{1}{n}\right)^{i/2}\right|},\nonumber\\
&{}\le{}&\begin{cases} 
          \frac{\gamma_n}{\phi_n}\sqrt{\frac{1+\tau}{n}},&\mbox{} 0<\tau\le n-1\\
          \frac{\gamma_n}{\phi_n}\left(\frac{1+\tau}{n}\right)^4,&\mbox{} \tau\ge n-1\nonumber
         \end{cases},
\end{eqnarray}
where $\gamma_n\triangleq H_1\alpha_\ast+\frac{H_2}{n}(1+\alpha_\ast)$, $\phi_n\triangleq \left|\theta_n-\sum^4_{i=1}\left(G_{i1}+\frac{G_{i2}}{n}\right)\left(\frac{1}{n}\right)^{i/2}\right|$ and 
$(H_i)^2_{i=1}$ are positive constants that do not depend on $n$ and $\tau$. In other words,
$$\pp{P}\left(\Vert\hat{\mu}_{\tilde{\alpha}}-\mu\Vert^2_{\hbspace}-\Vert\hat{\mu}_{\alpha_\ast}-\mu\Vert^2_{\hbspace}>\epsilon\right)\le
\begin{cases}
4\exp\left(1-n\left(\frac{\epsilon\phi_n}{\gamma_n}\right)^2\right),&\mbox{}\frac{\gamma_n}{\phi_n\sqrt{n}}\le\epsilon\le\frac{\gamma_n}{\phi_n}\\
4\exp\left(1-n\left(\frac{\epsilon\phi_n}{\gamma_n}\right)^{1/4}\right),&\mbox{}\epsilon\ge\frac{\gamma_n}{\phi_n}\nonumber
\end{cases}.
$$
Therefore,
\begin{eqnarray}
\pp{E}\Vert\hat{\mu}_{\tilde{\alpha}}-\mu\Vert^2_{\hbspace}-\pp{E}\Vert\hat{\mu}_{\alpha_\ast}-\mu\Vert^2_{\hbspace}&{}={}&\int^\infty_0 \pp{P}\left(\Vert\hat{\mu}_{\tilde{\alpha}}-\mu\Vert^2_{\hbspace}-\Vert\hat{\mu}_{\alpha_\ast}-\mu\Vert^2_{\hbspace}>\epsilon\right)\,\dd\epsilon\nonumber\\
&{}\le{}&\frac{\gamma_n}{\phi_n\sqrt{n}}+4\int^{\frac{\gamma_n}{\phi_n}}_{\frac{\gamma_n}{\phi_n\sqrt{n}}}\exp\left(1-n\left(\frac{\epsilon\phi_n}{\gamma_n}\right)^2\right)\,\dd\epsilon\nonumber\\
&{}{}&\qquad+\,4\int^\infty_{\frac{\gamma_n}{\phi_n}}\exp\left(1-n\left(\frac{\epsilon\phi_n}{\gamma_n}\right)^{1/4}\right)\,\dd\epsilon\nonumber\\
&{}={}& \frac{\gamma_n}{\phi_n\sqrt{n}}+\frac{2\gamma_n}{\phi_n\sqrt{n}}\int^{n-1}_0\frac{e^{-t}}{\sqrt{t+1}}\,\dd t
+\frac{16e\gamma_n}{n^4\phi_n}\int^\infty_{n}t^3e^{-t}\,\dd t.\nonumber
\end{eqnarray}
Since $\int^{n-1}_0\frac{e^{-t}}{\sqrt{t+1}}\,\dd t\le\int^\infty_0 e^{-t}\,\dd t=1$ and $\int^\infty_n t^3 e^{-t}\,\dd t\le \int^\infty_0 t^3e^{-t}\,\dd t=6$, we have
$$\pp{E}\Vert\hat{\mu}_{\tilde{\alpha}}-\mu\Vert^2_{\hbspace}-\pp{E}\Vert\hat{\mu}_{\alpha_\ast}-\mu\Vert^2_{\hbspace}\le \frac{3\gamma_n}{\phi_n\sqrt{n}}+\frac{96e\gamma_n}{n^4\phi_n}.$$
The claim in (\ref{Eq:oracle}) follows by noting that $\gamma_n=O(n^{-1})$ and $\phi_n=O(1)$ as $n\rightarrow\infty$.\vspace{-5mm}
\end{proof}

\subsection{Proof of Proposition~\ref{thm:skmse-loocv}}\label{proof:thm:skmse-loocv}
Define $\alpha\triangleq\frac{\lambda}{\lambda+1}$ and $\phi(x_i)\triangleq k(\cdot,x_i)$. Note that
\begin{eqnarray*}
LOOCV(\lambda) &{}\triangleq{}& \frac{1}{n}\sum_{i=1}^n\left\|\frac{(1-\alpha)}{n-1}\sum_{j\ne i}\phi(x_j) - \phi(x_i)\right\|^2_{\hbspace}\nonumber\\
  &{}={}& \frac{1}{n}\sum_{i=1}^n\left\|\frac{n(1-\alpha)}{n-1}\hat{\mu} - \frac{1-\alpha}{n-1}\phi(x_i) - \phi(x_i)\right\|^2_{\hbspace} \\
  &{}={}& \left\|\frac{n(1-\alpha)}{n-1}\hat{\mu}\right\|^2_{\hbspace}
  -
  \frac{2}{n}\left\langle\sum_{i=1}^n\frac{n-\alpha}{n-1}\phi(x_i),\frac{n(1-\alpha)}{n-1}\hat{\mu}\right\rangle_{\hbspace}+ \frac{1}{n}\sum_{i=1}^n\left\|\frac{n-\alpha}{n-1}\phi(x_i)\right\|^2_{\hbspace} \\
  &{}={}& \left(\frac{n^2(1-\alpha)^2}{(n-1)^2} - \frac{2n(n-\alpha)(1-\alpha)}{(n-1)^2}\right)\|\hat{\mu}\|^2_{\hbspace}  + \frac{(n-\alpha)^2}{n(n-1)^2}\sum_{i=1}^n k(x_i,x_i) \\
  &{}={}& \frac{1}{(n-1)^2}\left\{\alpha^2(n^2\rho-2n\rho+\varrho)+2n\alpha(\rho-\varrho)+n^2(\varrho-\rho)\right\}\triangleq\frac{F(\alpha)}{(n-1)^2}.
\end{eqnarray*} 
Since $\frac{d}{d\lambda}LOOCV(\lambda)=(n-1)^{-2}\frac{d}{d\alpha}F(\alpha)\frac{d\alpha}{d\lambda}=(n-1)^{-2}(1+\lambda)^{-2}\frac{d}{d\alpha}F(\alpha)$, equating it zero yields (\ref{eq:skmse-optimal}).
It is easy to show that the second derivative of $LOOCV(\lambda)$ is positive implying that $LOOCV(\lambda)$ is strictly convex and so $\lambda_r$ is unique.

\subsection{Proof of Theorem~\ref{thm:skmse-loocv-consistency}}\label{proof:thm:skmse-loocv-consistency}
 Since $\hat{\mu}_{\lambda_r}=\frac{\hat{\mu}}{1+\lambda_r}=(1-\alpha_r)\hat{\mu}$, we have 
 $\Vert \hat{\mu}_{\lambda_r}-\mu\Vert_\hbspace\le \alpha_r\Vert \hat{\mu}\Vert_\hbspace+\Vert\hat{\mu}-\mu\Vert_\hbspace$. Note that
 $$\alpha_r=\frac{n(\varrho-\rho)}{n(n-2)\rho+\varrho}=\frac{n\hat{\Delta}}{\hat{\Delta}+(n-1)\Vert\hat{\mu}\Vert^2_\hbspace}=\frac{\hat{\pp{E}}k(x,x)-\hat{\pp{E}}k(x,\tilde{x})}{\hat{\pp{E}}k(x,x)+(n-2)\hat{\pp{E}}k(x,\tilde{x})},$$
 where $\hat{\Delta}$, $\Vert\hat{\mu}\Vert^2_\hbspace$, $\hat{\pp{E}}k(x,x)$ and $\hat{\pp{E}}k(x,\tilde{x})$ are defined in Theorem~\ref{thm:conc}. Consider
 $|\alpha_r-\alpha_\ast|\le |\alpha_r-\tilde{\alpha}|+|\tilde{\alpha}-\alpha_\ast|$ where $\tilde{\alpha}$ is defined in (\ref{eq:empirical-alpha}). From Theorem~\ref{thm:conc}, we have
 $|\tilde{\alpha}-\alpha_\ast|=O_\pp{P}(n^{-3/2})$ as $n\rightarrow\infty$ and
 \begin{eqnarray}
\alpha_r-\tilde{\alpha}&{}={}&\frac{\hat{\pp{E}}k(x,x)-\hat{\pp{E}}k(x,\tilde{x})}{\hat{\pp{E}}k(x,x)+(n-2)\hat{\pp{E}}k(x,\tilde{x})}-\frac{\hat{\pp{E}}k(x,x)-\hat{\pp{E}}k(x,\tilde{x})}{2\hat{\pp{E}}k(x,x)+(n-2)\hat{\pp{E}}k(x,\tilde{x})}\nonumber\\
&{}={}&\frac{\tilde{\alpha}\hat{\pp{E}}k(x,x)}{\hat{\pp{E}}k(x,x)+(n-2)\hat{\pp{E}}k(x,\tilde{x})}=(\tilde{\alpha}-\alpha_\ast)\beta+\alpha_\ast\beta,\nonumber
 \end{eqnarray}
where $\beta\triangleq\frac{\hat{\pp{E}}k(x,x)}{\hat{\pp{E}}k(x,x)+(n-2)\hat{\pp{E}}k(x,\tilde{x})}$. Therefore, $|\alpha_r-\tilde{\alpha}|\le |\tilde{\alpha}-\alpha_\ast||\beta|+\alpha_\ast|\beta|$, where
$\alpha_\ast=O(n^{-1})$ as $n\rightarrow\infty$, which follows from Remark~\ref{rem:consistent}(i). Since $|\hat{\pp{E}}k(x,x)-\pp{E}_{x}k(x,x)|=O_\pp{P}(n^{-1/2})$ and $|\hat{\pp{E}}k(x,\tilde{x})-\pp{E}_{x,\tilde{x}}k(x,\tilde{x})|=O_\pp{P}(n^{-1/2})$, which follow from
(\ref{Eq:bound-1}) and (\ref{Eq:bound-4}) respectively, we have $|\beta|=O_{\pp{P}}(n^{-1})$ as $n\rightarrow\infty$. Combining the above, we have $|\alpha_r-\tilde{\alpha}|=O_\pp{P}(n^{-2})$, thereby yielding
$|\alpha_r-\alpha_\ast|=O_\pp{P}(n^{-3/2})$. Proceeding as in Theorem~\ref{thm:conc}, we have
$$\left| \Vert \hat{\mu}_{\lambda_r}-\mu\Vert_\hbspace-\Vert \hat{\mu}_{\alpha_\ast}-\mu\Vert_\hbspace\right|\le \Vert \hat{\mu}_{\lambda_r}-\mu_{\alpha_\ast}\Vert_\hbspace\le |\alpha_r-\alpha_\ast|\Vert \hat{\mu}-\mu\Vert_\hbspace+|\alpha_r-\alpha_\ast|\Vert\mu\Vert_\hbspace,$$
which from the above follows that $\left| \Vert \hat{\mu}_{\lambda_r}-\mu\Vert_\hbspace-\Vert \hat{\mu}_{\alpha_\ast}-\mu\Vert_\hbspace\right|=O_{\pp{P}}(n^{-3/2})$ as $n\rightarrow\infty$. By 
arguing as in Remark~\ref{rem:consistent}(i), it is easy to show that $\hat{\mu}_{\lambda_r}$ is a $\sqrt{n}$-consistent estimator of $\mu$. (\ref{eq:risk-lambda}) follows by 
carrying out the analysis as in the proof of Theorem~\ref{thm:conc} verbatim by replacing $\tilde{\alpha}$ with $\alpha_r$, while (\ref{eq:risk-lambda-1}) follows by appealing to 
Remark~\ref{rem:consistent}(ii).

\subsection{Proof of Proposition~\ref{pro:equivalence}}\label{proof:pro:equivalence}
First note that for any $i\in \{1,\ldots,n\}$,
$$\hat{\Sigma}_{\mathit{XX}}k(\cdot,x_i)=\frac{1}{n}\sum^n_{j=1}k(\cdot,x_j)k(x_i,x_j)=\frac{1}{n}\Phi \mathbf{k}^\top_i$$
with $\mathbf{k}_i$ being the $i^{th}$ row of $\kmat$. This implies for any $\mathbf{a}\in\pp{R}^n$,
$$\hat{\Sigma}_{\mathit{XX}}\Phi\mathbf{a}=\hat{\Sigma}_{\mathit{XX}}\left(\sum^n_{i=1}a_ik(\cdot,x_i)\right)\stackrel{(*)}{=}\sum^n_{i=1}a_i\hat{\Sigma}_{\mathit{XX}}k(\cdot,x_i)=\frac{1}{n}\sum^n_{i=1}a_i\Phi \mathbf{k}^\top_i,$$
where $(*)$ holds since $\hat{\Sigma}_{\mathit{XX}}$ is a linear operator. Also, since $\Phi$ is a linear operator, we obtain
\begin{equation}\hat{\Sigma}_{\mathit{XX}}\Phi\mathbf{a}=\frac{1}{n}\Phi\left(\sum^n_{i=1}a_i\mathbf{k}^\top_i\right)=\frac{1}{n}\Phi\mathbf{Ka}.\label{Eq:equality}\end{equation}
To prove the result, let us define $\mathbf{a}\triangleq(\kmat+n\lambda \mathbf{I})^{-1}\kmat\mathbf{1}_n$ and consider
\begin{eqnarray}
(\hat{\Sigma}_{\mathit{XX}}+\lambda I)\Phi \mathbf{a}&{}\stackrel{(\ref{Eq:equality})}{=}{}&n^{-1}\Phi\mathbf{Ka}+\lambda \Phi\mathbf{a}=\Phi(n^{-1}\mathbf{K}+\lambda \mathbf{I})\mathbf{a}=\frac{1}{n}\Phi\kmat\mathbf{1}_n\stackrel{(\ref{Eq:equality})}{=}\hat{\Sigma}_{\mathit{XX}}\Phi\mathbf{1}_n=\hat{\Sigma}_{\mathit{XX}}\hat{\mu}.\nonumber
\end{eqnarray}
%
%
%
%
Multiplying to the left on both sides of the above equation by $(\hat{\Sigma}_{\mathit{XX}}+\lambda I)^{-1}$, we obtain $\Phi(\kmat+n\lambda \mathbf{I})^{-1}\kmat\mathbf{1}_n=(\hat{\Sigma}_{\mathit{XX}}+\lambda I)^{-1}\hat{\Sigma}_{\mathit{XX}}\hat{\mu}$ 
and the result follows by noting that $(\hat{\Sigma}_{\mathit{XX}}+\lambda I)^{-1}\hat{\Sigma}_{\mathit{XX}}=\hat{\Sigma}_{\mathit{XX}}(\hat{\Sigma}_{\mathit{XX}}+\lambda I)^{-1}$.

\subsection{Proof of Theorem~\ref{thm:consistency-fkmse}}\label{proof:thm:consistency-fkmse}
By Proposition~\ref{pro:equivalence}, we have $\check{\mu}_\lambda=(\hat{\Sigma}_{\mathit{XX}}+\lambda I)^{-1}\hat{\Sigma}_{\mathit{XX}}\hat{\mu}$. Define $\mu_\lambda\triangleq (\Sigma_{\mathit{XX}}+\lambda I)^{-1}\Sigma_{\mathit{XX}}\mu$.
Let us consider the decomposition $\check{\mu}_\lambda-\mu= (\check{\mu}_\lambda-\mu_\lambda)+ (\mu_\lambda-\mu)$ with
\begin{eqnarray}
\check{\mu}_\lambda-\mu_\lambda &{}={}& (\hat{\Sigma}_{\mathit{XX}}+\lambda I)^{-1}(\hat{\Sigma}_{\mathit{XX}}\hat{\mu}-\hat{\Sigma}_{\mathit{XX}}\mu_\lambda-\lambda\mu_\lambda)\nonumber\\
&{}\stackrel{(*)}{=}{}& (\hat{\Sigma}_{\mathit{XX}}+\lambda I)^{-1}(\hat{\Sigma}_{\mathit{XX}}\hat{\mu}-\hat{\Sigma}_{\mathit{XX}}\mu_\lambda-\Sigma_{\mathit{XX}}\mu+\Sigma_{\mathit{XX}}\mu_\lambda)\nonumber\\
&{}={}&(\hat{\Sigma}_{\mathit{XX}}+\lambda I)^{-1}\hat{\Sigma}_{\mathit{XX}}(\hat{\mu}-\mu)-(\hat{\Sigma}_{\mathit{XX}}+\lambda I)^{-1}\hat{\Sigma}_{\mathit{XX}}(\mu_\lambda-\mu)\nonumber\\
&{}{}&\qquad\qquad+(\hat{\Sigma}_{\mathit{XX}}+\lambda I)^{-1}\Sigma_{\mathit{XX}}(\mu_\lambda-\mu),\nonumber
\end{eqnarray}
where we used $\lambda\mu_\lambda=\Sigma_{\mathit{XX}}\mu-\Sigma_{\mathit{XX}}\mu_\lambda$ in $(*)$. By defining $\mathcal{A}(\lambda)\triangleq\Vert \mu_\lambda-\mu\Vert_\hbspace$, we have
\begin{eqnarray}
 \Vert \check{\mu}_\lambda-\mu\Vert_\hbspace &{}\le{}& \Vert (\hat{\Sigma}_{\mathit{XX}}+\lambda I)^{-1}\hat{\Sigma}_{\mathit{XX}}(\hat{\mu}-\mu)\Vert_\hbspace + \Vert (\hat{\Sigma}_{\mathit{XX}}+\lambda I)^{-1}\hat{\Sigma}_{\mathit{XX}}(\mu_\lambda-\mu)\Vert_\hbspace\nonumber\\
 &{}{}&\qquad\qquad+\Vert (\hat{\Sigma}_{\mathit{XX}}+\lambda I)^{-1}\Sigma_{\mathit{XX}}(\mu_\lambda-\mu)\Vert_\hbspace+\mathcal{A}(\lambda)\nonumber\\
 &{}\le{}& \Vert (\hat{\Sigma}_{\mathit{XX}}+\lambda I)^{-1}\hat{\Sigma}_{\mathit{XX}}\Vert \left(\Vert\hat{\mu}-\mu\Vert_\hbspace + \mathcal{A}(\lambda)\right)+\Vert (\hat{\Sigma}_{\mathit{XX}}+\lambda I)^{-1}\Sigma_{\mathit{XX}}\Vert \mathcal{A}(\lambda)\nonumber\\
 &{}{}&\qquad\qquad+\mathcal{A}(\lambda),\label{eq:H}
\end{eqnarray}
where for any bounded linear operator $B$, $\Vert B\Vert$ denotes its operator norm. We now bound 
$\Vert (\hat{\Sigma}_{\mathit{XX}}+\lambda I)^{-1}\Sigma_{\mathit{XX}}\Vert$ as follows. It is easy to show that
\begin{eqnarray}
(\hat{\Sigma}_{\mathit{XX}}+\lambda I)^{-1}\Sigma_{\mathit{XX}}&{}={}&\left(I-(\Sigma_{\mathit{XX}}+\lambda I)^{-1}(\Sigma_{\mathit{XX}}-\hat{\Sigma}_{\mathit{XX}})\right)^{-1}(\Sigma_{\mathit{XX}}+\lambda I)^{-1}\Sigma_{\mathit{XX}}\nonumber\\
&{}={}&\left(\sum^\infty_{j=0}\left((\Sigma_{\mathit{XX}}+\lambda I)^{-1}(\Sigma_{\mathit{XX}}-\hat{\Sigma}_{\mathit{XX}})\right)^j\right)(\Sigma_{\mathit{XX}}+\lambda I)^{-1}\Sigma_{\mathit{XX}},\nonumber
\end{eqnarray}
where the last line denotes the Neumann series
and therefore
\begin{eqnarray}
\Vert(\hat{\Sigma}_{\mathit{XX}}+\lambda I)^{-1}\Sigma_{\mathit{XX}}\Vert&{}\le{}& \sum^\infty_{j=0}\left\Vert (\Sigma_{\mathit{XX}}+\lambda I)^{-1}(\Sigma_{\mathit{XX}}-\hat{\Sigma}_{\mathit{XX}})\right\Vert^j \Vert(\Sigma_{\mathit{XX}}+\lambda I)^{-1}\Sigma_{\mathit{XX}}\Vert\nonumber\\
&{}\le{}&  \sum^\infty_{j=0}\left\Vert (\Sigma_{\mathit{XX}}+\lambda I)^{-1}(\Sigma_{\mathit{XX}}-\hat{\Sigma}_{\mathit{XX}})\right\Vert^j_{\text{HS}},\nonumber
\end{eqnarray}
where we used $\Vert(\Sigma_{\mathit{XX}}+\lambda I)^{-1}\Sigma_{\mathit{XX}}\Vert\le 1$ and the fact that $\Sigma_{\mathit{XX}}$ and $\hat{\Sigma}_{\mathit{XX}}$ are Hilbert-Schmidt operators on 
$\hbspace$ as $\Vert \Sigma_{\mathit{XX}}\Vert_{\text{HS}}\le\kappa<\infty$ 
and $\Vert \hat{\Sigma}_{\mathit{XX}}\Vert_{\text{HS}}\le\kappa<\infty$ with $\kappa$ being the bound on the kernel. Define $\eta:\mathcal{X}\rightarrow \text{HS}(\hbspace)$,
$\eta(x)=(\Sigma_{\mathit{XX}}+\lambda I)^{-1}(\Sigma_{\mathit{XX}}-\Sigma_x)$, where $\text{HS}(\hbspace)$ is the space of
Hilbert-Schmidt operators on $\hbspace$ and $\Sigma_x\triangleq k(\cdot,x)\otimes k(\cdot,x)$. Observe that
$\pp{E}\frac{1}{n}\sum^n_{i=1}\eta(x_i)=0$. Also, for all $i\in\{1,\ldots,n\}$,
$\Vert\eta(x_i)\Vert_{\text{HS}}\le \Vert (\Sigma_{\mathit{XX}}+\lambda I)^{-1}\Vert \Vert
\Sigma_{\mathit{XX}}-\Sigma_x\Vert_{\text{HS}}\le \frac{2\kappa}{\lambda}\,\,\,\text{and}\,\,\,
\pp{E}\Vert\eta(x_i)\Vert^2_{\text{HS}}\le\frac{4\kappa^2}{\lambda^2}.$
Therefore, by Bernstein's inequality (see Theorem~\ref{thm:bernstein}), for any $\tau>0$, with probability at least $1-2e^{-\tau}$ over the
choice of $\{x_i\}^n_{i=1}$,
$$\Vert (\Sigma_{\mathit{XX}}+\lambda
I)^{-1}(\Sigma_{\mathit{XX}}-\hat{\Sigma}_{\mathit{XX}})\Vert_{\text{HS}}\le \frac{\kappa\sqrt{2\tau}}{\lambda\sqrt{n}}+\frac{2\kappa\tau}{\lambda n}\le
\frac{\kappa\sqrt{2\tau}(\sqrt{2\tau}+1)}{\lambda\sqrt{n}}.$$ For $\lambda\ge \frac{\kappa\sqrt{8\tau}(\sqrt{2\tau}+1)}{\sqrt{n}}$,
we obtain that $\Vert (\Sigma_{\mathit{XX}}+\lambda I)^{-1}(\Sigma_{\mathit{XX}}-\hat{\Sigma}_{\mathit{XX}})\Vert_{\text{HS}}\le \frac{1}{2}$ and therefore $\Vert
(\hat{\Sigma}_{\mathit{XX}}+\lambda I)^{-1}\Sigma_{\mathit{XX}}\Vert\le 2$. Using this along with $\Vert (\hat{\Sigma}_{\mathit{XX}}+\lambda I)^{-1}\hat{\Sigma}_{\mathit{XX}}\Vert\le 1$ and (\ref{Eq:bound-2}) in (\ref{eq:H}), we obtain that
for any $\tau>0$ and $\lambda\ge \frac{\kappa\sqrt{8\tau}(\sqrt{2\tau}+1)}{\sqrt{n}}$, 
with probability at least $1-2e^{-\tau}$ over the choice of $\{x_i\}^n_{i=1}$, 
\begin{equation}\Vert \check{\mu}_\lambda-\mu\Vert_\hbspace\le \frac{\sqrt{2\kappa\tau}+4\tau\sqrt{\kappa}}{\sqrt{n}}+4\mathcal{A}(\lambda).\label{eq:finall}\end{equation} 
We now analyze $\mathcal{A}(\lambda)$. Since $k$ is continuous and $\mathcal{X}$ is separable, $\hbspace$ is separable \cite[Lemma 4.33]{Steinwart-08}. Also $\Sigma_{\mathit{XX}}$ is compact since it is Hilbert-Schmidt. The consistency result therefore follows from 
\citet[Proposition A.2]{sriperumbudur13density} which ensures $\mathcal{A}(\lambda)\rightarrow 0$ as $\lambda\rightarrow 0$. The rate also follows from \citet[Proposition A.2]{sriperumbudur13density}
which yields $\mathcal{A}(\lambda)\le \Vert \Sigma^{-1}_{\mathit{XX}}\mu\Vert_\hbspace\lambda$, thereby obtaining $\Vert\check{\mu}_\lambda-\mu\Vert_\hbspace=O_{\pp{P}}(n^{-1/2})$ for $\lambda=cn^{-1/2}$ with $c>0$ being a constant independent of $n$.
%

\subsection{Proof of Proposition~\ref{thm:akmse-loocv}}\label{proof:thm:akmse-loocv}
From Proposition~\ref{pro:equivalence}, we have $\check{\mu}^{(-i)}_\lambda=(\hat{\Sigma}^{(-i)}+\lambda I)^{-1}\hat{\Sigma}^{(-i)}\hat{\mu}^{(-i)}$ where $$\hat{\Sigma}^{(-i)}\triangleq\frac{1}{n-1}\sum_{j\ne i}k(\cdot,x_j)\otimes k(\cdot,x_j).$$ 
and $\hat{\mu}^{(-i)}\triangleq\frac{1}{n-1}\sum_{j\ne i}k(\cdot,x_j)$. Define $a\triangleq k(\cdot,x_i)$. It is easy to verify that
$$\hat{\Sigma}^{(-i)}=\frac{n}{n-1}\left(\hat{\Sigma}-\frac{a\otimes a}{n}\right)\,\,\,\text{and}\,\,\,\hat{\mu}^{(-i)}=\frac{n}{n-1}\left(\hat{\mu}-\frac{a}{n}\right).$$
Therefore,
$$\check{\mu}^{(-i)}_\lambda=\frac{n}{n-1}\left((\hat{\Sigma}+\lambda^\prime_n I)-\frac{a\otimes a}{n}\right)^{-1}\left(\hat{\Sigma}-\frac{a\otimes a}{n}\right)\left(\hat{\mu}-\frac{a}{n}\right),$$
which after using Sherman-Morrison formula\footnote{The Sherman-Morrison formula states that $(\mathbf{A}+\mathbf{uv}^\top)^{-1}=\mathbf{A}^{-1}-\frac{\mathbf{A}^{-1}\mathbf{uv}^\top\mathbf{A}^{-1}}{1+\mathbf{v}^\top\mathbf{A}^{-1}\mathbf{u}}$ where 
$\mathbf{A}$ is an invertible square matrix, $\mathbf{u}$ and $\mathbf{v}$ are column vectors such that $1+\mathbf{v}^\top\mathbf{A}^{-1}\mathbf{u}\ne 0$.} reduces to
$$\check{\mu}^{(-i)}_\lambda=\frac{n}{n-1}\left((\hat{\Sigma}+\lambda^\prime_n I)^{-1}+\frac{(\hat{\Sigma}+\lambda^\prime_n I)^{-1}(a\otimes a)(\hat{\Sigma}+\lambda^\prime_n I)^{-1}}{n-\langle a,(\hat{\Sigma}+\lambda^\prime_n I)^{-1}a\rangle_\hbspace}\right)\left(\hat{\Sigma}-\frac{a\otimes a}{n}\right)\left(\hat{\mu}-\frac{a}{n}\right),$$
where $\lambda^\prime_n\triangleq\frac{n-1}{n}\lambda$. Using the notation in the proof of Proposition~\ref{pro:equivalence}, the following can be proved:
\begin{itemize}
 \item[(i)] $(\hat{\Sigma}+\lambda^\prime_n I)^{-1}\hat{\Sigma}\hat{\mu}=n^{-1}\Phi(\kmat+\lambda_n\mathbf{I})^{-1}\kmat\mathbf{1}.$
 \item[(ii)]$(\hat{\Sigma}+\lambda^\prime_n I)^{-1}\hat{\Sigma}a=\Phi(\kmat+\lambda_n\mathbf{I})^{-1}\mathbf{k}_i.$
 \item[(iii)] $(\hat{\Sigma}+\lambda^\prime_n I)^{-1}a=n\Phi(\kmat+\lambda_n\mathbf{I})^{-1}\mathbf{e}_i.$
\end{itemize}
Based on the above, it is easy to show that
\begin{itemize}
 \item[(iv)] $(\hat{\Sigma}+\lambda^\prime_n I)^{-1}(a\otimes a)\hat{\mu}= (\hat{\Sigma}+\lambda^\prime_n I)^{-1}a \langle a,\hat{\mu}\rangle_\hbspace=\Phi(\kmat+\lambda_n\mathbf{I})^{-1}\mathbf{e}_i\mathbf{k}^\top_i\mathbf{1}$.
 \item[(v)] $(\hat{\Sigma}+\lambda^\prime_n I)^{-1}(a\otimes a)a= (\hat{\Sigma}+\lambda^\prime_n I)^{-1}a \langle a,a\rangle_\hbspace=n\Phi(\kmat+\lambda_n\mathbf{I})^{-1}\mathbf{e}_ik(x_i,x_i)$.
 \item[(vi)] $(\hat{\Sigma}+\lambda^\prime_n I)^{-1}(a\otimes a)(\hat{\Sigma}+\lambda^\prime_n I)^{-1}\hat{\Sigma}\hat{\mu}=\Phi(\kmat+\lambda_n\mathbf{I})^{-1}\mathbf{e}_i\mathbf{k}^\top_i(\kmat+\lambda_n\mathbf{I})^{-1}\kmat\mathbf{1}.$
 \item[(vii)] $(\hat{\Sigma}+\lambda^\prime_n I)^{-1}(a\otimes a)(\hat{\Sigma}+\lambda^\prime_n I)^{-1}\hat{\Sigma}a=n\Phi(\kmat+\lambda_n\mathbf{I})^{-1}\mathbf{e}_i\mathbf{k}^\top_i(\kmat+\lambda_n\mathbf{I})^{-1}\mathbf{k}_i.$
 \item[(viii)] $(\hat{\Sigma}+\lambda^\prime_n I)^{-1}(a\otimes a)(\hat{\Sigma}+\lambda^\prime_n I)^{-1}(a\otimes a)\hat{\mu}=n\Phi(\kmat+\lambda_n\mathbf{I})^{-1}\mathbf{e}_i\mathbf{k}^\top_i(\kmat+\lambda_n\mathbf{I})^{-1}\mathbf{e}_i\mathbf{k}^\top_i\mathbf{1}.$
 \item[(ix)] $(\hat{\Sigma}+\lambda^\prime_n I)^{-1}(a\otimes a)(\hat{\Sigma}+\lambda^\prime_n I)^{-1}(a\otimes a)a=n^2\Phi(\kmat+\lambda_n\mathbf{I})^{-1}\mathbf{e}_i\mathbf{k}^\top_i(\kmat+\lambda_n\mathbf{I})^{-1}\mathbf{e}_ik(x_i,x_i).$
 \item[(x)] $\langle a,(\hat{\Sigma}+\lambda^\prime_n I)^{-1}a\rangle_\hbspace=n\mathbf{k}^\top_i (\kmat+\lambda_n\mathbf{I})^{-1}\mathbf{e}_i.$
\end{itemize}
Using the above in $\check{\mu}^{(-i)}_\lambda$, we obtain $$\check{\mu}^{(-i)}_\lambda=\frac{1}{n-1}\Phi(\kmat+\lambda_n\mathbf{I})^{-1}\mathbf{c}_{i,\lambda}.$$
Substituting the above in (\ref{eq:loocv-score}) yields the result.
\section{Experiments}
\label{sec:experiments}
  
In this section, we empirically compare the proposed shrinkage estimators to the standard estimator of the kernel mean on both synthetic and real-world datasets. 
Specifically, we consider the following estimators:
\begin{inparaenum}[\itshape i\upshape)]
  \item empirical/standard kernel mean estimator (KME),
  \item KMSE whose parameter is obtained via empirical bound (B-KMSE),
  \item regularized KMSE whose parameter is obtained via Proposition \ref{thm:skmse-loocv} (R-KMSE), and
  \item spectral KMSE whose parameter is obtained via Proposition \ref{thm:akmse-loocv} (S-KMSE). 
\end{inparaenum}

\subsection{Synthetic Data}
\label{sec:synthetic}  

Given the true data-generating distribution $\pp{P}$ and the i.i.d. sample $X=\{x_1,x_2,\ldots,x_n\}$ from $\pp{P}$, we evaluate different estimators using the loss function 
\begin{equation*}
  L(\bvec,X,\pp{P}) \triangleq \left\|\sum_{i=1}^n\beta_i k(x_i,\cdot) - \ep_{x\sim\pp{P}}[k(x,\cdot)]\right\|^2_{\hbspace} ,
\end{equation*}
where $\bvec$ is the weight vector associated with different estimators. Then, we can estimate the risk of the estimator by averaging over $m$ independent copies of $X$, \ie, $\widehat{R} = \frac{1}{m}\sum_{j=1}^m L(\bvec_j,X_j,\pp{P})$. 

To allow for an exact calculation of $L(\bvec,X,\pp{P})$, we consider $\pp{P}$ to be a mixture-of-Gaussians distribution and $k$ being one of the following kernel functions: 
\begin{inparaenum}[\itshape i\upshape)]
\item linear kernel $k(x,x')=x^\top x'$,
\item polynomial degree-2 kernel $k(x,x')=(x^\top x' + 1)^2$, 
\item polynomial degree-3 kernel $k(x,x')=(x^\top x' + 1)^3$ and
\item Gaussian RBF kernel $k(x,x') = \exp\left(-\|x-x'\|^2/2\sigma^2\right)$. 
\end{inparaenum}
We refer to them as LIN, POLY2, POLY3, and RBF, respectively. The analytic forms of $\ep_{x\sim\pp{P}}[k(x,\cdot)]$ for Gaussian distribution are given in \citet{Song08:TDE} 
and \citet{Muandet12:SMM}. Unless otherwise stated, we set the bandwidth parameter of the Gaussian kernel as $\sigma^2 = \mathrm{median}\left\{\|x_i-x_j\|^2:i,j=1,\ldots,n\right\}$, \ie, the median heuristic.

\begin{figure}[t!]
  \centering
  \includegraphics[width=0.32\textwidth]{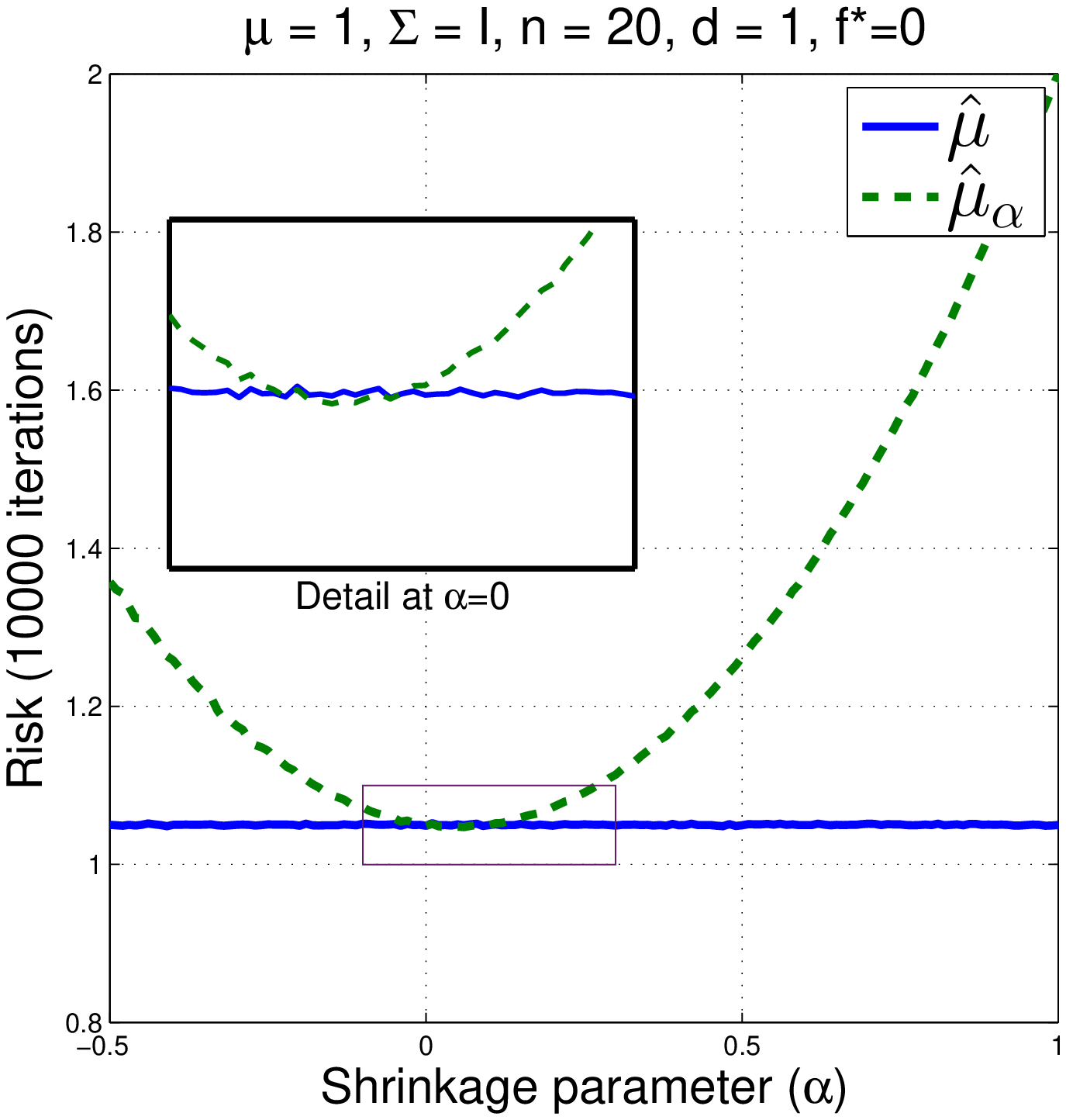}
  \includegraphics[width=0.32\textwidth]{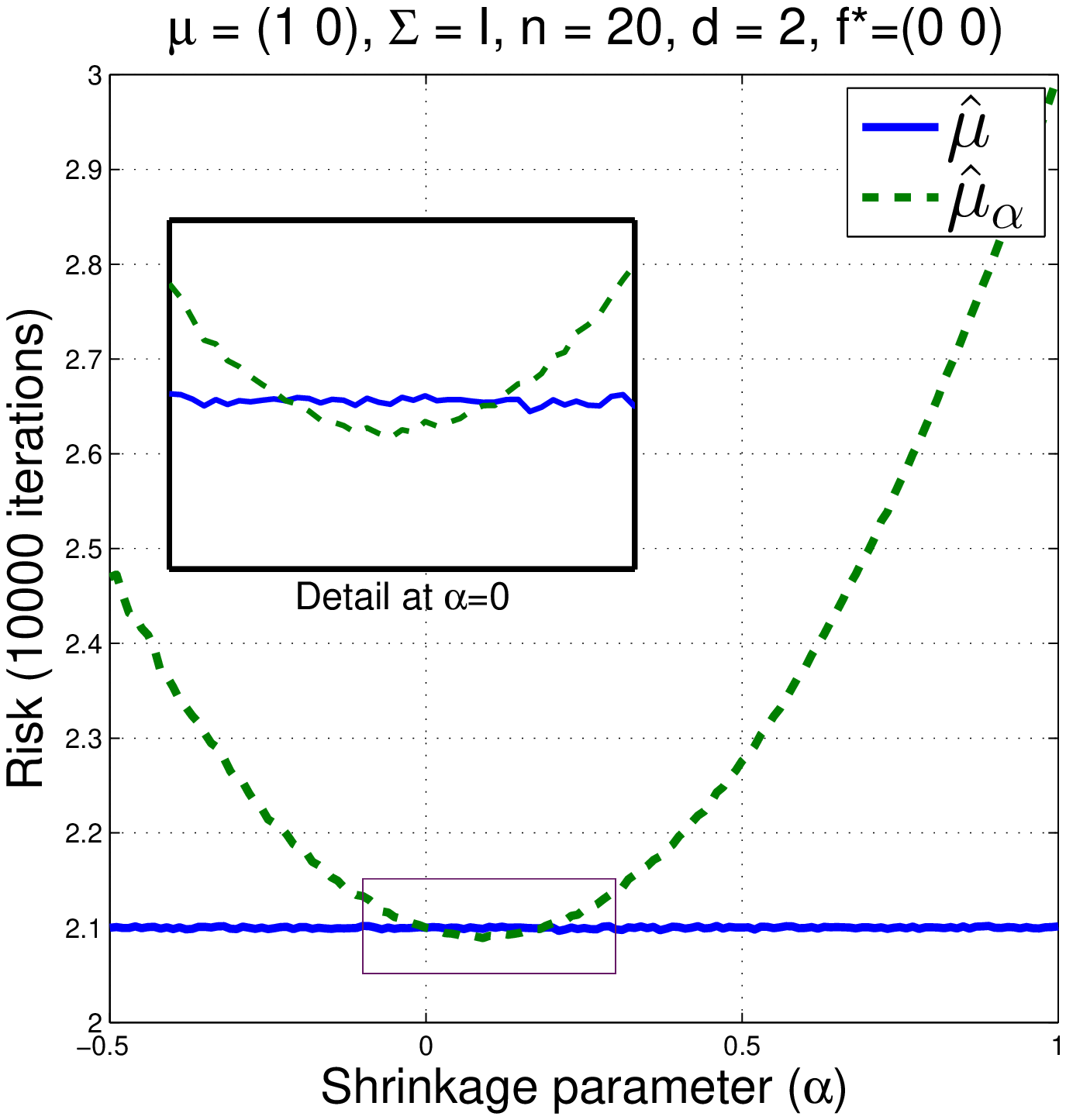}
  \includegraphics[width=0.32\textwidth]{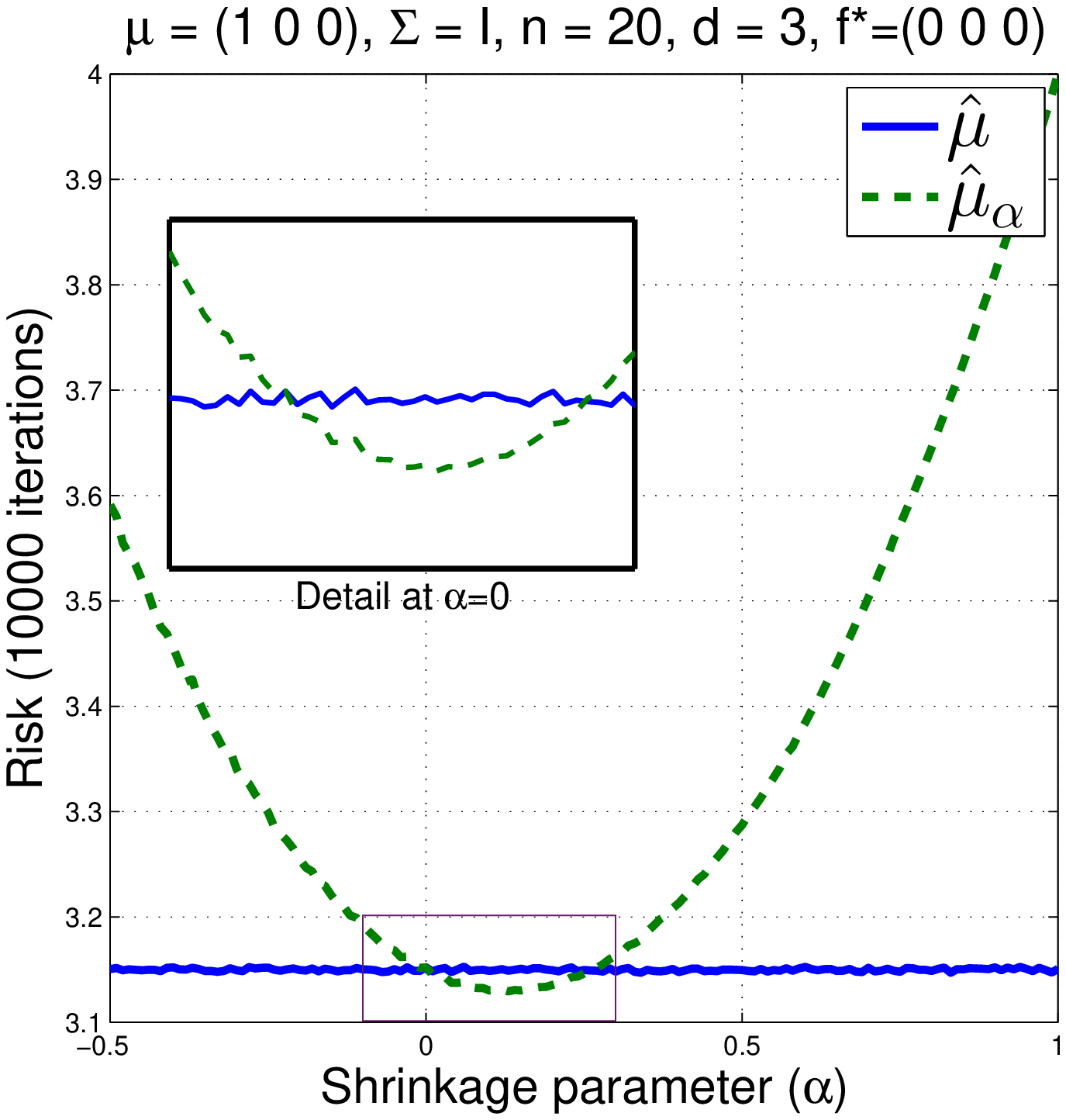}
  \caption{The comparison between standard estimator, $\hat{\mu}$ and shrinkage estimator, $\hat{\mu}_\alpha$ (with $f^*=0$) of the mean of the Gaussian distribution $\mathcal{N}(\mu,\Sigma)$ on $\rr^d$ where $d=1,2,3$.}
  \vspace{-3mm}
  \label{fig:gaussian1}
\end{figure}

\subsubsection{Gaussian Distribution}

We begin our empirical studies by considering the simplest case in which the distribution $\pp{P}$ is a Gaussian distribution $\mathcal{N}(\mu,\id)$ on $\rr^d$ where $d=1,2,3$ and 
$k$ is a linear kernel. In this case, the problem of kernel mean estimation reduces to just estimating the mean $\mu$ of the Gaussian distribution $\mathcal{N}(\mu,\id)$. We consider only 
shrinkage estimators of form 
$\hat{\mu}_{\alpha} = \alpha f^* + (1-\alpha)\hat{\mu}$. The true mean $\mu$ of the distribution is chosen to be $1$, $(1,0)^\top$, and $(1,0,0)^\top$, respectively. 
Figure \ref{fig:gaussian1} depicts the comparison between the standard estimator and the shrinkage estimator, $\hat{\mu}_\alpha$ when the target $f^*$ is the origin. We can 
clearly see that even in this simple case, an improvement can be gained by applying a small shrinkage. Furthermore, the improvement becomes more substantial as we increase 
the dimensionality of the underlying space. Figure \ref{fig:gaussian2} illustrates similar results when $f^*\ne 0$ but $f^*\in\{2,(2, 0)^\top,(2, 0, 0)^\top\}$. Interestingly, 
we can still observe similar improvement, which demonstrates that the choice of target $f^*$ can be arbitrary when no prior knowledge about $\mu_{\pp{P}}$ is 
available.

\subsubsection{Mixture of Gaussians Distributions}
To simulate a more realistic case, let $y$ be a sample from $\pp{P}\triangleq \sum_{i=1}^4\pi_i\mathcal{N}(\bm{\theta}_i,\Sigma_i)$. In the following experiments, the sample $x$ is generated from the following generative process:
  \begin{equation*}
    x = y + \varepsilon,\quad
    \theta_{ij} \sim  \mathcal{U}(-10,10), \quad
    \Sigma_i \sim \mathcal{W}(2\times\mathbf{I}_d,7), \quad 
    \varepsilon \sim \mathcal{N}(0,0.2\times\mathbf{I}_d),    
  \end{equation*}
  \noindent where $\mathcal{U}(a,b)$ and $\mathcal{W}(\Sigma_0,df)$ represent the uniform distribution and Wishart distribution, respectively. We set $\bm{\pi}=(0.05, 0.3, 0.4, 0.25)^\top$. The choice of parameters here is quite arbitrary; we have experimented using various parameter settings and the results are similar to those presented here. 
\begin{figure}[t!]
  \centering
  \includegraphics[width=0.32\textwidth]{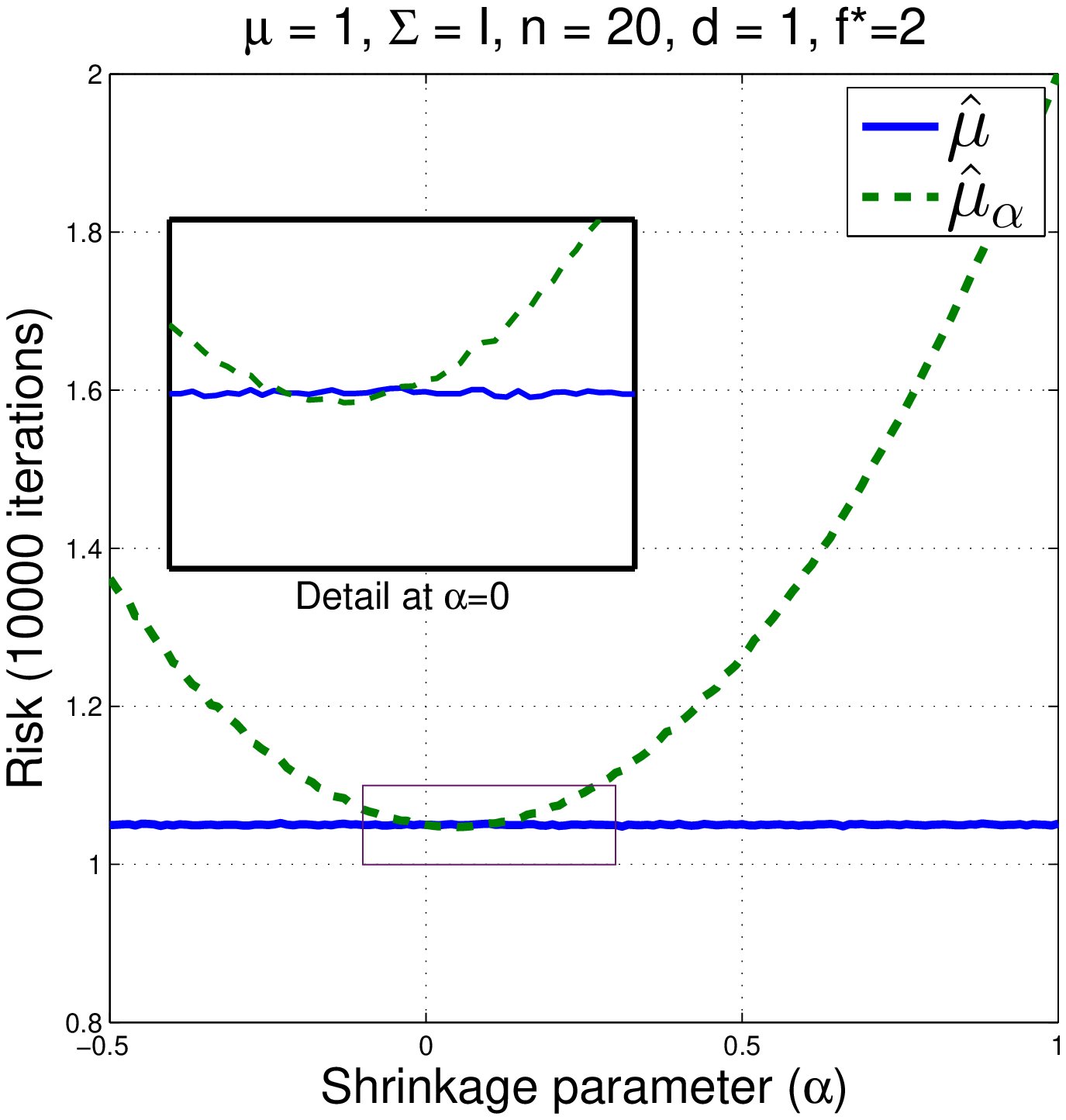}
  \includegraphics[width=0.32\textwidth]{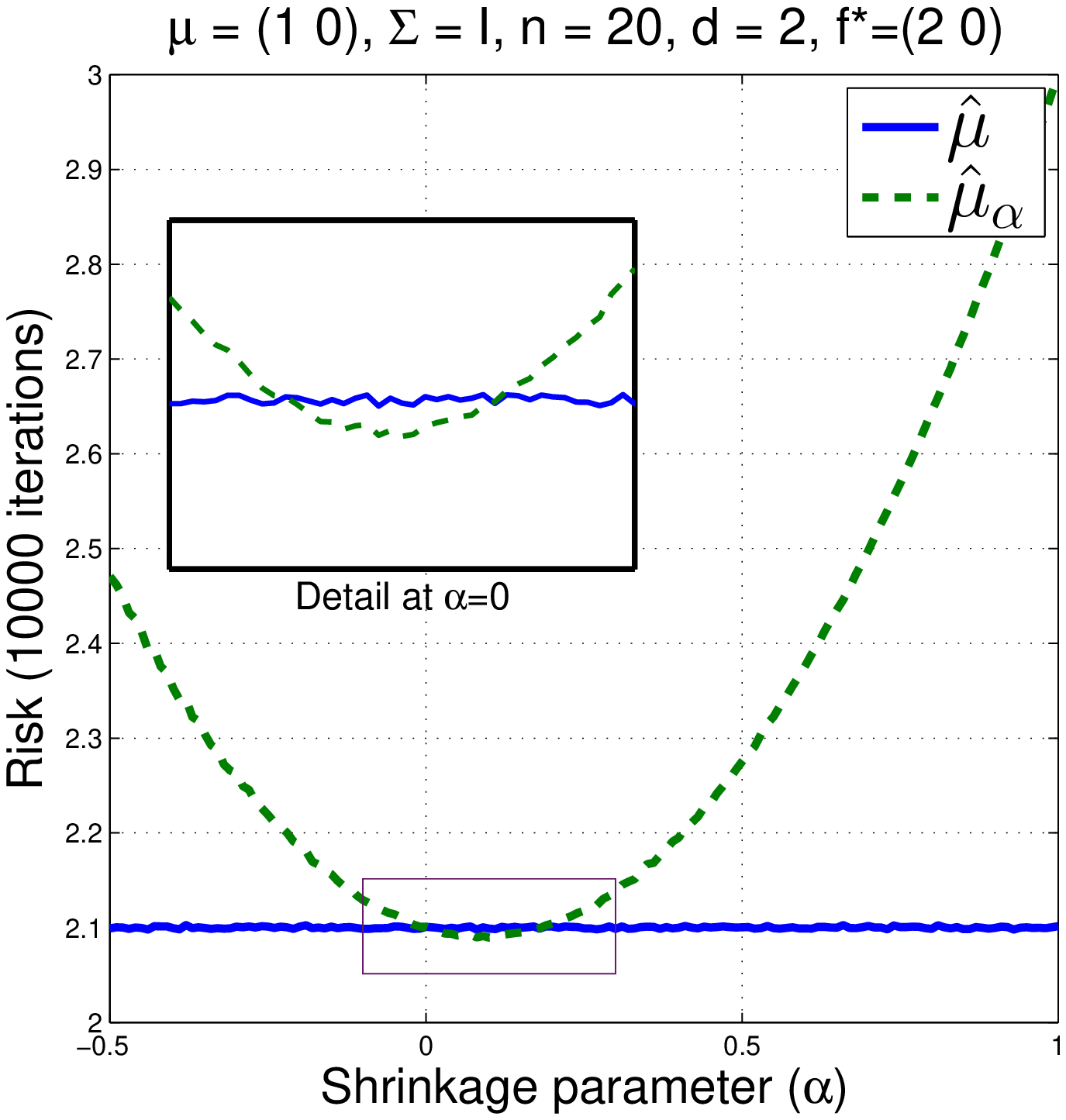}
  \includegraphics[width=0.32\textwidth]{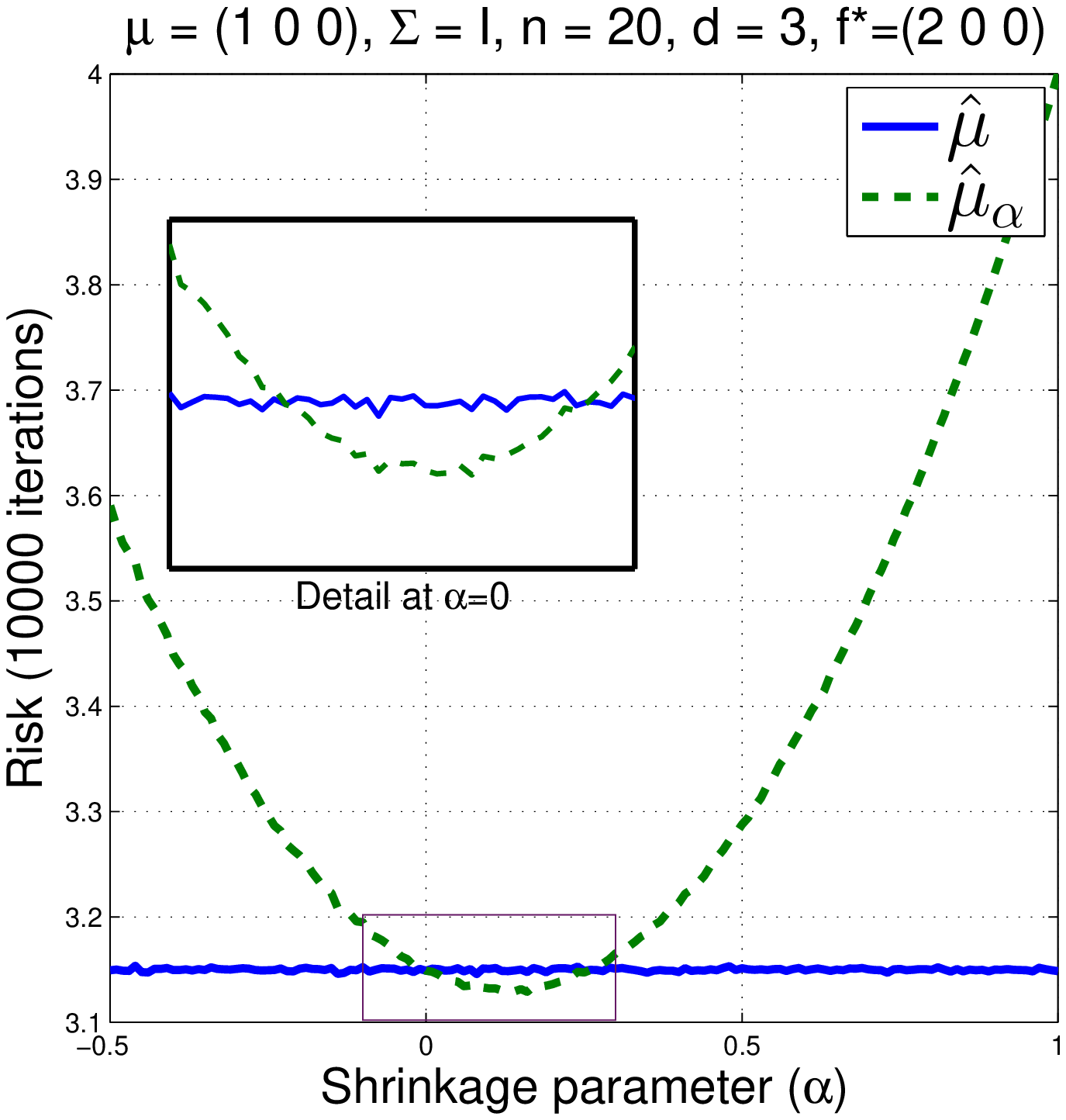}
  \caption{The risk comparison between standard estimator, $\hat{\mu}$ and shrinkage estimator, $\hat{\mu}_\alpha$ (with $f^*\in\{2, (2,0)^\top,(2,0,0)^\top\}$) of the mean of the Gaussian distribution 
  $\mathcal{N}(\mu,\Sigma)$ on $\rr^d$ where $d=1,2,3$.}\vspace{-3mm}
  \label{fig:gaussian2}
\end{figure}


Figure \ref{fig:simple-shk} depicts the comparison between the standard kernel mean estimator and the shrinkage estimator, $\hat{\mu}_{\alpha}$ when the kernel $k$ is the Gaussian RBF kernel. For shrinkage estimator $\hat{\mu}_{\alpha}$, we consider $f^* = C\times k(x,\cdot)$ where $C$ is a scaling factor and each element of $x$ is a realization of uniform random variable on $(0,1)$. That is, we allow the target $f^*$ to change depending on the value of $C$. As the absolute value of $C$ increases, the target function $f^*$ will move further away from the origin. The shrinkage parameter $\alpha$ is determined using the empirical bound, \ie, $\tilde{\alpha} = \hat{\Delta}/(\hat{\Delta} + \|f^*-\hat{\mu}\|^2_{\hbspace})$. As we can see in Figure \ref{fig:simple-shk}, the results reveal how important the choice of $f^*$ is. That is, we may get substantial improvement over the empirical estimator if appropriate prior knowledge is incorporated through $f^*$, which in this case suggests that $f^*$ should lie close to the origin. We intend to investigate the topic of prior knowledge in more detail in our future work.

\begin{figure}[t]
  \centering
  \subfigure[]{
    \includegraphics[height=2in]{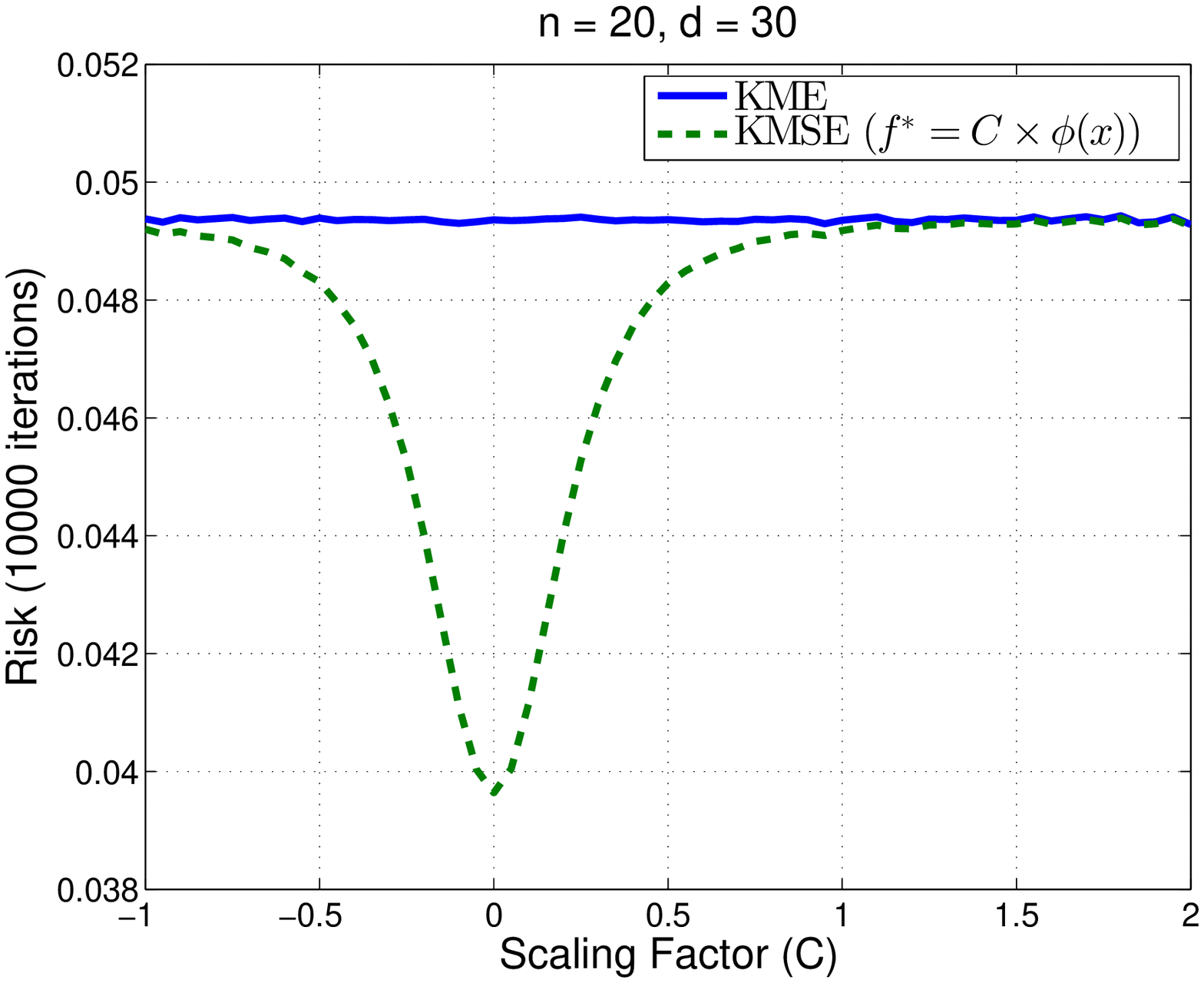}
    \label{fig:simple-shk}
  }
  \hspace{10pt}
  \subfigure[]{
    \includegraphics[height=2in]{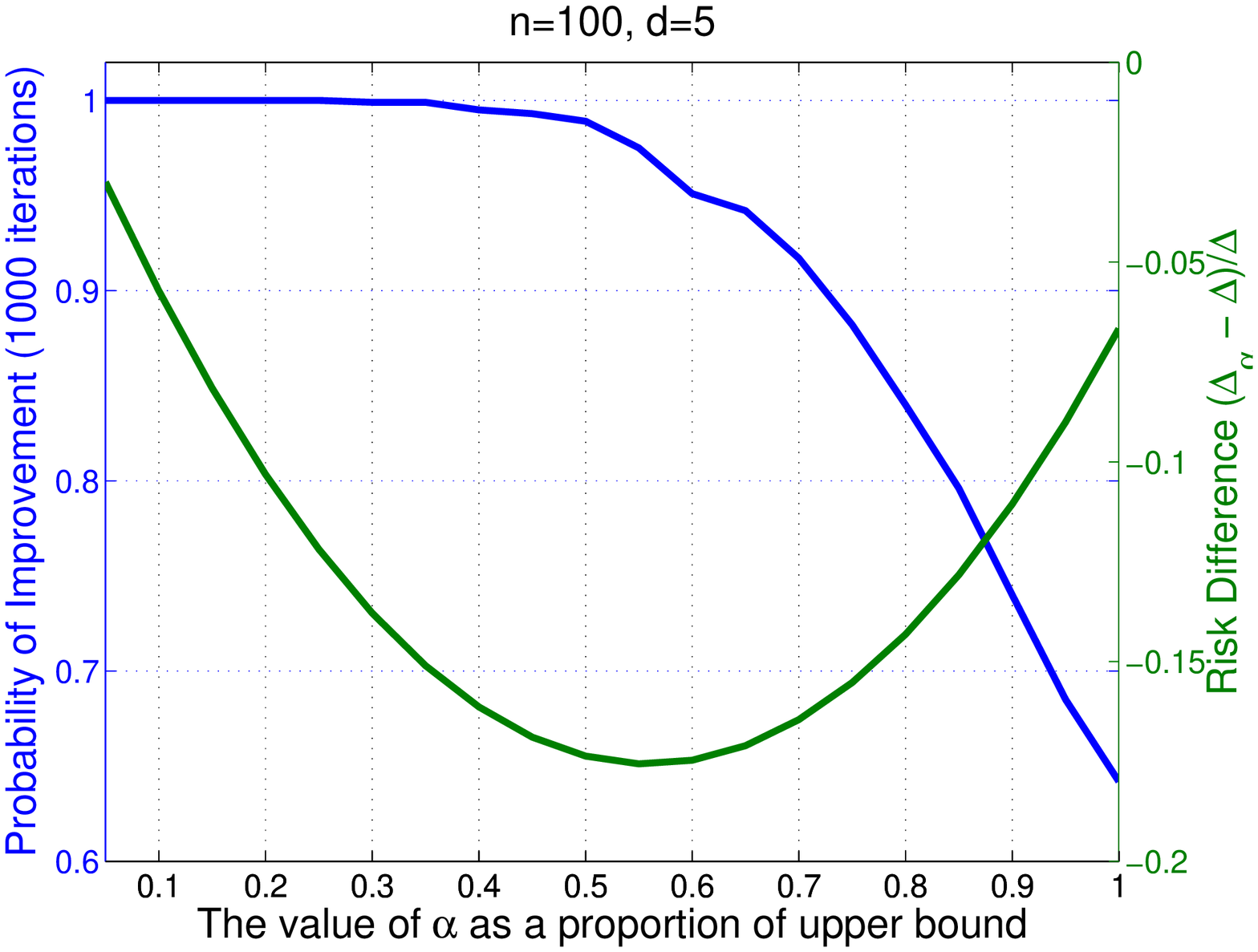}
    \label{fig:shrinkage-tradeoff}
  }\vspace{-3mm}
  \caption{\subref{fig:simple-shk} The risk comparison between $\hat{\mu}$ (KME) and $\hat{\mu}_{\tilde{\alpha}}$ (KMSE) where $\tilde{\alpha} = \hat{\Delta}/(\hat{\Delta} + \|f^*-\hat{\mu}\|^2_{\hbspace})$. We consider when $f^*= C\times k(x,\cdot)$ where $x$ is drawn uniformly from a pre-specified range and $C$ is a scaling factor. \subref{fig:shrinkage-tradeoff} The probability of improvement and the risk difference as a function of shrinkage parameter $\alpha$ averaged over 1,000 iterations. As the value of $\alpha$ increases, we get more improvement in term of the risk, 
  whereas the probability of improvement decreases as a function of $\alpha$.}\vspace{-6mm}
  \label{fig:simple-1}
\end{figure}
Previous comparisons between standard estimator and shrinkage estimator is based entirely on the notion of a risk, which is in fact not useful in practice as we only observe a 
single copy of sample from the probability distribution. Instead, one should also look at the probability that, given a single copy of sample, the shrinkage estimator outperforms 
the standard one in term of a loss. To this end, we conduct an experiment comparing the standard estimator and shrinkage estimator using the Gaussian RBF kernel. In addition to the 
risk comparison, we also compare the probability that the shrinkage estimator gives smaller loss than that of the standard estimator. To be more precise, the probability is defined 
as a proportion of the samples drawn from the same distribution whose shrinkage loss is smaller than the loss of the standard estimator. Figure \ref{fig:shrinkage-tradeoff} 
illustrates the risk difference ($\Delta_{\alpha} - \Delta$) and the probability of improvement (\ie, the fraction of times $\Delta_\alpha<\Delta$) 
as a function of shrinkage parameter $\alpha$. In this case, the value of $\alpha$ is specified as a proportion of empirical upper bound $2\hat{\Delta}/(\hat{\Delta} + \|\hat{\mu}\|^2_{\hbspace})$. 
The results suggest that the shrinkage parameter $\alpha$ controls the trade-off between the amount of improvement in terms of risk and the probability that the shrinkage estimator will improve upon the standard one. However, this trade-off only holds up to a certain value of $\alpha$. As $\alpha$ becomes too large, both the probability of improvement and the amount of improvement itself decrease, which coincides with the intuition given for the positive-part shrinkage estimators (cf. Section \ref{sec:positive-part}).


\subsubsection{Shrinkage Estimators via Leave-One-Out Cross-Validation}

In addition to the empirical upper bound, one can alternatively compute the shrinkage parameter using leave-one-out cross-validation proposed in Section \ref{sec:optimization}. 
Our goal here is to compare the B-KMSE, R-KMSE and S-KMSE on synthetic data when the shrinkage parameter $\lambda$ is chosen via leave-one-out cross-validation procedure. Note that 
the only difference between B-KMSE and R-KMSE is the way we compute the shrinkage parameter.

Figure \ref{fig:sim-results1} shows the empirical risk of different estimators using different kernels as we increase the value of shrinkage parameter $\lambda$ (note that 
R-KMSE and S-KMSE in Figure~\ref{fig:sim-results1} refer to those in (\ref{eq:simple-easy}) and (\ref{eq:flexible}) respectively). Here we scale the 
shrinkage parameter by the smallest non-zero eigenvalue $\gamma_0$ of the kernel matrix $\kmat$. In general, we find that R-KMSE and S-KMSE outperforms KME. Nevertheless, as the 
shrinkage parameter $\lambda$ becomes large, there is a tendency that the specific shrinkage estimate might actually perform worse than the KME, \eg, see LIN kernel and outliers in 
Figure \ref{fig:sim-results1}. The result also supports our previous observation regarding Figure \ref{fig:shrinkage-tradeoff}, which suggests that it is very important to choose the 
parameter $\lambda$ appropriately. 
\begin{figure*}[tp!]
  \centering
  \subfigure[LIN]{
    \includegraphics[width=0.9\linewidth]{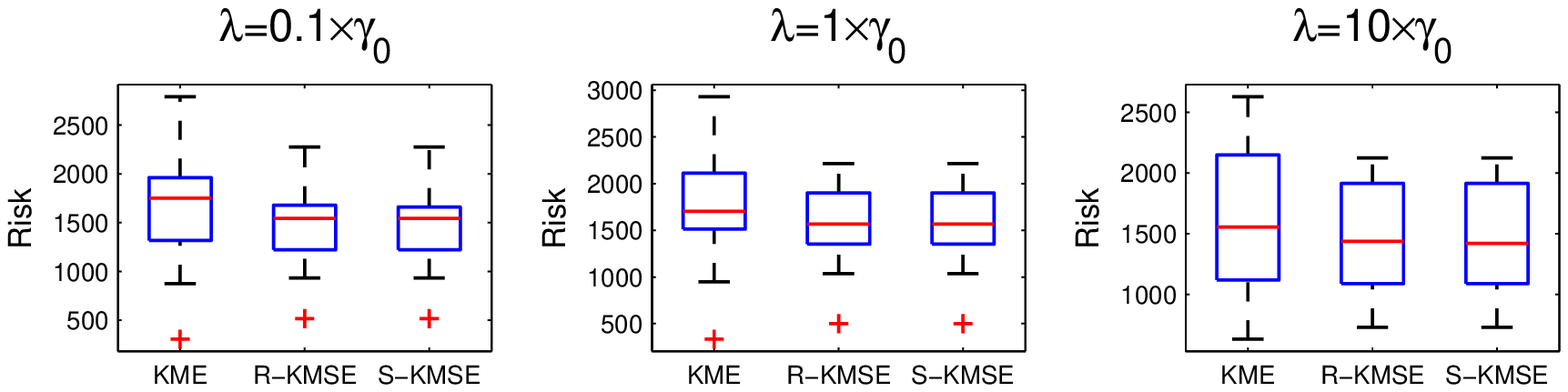}
    \label{fig:synthetic-sf1}
  }
  \hfill
  \subfigure[POLY2]{
    \includegraphics[width=0.9\linewidth]{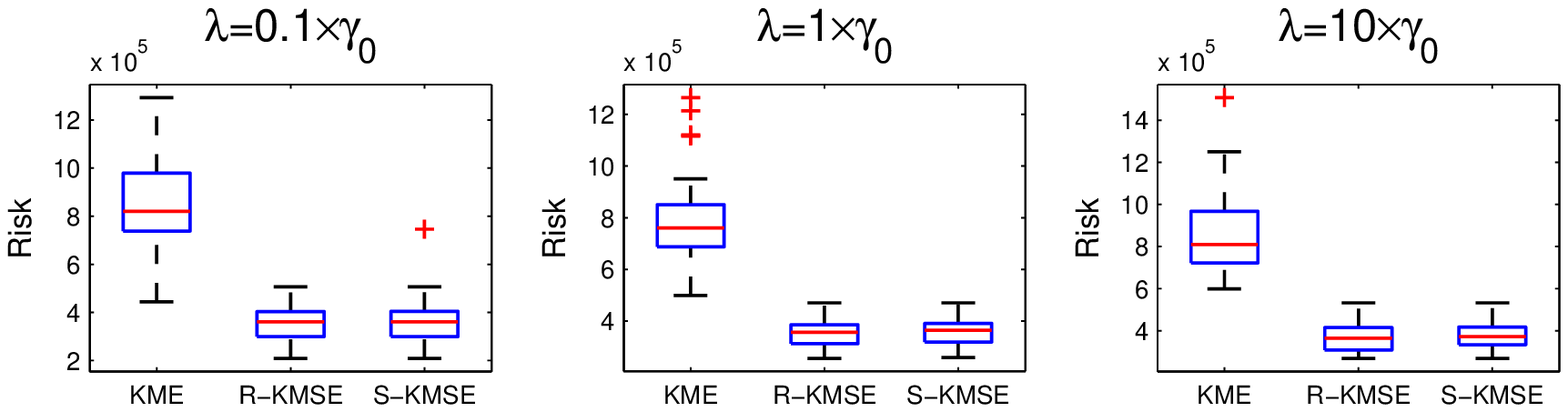}
    \label{fig:synthetic-sf2}
  }
  \hfill
  \subfigure[POLY3]{
    \includegraphics[width=0.9\linewidth]{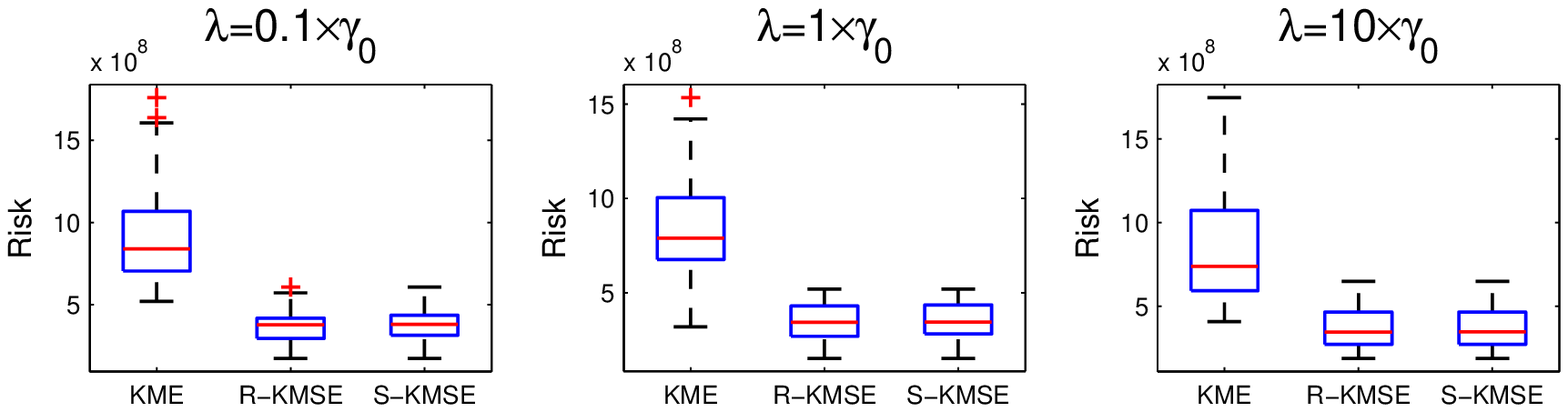}
    \label{fig:synthetic-sf3}
  }
  \hfill
  \subfigure[RBF]{
    \includegraphics[width=0.9\linewidth]{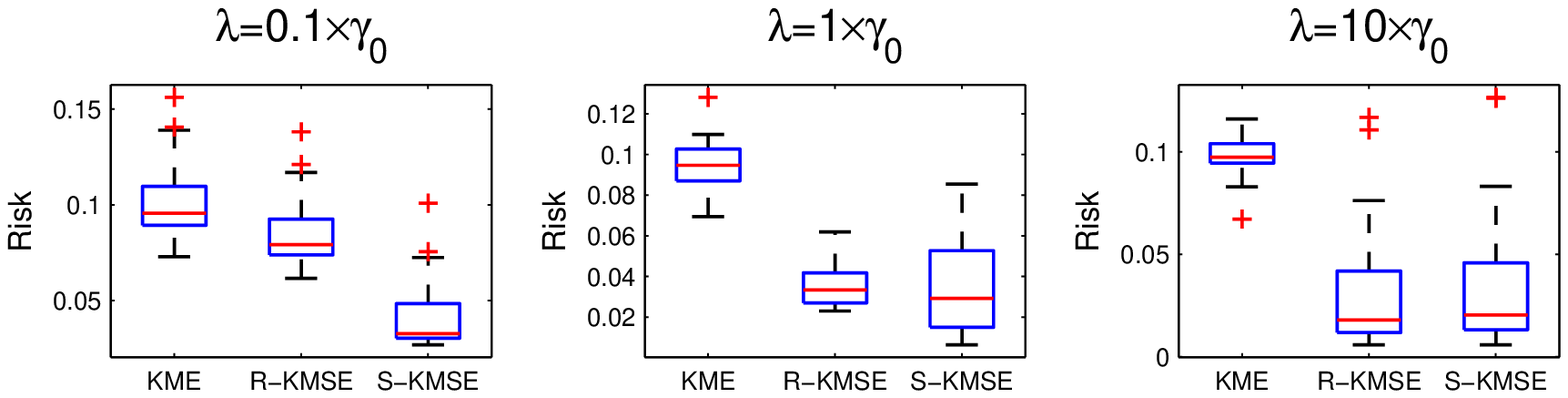}
    \label{fig:synthetic-sf4}
  }
  \caption{The average loss of KME (left), R-KMSE (middle) and S-KMSE (right) estimators with
    different values of shrinkage parameter. 
    We repeat the experiments over 30 different distributions with $n=10$ and $d=30$.}
  \label{fig:sim-results1}
\end{figure*}

To demonstrate the leave-one-out cross-validation procedure, we conduct similar experiments in which the parameter $\lambda$ is chosen by the proposed LOOCV procedure. 
Figure \ref{fig:sim-results2} depicts the percentage of improvement (with respect to the empirical risk of the KME\footnote{If we denote the loss of KME and KMSE as $\ell_{KME}$ and 
$\ell_{KMSE}$, respectively, the percentage of improvement is calculated as $100\times(\ell_{KME}-\ell_{KMSE})/\ell_{KME}$.}) as we vary the sample size and dimension of the data. 
Clearly, B-KMSE, R-KMSE and S-KMSE outperform the standard estimator. Moreover, both R-KMSE and S-KMSE tend to outperform the B-KMSE. We can also see that the performance of S-KMSE 
depends on the choice of kernel. This makes sense intuitively because S-KMSE also incorporates the eigen-spectrum of $\kmat$, whereas R-KMSE does not. The effects of both 
sample size and data dimensionality are also transparent from Figure \ref{fig:sim-results2}. 
While it is intuitive to see that the improvement gets smaller with increase in sample size, it is a bit surprising to see that we can gain much more in high-dimensional input space, 
especially when the kernel function is non-linear, because the estimation happens in the feature space associated with the kernel function rather than in the input space. Lastly, we 
note that the improvement is more substantial in the ``large $d$, small $n$'' paradigm. 
\begin{figure}[t!] 
  \centering
  \includegraphics[width=0.85\linewidth]{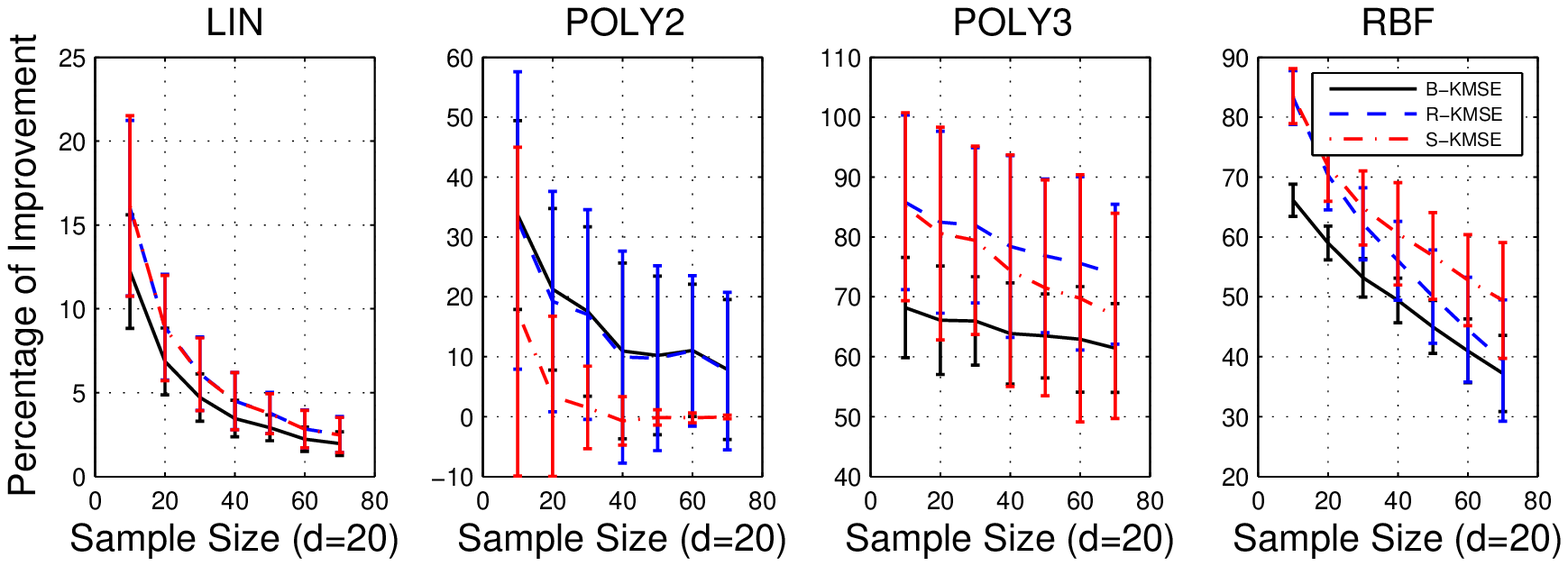} \\
  \includegraphics[width=0.85\linewidth]{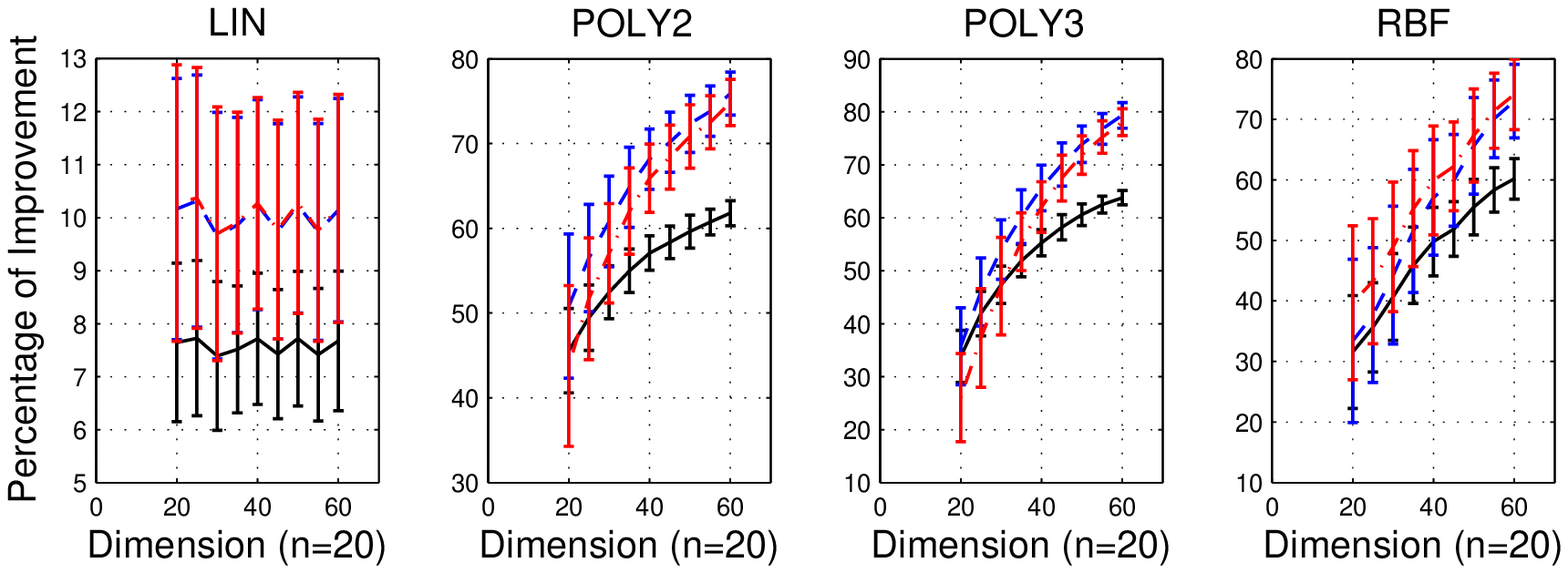}
  \caption{The percentage of improvement compared to KME over 30 different distributions of B-KMSE, R-KMSE and S-KMSE with varying sample size ($n$) and dimension ($d$). For B-KMSE, we calculate $\alpha$ using \eqref{eq:empirical-alpha}, whereas R-KMSE and S-KMSE use LOOCV to choose $\lambda$.}
  \vspace{-2mm}
  \label{fig:sim-results2}
\end{figure}  

\subsection{Real Data}  

To evaluate the proposed estimators on real-world data, we consider several benchmark applications, namely, classification via Parzen window classifier, density estimation via 
kernel mean matching \citep{Song08:TDE}, 
and discriminative learning on distributions \citep{Muandet12:SMM,Muandet13:OCSMM}. For some of these tasks we employ datasets from the UCI repositories. We use only real-valued features, each of which is normalized to have zero mean and unit variance.

\subsubsection{Parzen Window Classifiers}

One of the oldest and best-known classification algorithms is the \emph{Parzen window classifier} \citep{Duda00:PC}. It is easy to implement and is one of the powerful non-linear supervised learning techniques. Suppose we have data points from two classes, namely, positive class and negative class. For positive class, we observe $\mathfrak{X}\triangleq \{x_1,x_2,\ldots,x_n\}\subset\inspace$, while for negative class we have $\mathfrak{Y} \triangleq \{y_1,y_2,\ldots,y_m\}\subset\inspace$. 
Following \citet[Sec. 5.1.2]{Shawe04:KMPA}, the Parzen window classifier is given by
\begin{equation}
  \label{eq:parzen-window}
  f(z) = \mathrm{sgn}\left(\frac{1}{n}\sum_{i=1}^n k(z,x_i) - \frac{1}{m}\sum_{j=1}^m k(z,y_j) + b\right) = \mathrm{sgn}\left( \hat{\mu}_{\mathfrak{X}}(z) - \hat{\mu}_{\mathfrak{Y}}(z) + b\right),
\end{equation}
\noindent where $b$ is a bias term given by $b = \frac{1}{2}(\|\hat{\mu}_{\mathfrak{Y}}\|^2_{\hbspace} - \|\hat{\mu}_{\mathfrak{X}}\|^2_{\hbspace})$. Note that $f(z)$ is a threshold linear function in $\hbspace$ with weight vector $\mathbf{w} = (1/n)\sum_{i=1}^n\phi(x_i) - (1/m)\sum_{j=1}^m\phi(y_j)$ (see \citet[Sec. 5.1.2]{Shawe04:KMPA} for more detail). This algorithm is often referred to as the lazy algorithm as it does not require training.

\begin{table}[t]  
  \centering
  \resizebox{\textwidth}{!}{ 
  \begin{tabular}{|c||cccc|}
    \hline
    \multirow{2}{*}{\textbf{Dataset}} & \multicolumn{4}{c|}{\textbf{Classification Error Rate}} \\
    & \textbf{KME} & \textbf{B-KMSE} & \textbf{R-KMSE} & \textbf{S-KMSE} \\
    \hline
    \hline
    Climate Model & 0.0348$\pm$0.0118 & 0.0348$\pm$0.0118 & 0.0348$\pm$0.0118 & 0.0348$\pm$0.0118 \\
    Ionosphere &  0.2873$\pm$0.0343 & \textbf{0.2768$\pm$0.0359} & \textbf{0.2749$\pm$0.0341} & 0.2800$\pm$0.0367 \\
    Parkinsons & 0.1318$\pm$0.0441 & 0.1250$\pm$0.0366 & \textbf{0.1157$\pm$0.0395} & 0.1309$\pm$0.0396 \\
    Pima & 0.2951$\pm$0.0462 & 0.2921$\pm$0.0442 & 0.2937$\pm$0.0458 & 0.2943$\pm$0.0471 \\
    SPECTF & 0.2583$\pm$0.0829 & 0.2597$\pm$0.0817 & \textbf{0.2263$\pm$0.0626} & 0.2417$\pm$0.0651 \\
    Iris &  0.1079$\pm$0.0379 & 0.1071$\pm$0.0389 & 0.1055$\pm$0.0389 & 0.1040$\pm$0.0383 \\
    Wine &  0.1301$\pm$0.0381 & \textbf{0.1183$\pm$0.0445} & \textbf{0.1161$\pm$0.0414} & \textbf{0.1183$\pm$0.0431} \\
    \hline
  \end{tabular}
  }
  \caption{The classification error rate of Parzen window classifier via different kernel mean estimators. The boldface represents the result whose difference from the baseline, \ie, KME, is statistically significant.}
  \vspace{-7mm}
  \label{tab:parzen-window}
\end{table}


In brief, the classifier \eqref{eq:parzen-window} assigns the data point $z$ to the class whose empirical kernel mean $\hat{\mu}$ is closer to the feature map $k(z,\cdot)$ of the data point in the RKHS. On the other hand, we may view the empirical kernel mean $\hat{\mu}_{\mathfrak{X}} \triangleq \frac{1}{n}\sum_{i=1}^nk(x_i,\cdot)$ (resp. $\hat{\mu}_{\mathfrak{Y}} \triangleq \frac{1}{m}\sum_{j=1}^mk(y_j,\cdot)$) as a standard empirical estimate, \ie, KME, of the true kernel mean representation of the class-conditional distribution $\pp{P}(X|Y=+1)$ (resp. $\pp{P}(X|Y=-1)$). 
Given the improvement of shrinkage estimators over the empirical estimator of kernel mean, it is natural to expect that the performance of Parzen window classifier can be improved by 
employing shrinkage estimators of the true mean representation.

Our goal in this experiment is to compare the performance of Parzen window classifier using different kernel mean estimators. That is, we replace $\hat{\mu}_{\mathfrak{X}}$ 
and $\hat{\mu}_{\mathfrak{Y}}$ by their shrinkage counterparts and evaluate the resulting classifiers across several datasets taken from the UCI machine learning repository. In 
this experiment, we only consider the Gaussian RBF kernel whose bandwidth parameter is chosen by cross-validation procedure over a uniform grid $\sigma\in[0.1,2]$. We use 30\% of each dataset as a test set and the rest as a training set. We employ a simple pairwise coupling and majority vote for multi-class classification. We repeat the experiments 100 times and perform the paired-sample $t$-test on the results at 5\% significance level.
Table \ref{tab:parzen-window} reports the classification error rates of the Parzen window classifiers with different kernel mean estimators. Although the improvement is not substantial, we can see that the shrinkage estimators consistently give better performance than the standard estimator.

\subsubsection{Density Estimation}
  We perform density estimation via kernel mean matching \citep{Song08:TDE}, wherein we fit the density $Q=\sum_{j=1}^m\pi_j\mathcal{N}(\bm{\theta}_j,\sigma_j^2\id)$ to each dataset by the following minimization problem:
\begin{equation}  
  \min_{\bm{\pi},\bm{\theta},\bm{\sigma}} \|\hat{\mu} - \mu_{Q}\|_{\hbspace}^2 \quad \text{subject to} \quad \sum_{j=1}^m\pi_j = 1, \; \pi_j \geq 0 \,. \label{eq:optimaa}
\end{equation}
The empirical mean map $\hat{\mu}$ is obtained from samples using different estimators, whereas $\mu_{Q}$ is the kernel mean embedding of the density $Q$. Unlike experiments in 
\citet{Song08:TDE}, our goal is to compare different estimators of $\mu_{\pp{P}}$ (where $\pp{P}$ is the true data distribution), by replacing $\hat\mu$ in (\ref{eq:optimaa}) with different shrinkage 
estimators. A better estimate of $\mu_{\pp{P}}$ should lead to better density estimation, as measured by the negative log-likelihood of $Q$ on the test set, which we choose to be 30\%
of the dataset. For each dataset, we set the number of mixture components $m$ to be $10$. The model is initialized by running 50 random initializations using the k-means algorithm and 
returning the best. We repeat the experiments 30 times and perform the paired sign test on the results at 5\% significance level.\footnote{The paired sign test is a nonparametric 
test that can be used to examine whether two paired samples have the same distribution. In our case, we compare B-KMSE, R-KMSE and S-KMSE against KME.}

The average negative log-likelihood of the model $Q$, optimized via different estimators, is reported in Table \ref{tab:kmm}. In most cases, both R-KMSE and S-KMSE consistently achieve 
smaller negative log-likelihood when compared to KME. B-KMSE also tends to outperform the KME. However, in few cases the KMSEs achieve larger negative log-likelihood, especially when we 
use linear and degree-2 polynomial kernels. This highlight the potential of our estimators in a non-linear setting.

\subsubsection{Discriminative Learning on Probability Distributions}

The last experiment involves the discriminative learning on a collection of probability distributions via the kernel mean representation. A positive semi-definite kernel between 
distributions can be defined via their kernel mean embeddings. That is, given a training sample 
$(\widehat{\pp{P}}_1,y_1),\ldots,(\widehat{\pp{P}}_m,y_m)\in\mathscr{P}\times\{-1,+1\}$ where $\widehat{\pp{P}}_i := \frac{1}{n_i}\sum_{p=1}^{n_i}\delta_{x^{i}_{p}}$ and 
$x^{i}_{p}\sim\pp{P}_i$, the 
%
\noindent linear kernel between two distributions is approximated by 
   \begin{equation*}
     \langle\hat{\mu}_{\pp{P}_i},\hat{\mu}_{\pp{P}_j}\rangle_{\hbspace} 
     = \left\langle \sum_{p=1}^{n_i}\beta^i_p \phi(x^i_p),\sum_{q=1}^{n_j}\beta^j_q \phi(x^j_q)\right\rangle_{\hbspace} 
     = \sum_{p=1}^{n_i}\sum_{q=1}^{n_j}\beta^i_p\beta^j_q k(x^{i}_{p},x^{j}_{q}) ,
   \end{equation*}
\noindent where the weight vectors $\bvec^i$ and $\bvec^j$ come from the kernel mean estimates of $\mu_{\pp{P}_i}$ and $\mu_{\pp{P}_j}$, respectively. The non-linear kernel can then
be defined accordingly, \eg, $\kappa(\pp{P}_i,\pp{P}_j) = \exp(\|\hat{\mu}_{\pp{P}_i} - \hat{\mu}_{\pp{P}_j}\|^2_{\hbspace}/2\sigma^2)$, see \citet{Christmann10:Kernels}. Our goal in this experiment is to 
investigate if the shrinkage estimators of the kernel mean improve the performance of discriminative learning on distributions. To this end, we conduct experiments on natural 
scene categorization using support measure machine (SMM) \citep{Muandet12:SMM} and group anomaly detection on a high-energy physics dataset using one-class SMM (OCSMM) \citep{Muandet13:OCSMM}. We use both linear and non-linear kernels where the Gaussian RBF kernel is employed as an embedding kernel \citep{Muandet12:SMM}. All hyper-parameters are chosen by 10-fold 
cross-validation.\footnote{In principle one can incorporate the shrinkage parameter into the cross-validation procedure. In this work we are only interested in the value of $\lambda$ returned by the proposed LOOCV procedure.} For our unsupervised problem, we repeat the experiments using several parameter settings and report the best results.
Table \ref{tab:smm-ocsmm} reports the classification accuracy of SMM and the area under ROC curve (AUC) of OCSMM using different kernel mean estimators. All shrinkage estimators consistently lead to better performance on both SMM and OCSMM when compared to KME.

In summary, the proposed shrinkage estimators outperform the standard KME. While B-KMSE and R-KMSE are very competitive compared to KME, S-KMSE tends to outperform both B-KMSE and R-KMSE, however, sometimes leading to poor estimates depending on the dataset and the kernel function.

\afterpage{
\begin{landscape}
\begin{table}
  \centering  
  \resizebox{1.4\textwidth}{!}{ 
  \begin{tabular}{|rl|cccc|cccc|cccc|cccc|} 
    \hline
    \multicolumn{2}{|c|}{\multirow{2}{*}{\textbf{Dataset}}} & \multicolumn{4}{c|}{\textbf{LIN}} & \multicolumn{4}{c|}{\textbf{POLY2}} & \multicolumn{4}{c|}{\textbf{POLY3}} & \multicolumn{4}{c|}{\textbf{RBF}} \\    
    && KME & B-KMSE & R-KMSE & S-KMSE & KME & B-KMSE & R-KMSE & S-KMSE & KME & B-KMSE & R-KMSE & S-KMSE & KME & B-KMSE & R-KMSE & S-KMSE \\
    \hline \hline 
    1. & ionosphere & 39.878 & 40.038 & 39.859 & 39.823 & 34.651 & \textbf{34.352} & 34.390 & \textbf{34.009} & 35.943 & \textbf{35.575} & 35.543 & \textbf{34.617} & 41.601 & 40.976 & \textbf{40.817} & 41.229 \\
    2. & sonar & 72.240 & \textbf{72.044} & 72.198 & \textbf{72.157} & 100.420 & \textbf{99.573} & \textbf{97.844} & \textbf{97.783} & 72.294 & \textbf{71.933} & 72.003 & 71.835 & 98.540 & \textbf{95.815} & \textbf{93.458} & \textbf{93.010} \\
    3. & Australian & 18.277 & 18.280 & 18.294 & 18.293 & 18.357 & 18.381 & 18.391 & 18.429 & 18.611 & 18.463 & 18.466 & 18.495 & 19.428 & 19.325 & 19.418 & 19.393 \\
    4. & specft & 57.444 & \textbf{57.2808} & 57.218 & \textbf{57.224} & 67.018 & 66.979 & \textbf{66.431} & \textbf{66.391} & 59.585 & \textbf{58.969} & 60.006 & 60.616 & 65.674 & 65.138 & 65.039 & \textbf{64.699} \\ 
    5. & wdbc & 31.801 &  31.759 & 31.776 & 31.781 & 32.421 & 32.310 & 32.373 & 32.316 & 31.183 & \textbf{31.167} & \textbf{31.127} & 31.110 & 36.471 & 36.453 & 36.335 & 35.898 \\ 
    6. & wine & 16.019 & 16.000 & 16.039 & 16.009 & 17.070 & 16.920 & \textbf{16.886} & 16.960 & 16.393 & 16.300 & 16.309 & 16.202 & 17.569 & 17.546 & 17.498 & 17.498 \\
    7. & satimage & 25.258 & 25.317 & 25.219 & 25.186 & 24.214 & 24.111 & 24.132 & 24.259 & 25.284 & 25.276 & 25.239 & 25.263 & 23.741 & 23.753 & 23.728 & \textbf{24.384} \\
    8. & segment & 18.326 & \textbf{17.868} & 18.055 & \textbf{18.124} & 18.571 & 18.292 & 18.277 & 18.631 & 19.642 & 19.549 & 19.404 & 19.628 & 21.946 & \textbf{21.598} & \textbf{21.580} & \textbf{21.822} \\
    9. & vehicle & 16.633 & 16.519 & 16.521 & 16.499 & 16.096 & \textbf{15.998} & 16.031 & 16.041 & 16.288 & 16.278 & \textbf{16.281} & \textbf{16.263} & 18.260 & 18.056 & 18.119 & \textbf{17.911} \\
    10. & svmguide2 & 27.298 & 27.273 & 27.281 & 27.276 & 27.812 & \textbf{28.030} & 27.985 & \textbf{27.975} & 28.014 & \textbf{28.177} & \textbf{28.321} & 28.250 & 28.132 & 28.122 & 28.119 & 28.020 \\
    11. & vowel & 12.632 & 12.626 & 12.629 & 12.656 & 12.532 & 12.471 & 12.479 & 12.472 & 13.069 & 13.061 & 13.056 & \textbf{13.054} & 13.526 & 13.486 & 13.462 & \textbf{13.453} \\
    12. & housing & 14.637 & 14.441 & 14.469 & \textbf{14.296} & 15.543 & 15.467 & 15.414 & 15.390 & 15.592 & \textbf{15.543} & \textbf{15.509} & \textbf{15.408} & 16.487 & 16.239 & 16.424 & \textbf{16.019} \\
    13. & bodyfat & 17.527 & 17.362 & \textbf{17.348} & 17.396 & 17.386 & 17.358 & 17.356 & 17.329 & 16.418 & 16.393 & \textbf{16.305} & \textbf{16.194} & 17.875 & 17.652 & 17.607 & \textbf{17.651} \\
    14. & abalone & 5.706 &  5.665 & 5.708 &  5.722 & 7.281 & 7.116 & 7.185 & 7.025 & 5.864 & 5.847 & 5.853 & 5.832 & 6.068 & 6.039 & 6.049 & 5.910 \\
    15. & glass & 9.245 & 9.211 & 9.198 & 9.217 & 8.571 & 8.473 & 8.457 & 8.414 & 9.050 & 8.991 & 9.012 & \textbf{8.737} & 9.606 & \textbf{9.605} & \textbf{9.575} & \textbf{9.573} \\
    \hline 
  \end{tabular}} 
  \caption{Average negative log-likelihood of the model $Q$ on test points over 30 randomizations. The boldface represents the result whose difference from the baseline, \ie, KME, is statistically significant.}
  \label{tab:kmm}  
\end{table}
\begin{table}[tp!]
  \centering
  \begin{tabular}{|l|cc|cc|}
    \hline
    \multirow{2}{*}{\textbf{Estimator}} & \multicolumn{2}{c|}{\textbf{Linear Kernel}} & \multicolumn{2}{c|}{\textbf{Non-linear Kernel}} \\
    & SMM & OCSMM & SMM & OCSMM \\
    \hline
    KME & 0.5432 & 0.6955 & 0.6017 & 0.9085 \\
    B-KMSE & 0.5455 & 0.6964 & 0.6106 & 0.9088 \\
    R-KMSE & 0.5521 & 0.6970 & 0.6303 & 0.9105 \\
    S-KMSE & 0.5606 & 0.6970 & 0.6412 & 0.9063 \\
    \hline
  \end{tabular}
  \caption{The classification accuracy of SMM and the area under ROC curve (AUC) of OCSMM using different estimators to construct the kernel on distributions.}
  \label{tab:smm-ocsmm}
\end{table}
\end{landscape}
}

\section{Conclusion and Discussion}
\label{sec:conclusions}
Motivated by the classical James-Stein phenomenon, in this paper, we proposed a shrinkage estimator for the kernel mean $\mu$ in a reproducing kernel Hilbert space $\hbspace$ and showed they improve upon the empirical estimator $\hat{\mu}$ in the mean squared sense. We 
showed the proposed shrinkage estimator $\tilde{\mu}$ (with the shrinkage parameter being learned from data) to be $\sqrt{n}$-consistent and satisfies $\pp{E}\Vert \tilde{\mu}-\mu\Vert^2_\hbspace<\pp{E}\Vert \hat{\mu}-\mu\Vert^2_\hbspace+O(n^{-3/2})$ as $n\rightarrow\infty$. We 
also provided a regularization interpretation to shrinkage estimation, using which we also presented two shrinkage estimators, namely regularized shrinkage estimator and spectral shrinkage estimator, wherein 
the first one is closely related to $\tilde{\mu}$ while the latter exploits the spectral decay of the covariance operator in $\hbspace$. 
We showed through numerical experiments that the proposed estimators outperform the empirical estimator in various scenarios. Most importantly, the shrinkage estimators 
not only provide more accurate estimation, but also lead to superior performance on many real-world applications. 
 
In this work, while we focused mainly on an estimation of the mean function in RKHS, it is quite straightforward to extend the shrinkage idea to
estimate covariance (and cross-covariance) operators and tensors in RKHS (see Appendix \ref{appx:shrink-cov} for a brief description). The key observation is that the covariance operator can be viewed as a mean function in a tensor RKHS. 
Covariance operators in RKHS are ubiquitous in many classical learning algorithms such as kernel PCA, kernel FDA, and kernel CCA. Recently, a preliminary investigation with some numerical results on shrinkage estimation of covariance operators is carried out 
in \citet{Muandet14:KMSE} and \citet{Wehbe15:SmallSize}. In the future, we intend to carry out a detailed study on the shrinkage estimation of covariance (and cross-covariance) 
operators.

\acks{The authors thanks the reviewers and the action editor for their detailed comments that signficantly improved the manuscript. This work was partly done while Krikamol Muandet was visiting the Institute of Statistical Mathematics, Tokyo, and New York University, New York; and while Bharath Sriperumbudur was visiting the Max Planck Institute for Intelligent Systems, Germany. The authors wish to thank David Hogg and Ross Fedely for reading the first draft and giving valuable comments. We also thank Motonobu Kanagawa, Yu Nishiyama, and Ingo Steinwart for fruitful discussions. Kenji Fukumizu has been supported in part by MEXT Grant-in-Aid for Scientific Research on Innovative Areas 25120012.}

\appendix

\section{Shrinkage Estimation of Covariance Operator}
\label{appx:shrink-cov}


Let $(\hbspace_X,k_X)$ and $(\hbspace_Y,k_Y)$ be separable RKHSs of functions on measurable spaces $\inspace$ and $\mathcal{Y}$, with measurable reproducing kernels $k_X$ and $k_Y$ 
(with corresponding feature maps $\phi$ and $\varphi$), respectively. We consider a random vector $(X,Y):\Omega \rightarrow \inspace\times\mathcal{Y}$ with distribution $\pp{P}_{\mathit{XY}}$. The marginal distributions of $X$ and $Y$ are denoted by $\pp{P}_X$ and $\pp{P}_Y$, respectively. If $\ep_Xk_X(X,X) < \infty$ and $\ep_Yk_Y(Y,Y) < \infty$, then there exists a unique \emph{cross-covariance operator} $\Sigma_{\mathit{YX}}:\hbspace_X\rightarrow\hbspace_Y$ such that
\begin{equation*}
  \langle g,\Sigma_{\mathit{YX}}f\rangle_{\hbspace_Y} = \ep_{\mathit{XY}}[(f(X) - \ep_X[f(X)])(g(Y) - \ep_Y[g(Y)])] = Cov(f(X),g(Y))
\end{equation*}
holds for all $f\in\hbspace_X$ and $g\in\hbspace_Y$ \citep{Baker1973,Fukumizu04:DRS}. If $X$ is equal to $Y$, we obtain the self-adjoint operator $\Sigma_{\mathit{XX}}$ called the \emph{covariance operator}. 
Given i.i.d sample $\{(x_i,y_i)\}^n_{i=1}$ from $\pp{P}_{\mathit{XY}}$, we can write the empirical cross-covariance operator $\widehat{\Sigma}_{\mathit{YX}}$ as
\begin{equation}
  \label{eq:emp-cco}
  \widehat{\Sigma}_{\mathit{YX}} \triangleq \frac{1}{n}\sum_{i=1}^n\phi(x_i)\otimes\varphi(y_i) - \hat{\mu}_X\otimes\hat{\mu}_Y
\end{equation}
\noindent where $\hat{\mu}_X = \frac{1}{n}\sum_{i=1}^n\phi(x_i)$ and $\hat{\mu}_Y = \frac{1}{n}\sum_{i=1}^n\varphi(y_i)$.\footnote{Although it is possible to estimate $\hat{\mu}_X$ and $\hat{\mu}_Y$ using our shrinkage estimators, the key novelty here is to directly shrink the \emph{centered} covariance operator.} 
Let $\tilde{\phi}$ and $\tilde{\varphi}$ be the centered version of the feature map $\phi$ and $\varphi$ defined as $\tilde{\phi}(x)=\phi(x)-\hat{\mu}_X$ and $\tilde{\varphi}(y)=\varphi(y)-\hat{\mu}_Y$, respectively. Then, the empirical cross-covariance operator in \eqref{eq:emp-cco} can be rewritten as
\begin{equation*}
  \widehat{\Sigma}_{\mathit{YX}}=\frac{1}{n}\sum_{i=1}^n\tilde{\phi}(x_i)\otimes\tilde{\varphi}(y_i),
\end{equation*}
and therefore a shrinkage estimator of $\Sigma_{\mathit{YX}}$, \eg, an equivalent of B-KMSE, can be constructed based on the ideas presented in this paper. That is, by the inner product property in product space, we have 
 \begin{eqnarray*}
   \langle \tilde{\phi}(x)\otimes\tilde{\varphi}(y),\tilde{\phi}(x')\otimes\tilde{\varphi}(y')\rangle_{\hbspace_X\otimes\hbspace_Y} 
   &{}={}& \langle\tilde{\phi}(x),\tilde{\phi}(x')\rangle_{\hbspace_X}\langle\tilde{\varphi(y)},\tilde{\varphi(y')}\rangle_{\hbspace_Y} \\
   &{}={}& \tilde{k}_X(x,x')\tilde{k}_Y(y,y'). 
\end{eqnarray*}
\noindent where $\tilde{k}_X$ and $\tilde{k}_Y$ denote the centered kernel functions. 
As a result, we can obtain the shrinkage estimators for $\Sigma_{\mathit{YX}}$ by plugging the above kernel into the KMSEs. 
\vskip 0.2in
\bibliographystyle{plainnat}
\bibliography{kmse-jmlr2014}

\end{document}